%% file: arxiv.tex
\pdfoutput=1
\documentclass{article}

\usepackage{microtype}
\usepackage{graphicx}
\usepackage{booktabs} 
\usepackage[caption=false]{subfig}
\usepackage{amsmath}
\usepackage{amsfonts}
\usepackage{makecell}
\usepackage{multicol, latexsym, amsmath, amssymb}
\usepackage{blindtext}
\usepackage{caption}
\usepackage{multirow}
\usepackage{hhline}
\usepackage[font={bf}]{caption}
\usepackage{mathrsfs}
\usepackage{eufrak}
\usepackage{algorithm}
\usepackage{algorithmic}
\usepackage{xr}
\usepackage[linesnumbered,ruled,algo2e]{algorithm2e}
\usepackage{subfiles}
\usepackage[hyphens]{url}
\usepackage{xr-hyper}
\usepackage[accepted]{icml2020}
\usepackage{hyperref}

\usepackage{custom_tex}
\usepackage{xr}
\usepackage{enumitem}

\usepackage{tikz}
\usetikzlibrary{arrows.meta}
\usepackage{mathtools}
\usepackage[bottom]{footmisc}

\usetikzlibrary{shapes,arrows}
\usetikzlibrary{calc} 
\newcommand{\nocontentsline}[3]{}
\newcommand{\tocless}[2]{\bgroup\let\addcontentsline=\nocontentsline#1{#2}\egroup}

\icmltitlerunning{DeltaGrad}

\begin{document}

\twocolumn[
\icmltitle{DeltaGrad: Rapid retraining of machine learning models}




\begin{icmlauthorlist}
\icmlauthor{Yinjun Wu}{to}
\icmlauthor{Edgar Dobriban}{goo}
\icmlauthor{Susan B. Davidson}{to}
\end{icmlauthorlist}

\icmlaffiliation{to}{Department of Computer and Information Science, University of Pennsylvania, PA, United States}
\icmlaffiliation{goo}{Department of Statistics, University of Pennsylvania, PA, United States}

\icmlcorrespondingauthor{Yinjun Wu}{wuyinjun@seas.upenn.edu}
\icmlcorrespondingauthor{Edgar Dobriban}{dobriban@wharton.upenn.edu}
\icmlcorrespondingauthor{Susan B. Davidson}{susan@cis.upenn.edu}
\icmlkeywords{Machine Learning, ICML}

\vskip 0.3in
]



\printAffiliationsAndNotice{}  

\begin{abstract}
Machine learning models are not static and may need to be retrained on slightly changed datasets, for instance, with the addition or deletion of a set of datapoints. This has many applications, including privacy, robustness, bias reduction, and uncertainty quantification. However, it is expensive to retrain models from scratch. To address this problem, we propose the DeltaGrad algorithm for rapid retraining machine learning models based on information cached during the training phase. We provide both theoretical and empirical support for the effectiveness of DeltaGrad, and show that it compares favorably to the state of the art.
\end{abstract}

\input{Sections/introduction.tex}
\input{Sections/Algorithm.tex}
\input{Sections/Alg_extensions.tex}
\input{Sections/Experiment.tex}
\input{Sections/Applications.tex}
\input{Sections/Conclusion.tex}

\clearpage

\section*{Acknowledgements}
This material is based upon work that is in part supported by the Defense Advanced Research Projects Agency (DARPA) under Contract No. HR001117C0047.
Partial support was provided by NSF Awards 1547360 and 1733794.

\bibliography{ref}
\bibliographystyle{icml2020}




\clearpage

\setcounter{equation}{0}
\setcounter{figure}{0}
\setcounter{table}{0}
\setcounter{page}{1}
\makeatletter
\renewcommand{\theequation}{S\arabic{equation}}
\renewcommand{\thefigure}{S\arabic{figure}}
\renewcommand{\bibnumfmt}[1]{[S#1]}
\renewcommand{\citenumfont}[1]{S#1}
\setcounter{section}{0}
\renewcommand\thesection{\Alph{section}}
\onecolumn

\input{supplement_2.tex}

\end{document}

%% file: Sections/introduction.tex

\tocless\section{Introduction}

Machine learning models are used increasingly often, and are rarely static. Models may need to be retrained on slightly changed datasets, for instance when datapoints have been added or deleted. 
This has many applications, including privacy, robustness, bias reduction, and uncertainty quantification. For instance, it may be necessary to remove certain datapoints from the training data for privacy and robustness reasons. Constructing models with some datapoints removed can also be used for constructing bias-corrected models, such as in jackknife resampling \cite{quenouille1956notes} which requires retraining the model on all leave-one-out datasets. In addition, retraining models on subsets of data can be used for uncertainty quantification, such as constructing statistically valid prediction intervals via conformal prediction e.g., \citet{shafer2008tutorial}.

Unfortunately, it is expensive to retrain models from scratch. The most common training mechanisms for large-scale models are based on (stochastic) gradient descent (SGD) and its variants. Retraining the models on a slightly different dataset would involve re-computing the entire optimization path. When adding or removing a small number of data points, this can be of the same complexity as the original training process. 

\begin{figure}[t]
\begin{center}
\centerline{\includegraphics[width=0.9\columnwidth, height=0.45\columnwidth]{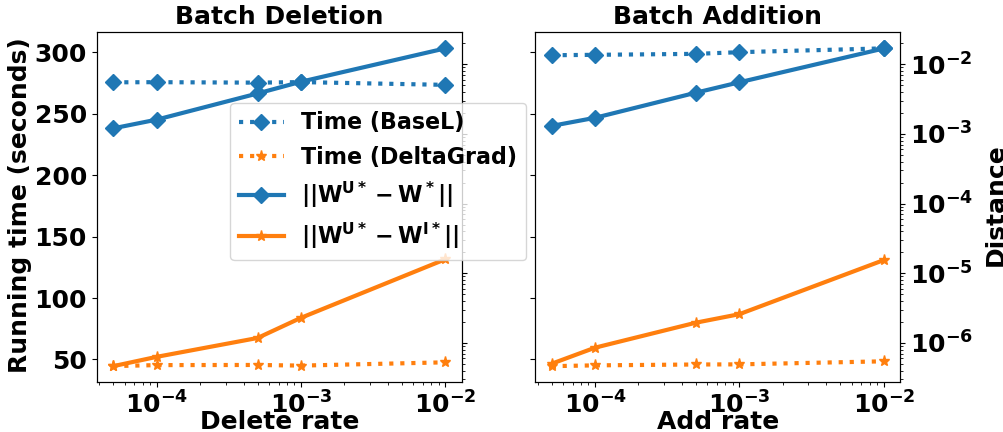}}
\vspace{-2mm}
\caption{Running time of our \Increm\  algorithm for retraining a logistic regression model on  \rcv\ as a function of the fraction of data deleted and added. Our algorithm is faster than training from scratch (Running time BaseL). Also shown is the distance of \Increm\ and the model trained on full data from the correct values (Distance BaseL and Distance DeltaGrad, resp.), illustrating that our algorithm is accurate. See Section \ref{exp} for details.}
\label{fig: DNN_MNIST_batch_deletion}
\label{fig: e}
\end{center}
\vskip -0.2in
\end{figure}

However, we expect models on two similar datasets to be similar. If we retrain the models on many different new datasets, it may be more efficient to cache some information about the training process on the original data, and compute the ``updates''. Such ideas have been used recently e.g., \citet{ginart2019making,guo2019certified,wu2020data}. However, the existing approaches have various limitations: They only apply to specialized problems such as k-means \cite{ginart2019making} or logistic regression \cite{wu2020data}, or they require additional randomization leading to non-standard training algorithms \cite{guo2019certified}.

To address this problem, we propose the \emph{DeltaGrad} algorithm for rapid retraining of machine learning models when slight changes happen in the training dataset, e.g. deletion or addition of samples, based on information cached during training. DeltaGrad addresses several limitations of prior work: it is applicable to general machine learning models defined by empirical risk minimization trained using SGD, and does not require additional randomization. 
It is based on the idea of ``\emph{differentiating the optimization path}'' with respect to the data, and is inspired by ideas from Quasi-Newton methods.

We provide both theoretical and empirical support for the effectivenss of DeltaGrad. We prove that it approximates the true optimization path at a fast rate for strongly convex objectives. We show experimentally that it is accurate and fast on several medium-scale problems on standard datasests, including two-layer neural networks. The speed-ups can be up to 6.5x with negligible accuracy loss (see e.g., Fig. \ref{fig: e}). This paves the way toward a large-scale, efficient, general-purpose data deletion/addition  machine learning system. We also illustrate how it can be used in several applications described above.

\tocless\subsection{Related work}

There is a great deal of work on model retraining and updating. Recently, this has gotten attention due to worldwide efforts on human-centric AI, data confidentiality and privacy, such as the General Data Protection Regulation (GDPR) in the European Union \cite{gdpr}. This mandates that users can ask for their data to be removed from analysis in current AI systems. The required guarantees are thus \emph{stronger} than what is provided by differential privacy (which may leave a non-vanishing contribution of the datapoints in the model, \citet{dwork2014algorithmic}), and or defense against data poisoning attacks (which only requires that the performance of the models does not degrade after poisoning, \citet{steinhardt2017certified}).

Efficient data deletion is also crucial for many other applications, e.g. \textit{model interpretability} and \textit{model debugging}. For example, repeated retraining by removing different subsets of training data each time is essential in many existing data systems \cite{doshi2017roadmap,krishnan2017palm} to understand the effect of those removed data over the model behavior. It is also close to \textit{deletion diagnotics}, targeting locating the most influential data point for the ML models through deletion in the training set, dating back to \cite{cook1977detection}. Some recent work \cite{koh2017understanding} targets general ML models, but requires explicitly maintaining Hessian matrices and can only handle the deletion of one sample, thus inapplicable for many large-scale applications.

Efficient model updating for adding and removing datapoints is possible for linear models, based on efficient rank one updates of matrix inverses \citep[e.g.,][etc]{birattari1999lazy,horn2012matrix,cao2015towards}. The scope of linear methods is extended if one uses linear feature embeddings, either randomized or learned via pretraining. Updates have been proposed for support vector machines \citep{syed1999incremental,cauwenberghs2001incremental} and nearest neighbors \citep{schelteramnesia}.

\citet{ginart2019making} propose a definition of data erasure completeness and a quantization-based algorithm for k-means clustering achieving this. They also propose several principles that can enable efficient model updating. \citet{guo2019certified} propose a general theoretical condition that guarantees that randomized algorithms can remove data from machine learning models. Their randomized approach needs standard algorithms such as logistic regression to be changed to apply. \cite{bourtoule2019machine} propose the SISA (or Sharded, Isolated,
Sliced, Aggregated) training framework for ``un-learning'', which relies on ideas similar to distributed training. Their approach requires dividing
the training data in multiple shards such that a training point
is included in a small number of shards only.

Our approach relies on large-scale optimization, which has an enormous literature.  Stochastic gradient methods date back to \citet{robbins1951stochastic}. More recently a lot of work \citep[see e.g.,][]{bottou1998online,bottou2003stochastic,zhang2004solving,bousquet2008tradeoffs, bottou2010large, bottou2018optimization} focuses on empirical risk minimization. 

The convergence proofs for SGD are based on the contraction of the expected residuals. They are based on assumptions such as bounded variances, the strong or weak growth, smoothness, convexity (or Polyak-Lojasiewicz) on the individual and overall loss functions. See e.g., \cite{gladyshev1965stochastic,amari1967theory,kul1992estimation,bertsekas1996neuro,moulines2011non,karimi2016linear,bottou2018optimization,gorbunov2019unified,gower2019sgd}, etc, and references therein. Our approach is similar, but the technical details are very different, and more closely related to Quasi-Newton methods such as L-BFGS \citep{zhu1997algorithm}.

{\bf Contributions.}
Our contributions are:
\begin{enumerate}[topsep=0pt,partopsep=0pt,itemsep=0pt, parsep=0pt]
    \item {\bf DeltaGrad}: We propose the DeltaGrad algorithm for fast retraining of (stochastic) gradient descent based machine learning models on small changes  of the data (small number of added or deleted points).
    \item {\bf Theoretical support}: We provide theoretical results showing the accuracy of the DeltaGrad. Both for GD and SGD we show the error is of smaller order than the fraction of points removed.
    \item {\bf Empirical results}: We provide empirical results showing the speed and accuracy of DeltaGrad, for addition, removal, and continuous updates, on a number of standard datasets.
    \item {\bf Applications}: We describe the applications of DeltaGrad to several problems in machine learning, including privacy, robustness, debiasing, and statistical inference. 
\end{enumerate}%

%% file: Sections/Algorithm.tex

\tocless\section{Algorithms}\label{sec: alg}


\tocless\subsection{Setup}
The training set $\{(\textbf{x}_i, \textbf{y}_i)\}_{i=1}^n$ has $n$ samples. The loss or objective function for a general machine learning model is defined as:
$$F\left(\w\right) = \frac{1}{n}\sum_{i=1}^n F_i\left(\w\right)
$$
where $\w$ represents a vector of the model parameters and $F_i\left(\w\right)$ is the loss for the $i$-th sample.  
The gradient and  Hessian matrix of $F\left(\w\right)$ are
$$
        \nabla F\left(\w\right) = \frac{1}{n}\sum_{i=1}^n \nabla F_i\left(\w\right),\,\,\,
        \bH\left(\w\right) = \frac{1}{n}\sum_{i=1}^n \bH_i\left(\w\right)
$$
Suppose the model parameter is updated through mini-batch stochastic gradient descent (SGD) for $t=1,\ldots, T$:
\begin{align*}
        & \w_{t+1} \leftarrow \w_{t} - \frac{\eta_t}{B}\sum_{i \in \mathscr{B}_{t}} \nabla F_i\left(\w_{t}\right)  
\end{align*}
\yw{
where $\mathscr{B}_{t}$ is a randomly sampled mini-batch of size $B$ and $\eta_t$ is the learning rate at the $t_{th}$ iteration. As a special case of SGD, the update rule of gradient descent (GD) is $\w_{t+1} \leftarrow \w_{t} - \eta_t/n\sum_{i=1}^n \nabla F_i\left(\w_{t}\right)$.  
}
After training on the full dataset, the training samples with indices $R = \{i_1,i_2,\dots,i_r\}$ are removed, where $r \ll n$. Our goal is to efficiently update the model parameter to the minimizer of the new empirical loss. Our algorithm also applies when $r$ new datapoints are \emph{added}.

\yw{The naive solution is to apply GD directly over the remaining training samples (we use $\uw$ to denote the corresponding model parameter), i.e. run: }
\begin{align}\label{eq: update_rule_naive} 
        & \uw_{t+1} \leftarrow \uw_{t} - \frac{\eta_t}{n-r}\sum_{\substack{i \not\in R}} \nabla F_i\left(\uw_{t}\right) 
\end{align}
which aims to minimize $\uF\left(\w\right)=1/(n-r)\sum_{i \not\in R} F_i\left(\w\right)$.  

\tocless\subsection{Proposed \Increm~Algorithm}

\begin{algorithm}
\small
\SetKwInOut{Input}{Input}
\SetKwInOut{Output}{Output}
\Input{The full training set $\left(\textbf{X}, \textbf{Y}\right)$, model parameters cached during the training phase over the full training samples $\{\w_{0}, \w_{1}, \dots, \w_{t}\}$ and corresponding gradients $\{\nabla F\left(\w_{0}\right), \nabla F\left(\w_{1}\right), \dots, \nabla F\left(\w_{t}\right)\}$, the indices of the removed training samples $R$, period $T_0$, total iteration number $T$, history size $m$, ``burn-in'' iteration number $j_0$, learning rate $\eta_t$}
\Output{Updated model parameter $\iw_{t}$}
Initialize $\iw_{0} \leftarrow \w_{0}$

Initialize an array $\Delta G = \left[\right]$

Initialize an array $\Delta W = \left[\right]$

\For{$t=0;t<T; t++$}{

\eIf{$[((t-j_0) \mod T_0) == 0]$ or $t \leq j_0$}
{
    compute $\nabla F\left(\iw_{t}\right)$ exactly
    
    compute $\nabla F\left(\iw_{t}\right) - \nabla F\left(\w_{t}\right)$ based on the cached gradient $\nabla F\left(\w_{t}\right)$
    
    set $\Delta G\left[k\right] = \nabla F\left(\iw_{t}\right) - \nabla F\left(\w_{t}\right)$
    
    set $\Delta W\left[k\right] = \iw_{t} - \w_{t}$, based on the cached parameters $\w_{t}$
    
    $k\leftarrow k+1$
    
    compute $\iw_{t+1}$ by using exact GD update (equation \eqref{eq: update_rule_naive})
}
{
    Pass $\Delta W\left[-m:\right]$, $\Delta G\left[-m:\right]$, the last $m$ elements in $\Delta W$ and $\Delta G$, which are from the $j_1^{th}, j_2^{th},\dots, j_m^{th}$ iterations where $j_1 < j_2< \dots < j_m$ depend on $t$, $\textbf{v} = \iw_{t} - \w_{t}$, and the history size $m$, to the L-BFGFS Algorithm (see Section \ref{sec: quasi_newton} in the Appendix) to get the approximation of $\bH(\w_{t})\textbf{v}$, i.e., $\B_{j_m}\textbf{v}$
    
    Approximate $\nabla F\left(\iw_{t}\right) = \nabla F\left(\w_{t}\right) + \B_{j_m}\left(\iw_{t} - \w_{t}\right)$
    
    Compute $\iw_{t+1}$ by using the "leave-$r$-out" gradient formula, based on the approximated $\nabla F(\iw_{t})$ 
}
}

\Return $\iw_{t}$
\caption{DeltaGrad}
\label{alg: update_algorithm}
\end{algorithm}

\yw{To obtain a more efficient method, we rewrite Equation \eqref{eq: update_rule_naive} via the following \emph{``leave-$r$-out'' gradient formula} (we use $\iw$ to denote the model parameter derived by \Increm):}
\begin{align}\label{eq: approx_w_t}
\begin{split}
\iw_{t+1} 
         = \iw_{t} - \frac{\eta_t}{n-r}\left[n \nabla F\left(\iw_{t}\right) - \sum_{\substack{i \in R}} \nabla F_i\left(\iw_{t}\right)\right].
\end{split}
\end{align}

\yw{Computing the sum $\sum_{\substack{i \in R}} \nabla F_i\left(\iw_{t}\right)$ of a small number of terms is more efficient than computing $\sum_{\substack{i \not\in R}}$ $\nabla F_i\left(\iw_{t}\right)$ when $|R|= r \ll n$. For this we need to approximate $n\nabla F\left(\iw_{t}\right) =\sum_{i=1}^n \nabla F_i\left(\iw_{t}\right)$ by leveraging the historical gradient $\nabla F\left(\w_{t}\right)$ (recall that $\w_{t}$ is the model parameter before deletions), for each of the $T$ iterations. 
}

\yw{Suppose we can \emph{cache} the model parameters $\w_{0}$, $\dots,\w_{t}$ and the gradients $\nabla F\left(\w_{0}\right)$, $\dots, \nabla F\left(\w_{t}\right)$ for each iteration of training over the original dataset. Suppose that we have been able to approximate $\iw_{0}$, $\dots,\iw_{t}$. Then at iteration $t+1$, $\nabla F\left(\iw_{t}\right)$ can be approximated using the Cauchy mean-value theorem:}
\begin{align}\label{eq: taylor_expansion}
    \begin{split}
        \nabla F\left(\iw_{t}\right) = \nabla F\left(\w_{t}\right) + \bH_{t}\cdot\left(\iw_{t} - \w_{t}\right)
    \end{split}
\end{align}
in which $\bH_{t}$ is an integrated Hessian, $\bH_{t} = \int_0^1 \bH\left(\w_t + x\left(\iw_t - \w_t\right)\right) dx$.

Equation \eqref{eq: taylor_expansion} requires a Hessian-vector product at every iteration. We leverage the L-BFGS algorithm to approximate this, see e.g. \citet{matthies1979solution,nocedal1980updating,byrd1994representations,byrd1995limited,zhu1997algorithm,nocedal2006numerical,mokhtari2015global} and references therein. 
The L-BFGS algorithm uses past data to approximate the projection of the Hessian matrix in the direction of $\w_{t+1}-\w_{t}$.
 We denote the required historical observations at prior iterations $j$ as:
$\Dw_{j} = \iw_{j} - \w_{j}$,
 $\Dg_{j} = \nabla F\left(\iw_{j}\right) - \nabla F\left(\w_{j}\right)$.

\yw{L-BFGS computes \emph{Quasi-Hessians} $\B_{t}$ approximating the true Hessians $\bH_t$ (we follow the notations from the classical L-BFGS papers, e.g., \citet{byrd1994representations}).
DeltaGrad (Algorithm \ref{alg: update_algorithm}) starts with a ``burn-in" period of $j_0$ iterations, where it computes the full gradients  $\nabla F\left(\iw_{t}\right)$ exactly. Afterwards, it only computes the full gradients every $T_0$ iterations. For other iterations $t$, it uses the L-BGFS algorithm, maintaining a set of updates at some prior iterations $j_1$, $j_2$, $\dots$, $j_m$, i.e. $\Dw_{j_1}, \Dw_{j_2}$ , $\dots$, $\Dw_{j_m}$ and $\Dg_{j_1}, \Dg_{j_2}, \dots, \Dg_{j_m}$ where $j_k - j_{k-1} \leq T_0$. Then it uses an efficient L-BGFS update from \citet{byrd1994representations} (see Appendix \ref{sec: quasi_newton} for the details of the L-BGFS update).}

\yw{By approximating $\bH_t$ with $\B_t$ in Equation \eqref{eq: taylor_expansion} and plugging Equation \eqref{eq: taylor_expansion} into Equation \eqref{eq: approx_w_t}, the DeltaGrad update is:}
$\iw_{t+1}- \iw_{t} = \eta_t/(n-r)\cdot$
$$\cdot\begin{cases}
        \sum_{i\not\in R}\nabla F(\iw_t),\,\, (t - j_0) \mod T_0 = 0\textnormal{ or } t \leq j_0 \\
        n[\B_{j_m}(\iw_t - \w_t) + \nabla F(\w_t)] - \sum_{i\in R} \nabla F(\iw_t),\,\, \textnormal{else}
\end{cases}$$
\tocless\subsection{Convergence rate for strongly convex objectives}\label{converge_analysis}

We provide the convergence rate of DeltaGrad for strongly convex objectives in Theorem \ref{main1}. We need to introduce some assumptions. \yw{The norm used throughout the rest of the paper is $\ell_2$ norm.}
\begin{assumption}[Small number of samples removed]\label{assp: small_r}
The number of removed samples, $r$, is far smaller than the total number of training samples, $n$. There is a small constant $\delta>0$ such that $r/n \le \delta$.
\end{assumption}


\begin{assumption}[Strong convexity and smoothness]\label{assp: strong_convex smooth}
Each $F_i \left(\w\right)$ ($i=1,2,\dots,n$) is $\mu-$strongly convex and $L$-smooth with $\mu > 0$, so for any $\w_1,\w_2$
$$
(\nabla F_i\left(\w_1\right) - \nabla F_i\left(\w_2\right))^T \left(\w_1 - \w_2\right) \geq \mu\|\w_1 - \w_2\|^2,
$$
$$
    \|\nabla F_i\left(\w_1\right) - \nabla F_i\left(\w_2\right)\| \leq L\|\w_1 - \w_2\|.
$$
\end{assumption}

Then $F\left(\w\right)$ and $\uF\left(\w\right)$ are $L$-smooth and $\mu$-strongly convex. Typical choices of $\eta_t$ are based on the smoothness and strong convexity parameters, so the same choices lead to the convergence for both $\w_{t}$ and $\uw_{t}$. For instance, GD over a strongly convex objective with fixed step size $\eta_t = \eta \leq 2/[L+\mu]$ converges geometrically at rate $(L-\mu)/(L+\mu)<1$. For simplicity, we will use a constant learning rate $\eta_t = \eta\leq 2/[L+\mu]$. 

We assume bounded gradients and Lipschitz Hessians, which are standard \citep{boyd2004convex,bottou2016optimization}. The proof may be relaxed to weak growth conditions, see the related works for references.
\begin{assumption}[Bounded gradients]\label{assp: gradient_upper_bound}
For any model parameter $\w$ in the sequence $[\w_{0}, \w_{1}, \w_{2},$ $\dots,$ $\w_{t}, \dots]$, the norm of the gradient at every sample is bounded by a constant $c_2$, i.e. for all $i,j$:
$$
    \|\nabla F_i\left(\w_j\right)\| \leq c_2.
$$
\end{assumption}

\begin{assumption}[Lipschitz Hessian]\label{assp: hessian_continuous}
The Hessian $\bH\left(\w\right)$ is Lipschitz continuous. There exists a constant $c_0$ such that for all $\w_1$ and $\w_2$,
$$
    \|\bH\left(\w_1\right)-\bH\left(\w_2\right)\| \leq c_0 \|\w_1 - \w_2\|.
$$
\end{assumption}

An assumption specific to Quasi-Newton methods is the \emph{strong independence} of the weight updates: the smallest singular value of the normalized weight updates is bounded away from zero \citep{ortega1970iterative,conn1991convergence}. This has sometimes been motivated empirically, as the iterates of certain quasi-newton iterations empirically satisfy it \citep{conn1988testing}. 
\begin{assumption}[Strong independence]\label{assp: singular_lower_bound}
For any sequence, $\left[\Dw_{j_1}, \Dw_{j_2}, \dots, \Dw_{j_m}\right]$, the matrix of normalized vectors
$$\Delta W_{j_1, j_2,\dots, j_m} 
= [\Dw_{j_1}, \Dw_{j_2}, \dots, \Dw_{j_m}]/s_{j_1,j_m}$$
where $s_{j_1,j_m} = \max\left(\|\Dw_{j_1}\|,\|\Dw_{j_2}\|, \dots, \|\Dw_{j_m}\|\right)$,
has its minimum singular value $\sigma_{\min}$ bounded away from zero. We have
$
    \sigma_{\min}\left(\Delta W_{j_1, j_2,\dots, j_m}\right) \geq c_1
$
where $c_1$ is independent of ${\left(j_1, j_2,\dots, j_m\right)}$.
\end{assumption}
Empirically, we find $c_1$ around 0.2 for the MNIST dataset using our default hyperparameters.

\tocless\subsubsection{Results}
\tikzstyle{decision} = [diamond, draw, fill=blue!20, 
    text width=5em, text badly centered, node distance=3cm, inner sep=0pt]
\tikzstyle{iterations} = [rectangle, draw, dashed, text width=3em, text centered, minimum height=2em, font=\footnotesize]
\tikzstyle{iterations2} = [rectangle, draw, dashed, text width=12em, text centered, minimum height=2em, font=\footnotesize]
\tikzstyle{block} = [rectangle, draw, fill=blue!20, 
    text width=5em, text centered, rounded corners, minimum height=3em, font=\footnotesize]
\tikzstyle{empty} = [rectangle, draw, text width=5em, text centered, minimum height=3em, font=\footnotesize]
\tikzstyle{line} = [draw]
\tikzstyle{arrow} = [thick,->,>=stealth]

Then our first main result is the convergence rate of the \Increm\ algorithm.

\begin{theorem}[Bound between true and incrementally updated iterates]\label{main1}
For a large enough iteration counter $t$, the result $\iw_t$ of  \Increm\ (Algorithm \ref{alg: update_algorithm}) approximates the correct iteration values $\uw_t$ at the rate
$$\|\uw_t - \iw_t\| = o\left(\frac{r}{n}\right).$$
So $\|\uw_t - \iw_t\|$ is of a lower order than $r/n$.
\end{theorem}

\yw{The baseline error rate between the full model parameters $\w_t$ and $\iw_t$ is expected to be of the order $r/n$, as can be seen from the example of the sample mean. This shows that \Increm\ has a better convergence rate for approximating $\iw_t$. The proof is quite involved. It relies on a delicate analysis of the difference between the approximate Hessians $\B_t$ and the true Hessians $\bH_t$ (see the Appendix, and specifically \ref{sec: gd_proof}).}

\tocless\subsection{Complexity analysis}\label{sec: complexity}
We will do our complexity analysis assuming that the model is given by a computation graph.
Suppose the number of model parameters is $p$ and the time complexity for forward propagation is $f(p)$. Then according to the Baur-Strassen theorem \cite{griewank2008evaluating}, the time complexity of backpropagation in one step will be at most $5f(p)$ and thus the total complexity to compute the derivatives for each training sample is $6 f(p)$. Plus, the overhead of computing the product of $\B_{j_m}(\iw_t - \w_t)$ is $O(m^3) + 6mp + p$ according to \cite{byrd1994representations}, which means that the total time complexity at the step where the gradients are approximated is $6 r f(p) + O(m^3) + 6mp + p$ (the gradients of $r$ removed/added samples are explicitly evaluated), which is more efficient than explicit computation of the gradients over the full batch (a time complexity of $6  (n -r)f(p)$) when $r \ll n$. 

Suppose there are $T$ iterations in the training process. Then the running time of \Basel\ will be $6  (n -r)f(p) T$. \Increm\ evaluates the gradients for the first $j_0$ iterations and once every $T_0$ iterations. So its total running time is $6  (n -r)f(p)\times \frac{T-j_0}{T_0} + (6 r f(p) + O(m^3) + 6mp + p) \times (1-\frac{1}{T_0})(T-j_0)$, which is close to $6 n f(p)\times \frac{T-j_0}{T_0} + (O(m^3) + 6mp + p) \times (1-\frac{1}{T_0})(T-j_0)$ since $r$ is small. Also, when $n$ is large, the overhead of approximate computation, i.e. $(O(m^3) + 6mp + p)$ should be much smaller than that of explicit computation. Thus \emph{speed-ups of a factor $T_0$ are expected} when $j_0$ is far smaller than $T$. 

%% file: Sections/Alg_extensions.tex
\tocless\section{Extension to SGD}

Consider now mini-batch stochastic gradient descent:
\vspace{-1mm}
\begin{align*}
        & \sw_{t+1} = \sw_t - \frac{\eta}{B}\sum_{i\in \miniB_t} \nabla F_i(\sw_t).
\end{align*}
\vspace{-1mm}
The naive solution for retraining the model is:
\vspace{-1mm}
\begin{align*}
    \usw_{t+1} = \usw_{t} - \frac{\eta }{B - \Delta B_t} \sum_{i\in \miniB_t, i \not\in R} \nabla F_i(\usw_{t}).
\end{align*}
\vspace{-1mm}
Here $\Delta B_t$ is the size of the subset removed from the $t$-th minibatch. If $B - \Delta B_t=0$, then we do not change the parameters at that iteration.
\Increm\ can be naturally extended to this case: $\isw_{t+1}- \isw_{t} = \eta_t/(B-\Delta B_t)\cdot$
\vspace{-2mm}
\begin{align*}
\cdot
    \begin{cases}
        \sum_{i\not\in R}\nabla F(\iw_t),\,\, t \mod T_0 = 0\ or\ t \leq j_0 \\
        [B(\B_{j_m}(\iw_t - \w_t)) - \sum_{i\in R} \nabla F(\iw_t)],\,\, else
        \end{cases}
\end{align*}
which relies on a series of historical observations: $\Dsw_{j} = \isw_{j} - \sw_{j}$,
 $\Dsg_{j} = {B}^{-1}\sum_{i\in \miniB_j} \nabla F_i(\isw_j)$ $-$ ${B}^{-1}\sum_{i\in \miniB_j} \nabla F_i(\sw_j)$. 
 \tocless\subsection{Convergence rate for strongly convex objectives}

Recall $B$ is the mini-batch size, $p$ is the total number of model parameters and $T$ is the number of iterations in SGD. Our main result for SGD is the following.
\begin{theorem}[SGD bound for \Increm]\label{main1_sgd1}
With probability at least 
\vspace{-2mm}
\begin{align*}
&1-T\cdot[\prob \\
&+\probw + \wuprob ],
\end{align*}
the result $\isw_t$ of Algorithm \ref{alg: update_algorithm} approximates the correct iteration values $\usw_t$ at the rate
$$\|\usw_t - \isw_t\| = o\left(\frac{r}{n} + \frac{1}{B^{\frac{1}{4}}}\right).$$
Thus, when $B$ is large, and when $r/n$ is small, our algorithm accurately approximates the correct iteration values.

\end{theorem}

Its proof is in the Appendix (Section \ref{sec: sgd_proof}).

%% file: Sections/Experiment.tex
\tocless\section{Experiments}\label{sec: experiment}





\tocless\subsection{Experimental setup}
\label{exp}
\textbf{Datasets.}
We used four datasets for evaluation: \MNIST\  \cite{lecun1998gradient}, \covtype\ \cite{blackard1999comparative}, \higgs\ \cite{baldi2014searching} and \rcv\  \cite{lewis2004rcv1}  \footnote{We used its binary version from LIBSVM: 
\url{https://www.csie.ntu.edu.tw/~cjlin/libsvmtools/datasets/binary.html\#rcv1.binary}
}
. \MNIST\ contains 60,000 images as the training dataset and 10,000 images as the test dataset; each image has $28\times 28$ features (pixels), containing one  digit from 0 to 9. The \covtype\ dataset consists of 581,012 samples with 54 features, each of which may come from one of the seven forest cover types; as a test dataset, we randomly picked 10\% of the data. \higgs\ is a dataset produced by Monte Carlo simulations for binary classification, containing 21 features with 11,000,000 samples in total; 500,000 samples are used as the test dataset. \rcv\ is a corpus dataset; we use its binary version which consists of 679,641 samples and 47,236 features, of which the first 20,242 samples are used for training.

\textbf{Machine configuration.} All experiments are run over a GPU machine with one Intel(R) Core(TM) i9-9920X CPU with 128 GB DRAM and 4 GeForce 2080 Titan RTX GPUs (each GPU has 10 GB DRAM). We implemented \Increm\ with PyTorch 1.3 and used one GPU for accelerating the tensor computations.

\textbf{Deletion/Addition benchmark.}
We run regularized logistic regression over the four datasets with L2 norm coefficient 0.005, fixed learning rate 0.1. The mini-batch sizes for \rcv\ and other three datasets are 16384 and 10200 respectively (Recall that \rcv\ only has around 20k training samples).
We also evaluated our approach over a two-layer neural network with 300 hidden ReLU neurons over \MNIST. 
\yw{There L2 regularization with rate 0.001 is added along with a decaying learning rate (first 10 iterations with learning rate 0.2 and the rest with learning rate 0.1) 
and with deterministic GD. There are no strong convexity or smoothness guarantees for DNNs. Therefore, we adjusted Algorithm \ref{alg: update_algorithm} to fit general DNN models (see Algorithm \ref{alg: update_algorithm_general} in the Appendix \ref{sec: non_convex_alg}). In Algorithm \ref{alg: update_algorithm}, we assume that the convexity holds locally where we use the L-BFGS algorithm to estimate the gradients. For all the other regions, we explicitly evaluate the gradients. The details on how to check which regions satisfy the convexity for DNN models can be found in Algorithm \ref{alg: update_algorithm_general}. We also explore the use of \Increm\ for more complicated neural network models such as ResNet by reusing and fixing the pre-trained parameters in all but the last layer during the training phase, presented in detail in Appendix \ref{sec: large_model_exp}.}


\yw{We evaluate two cases of addition/deletion: \textit{batch} and \textit{online}. Multiple samples are grouped together for addition and deletion in the former, while samples are removed one after another in the latter. Algorithm \ref{alg: update_algorithm} is slightly modified to fit the online deletion/addition cases (see Algorithm \ref{alg: update_algorithm_online} in Appendix \ref{sec: online_del_add}).
In what follows, unless explicitly specified, Algorithm \ref{alg: update_algorithm} and Algorithm \ref{alg: update_algorithm_online} are used for experiments in the \textit{batch} addition/deletion case and \textit{online} addition/deletion case respectively.}

To simulate deleting training samples, $\ow$ is evaluated over the full training dataset of $n$ samples, which is followed by the random removal of $r$ samples and evaluation over the remaining $n-r$ samples using \Basel\ or \Increm. To simulate adding training samples, $r$ samples are deleted first. After $\ow$ is evaluated over the remaining $n-r$ samples, the $r$ samples are added back to the training set for updating the model. The ratio of $r$ to the total number of training samples $n$ is called the \textit{Delete rate} and \textit{Add rate} for the two scenarios, respectively.

Throughout the experiments, the running time of \Basel\ and \Increm\ to update the model parameters is recorded. To show the difference between $\ouw$ (the output of \Basel, and the correct model parameters after deletion or addition) and $\oiw$ (the output of \Increm), we compute the $\ell_2$-norm or distance $\|\ouw - \oiw\|$. For comparison and justifying the theory in Section \ref{converge_analysis}, $\|\ow-\ouw\|$ is also recorded ($\ow$ are the parameters trained over the full training data). Given the same set of added or deleted samples, the experiments are repeated 10 times, with different minibatch randomness each time. After the model updates, $\ouw$ and $\oiw$ are evaluated over the test dataset and their prediction performance is reported.

\yw{\textbf{Hyperparameter setup.} We set $T_0$ (the period of explicit gradient updates) and $j_0$ (the  length of the inital ``burn-in") as follows. For regularized logistic regression, we set $T_0=10, j_0=10$ for \rcv, $T_0=5, j_0 = 10$ for \MNIST\ and \covtype, and $T_0 = 3, j_0=300$ for \higgs. For the 2-layer DNN, $T_0=2$ is even smaller and the first quarter of the iterations are used as ``burn-in". The history size $m$ is 2 for all experiments. The effect of hyperparameters and suggestions on how to choose them is discussed in the Appendix  \ref{sec: other_para_influence}. }

\tocless\subsection{Experimental results}

\tocless\subsubsection{Batch addition/deletion.}

\begin{figure*}[ht]
\begin{center}
\centerline{\includegraphics[width=2\columnwidth,height = 0.45\columnwidth]{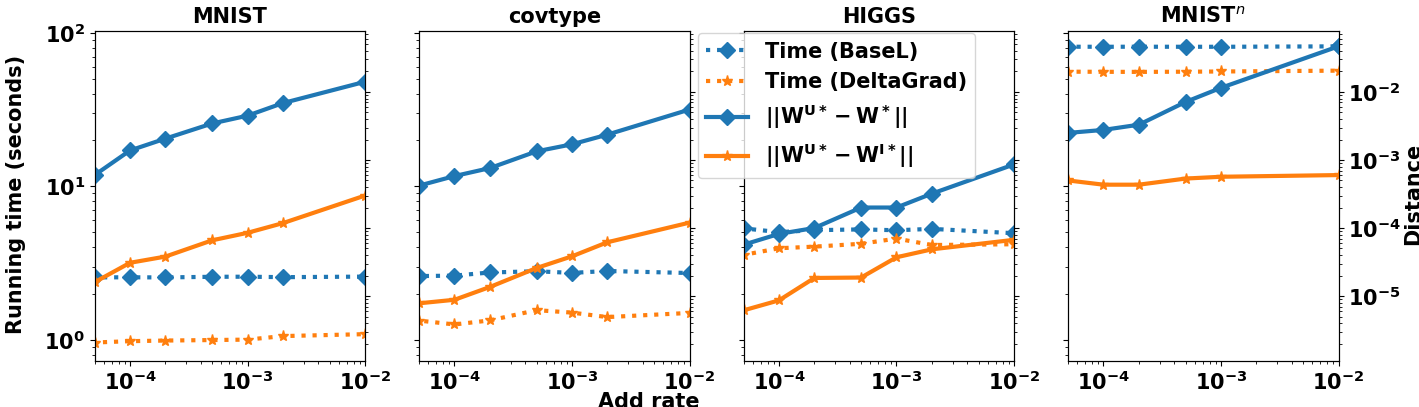}}
\caption{Running time and distance with varied add rate}
\label{fig: Time_distance_add}
\end{center}
\begin{center}
\centerline{\includegraphics[width=2\columnwidth,height = 0.45\columnwidth]{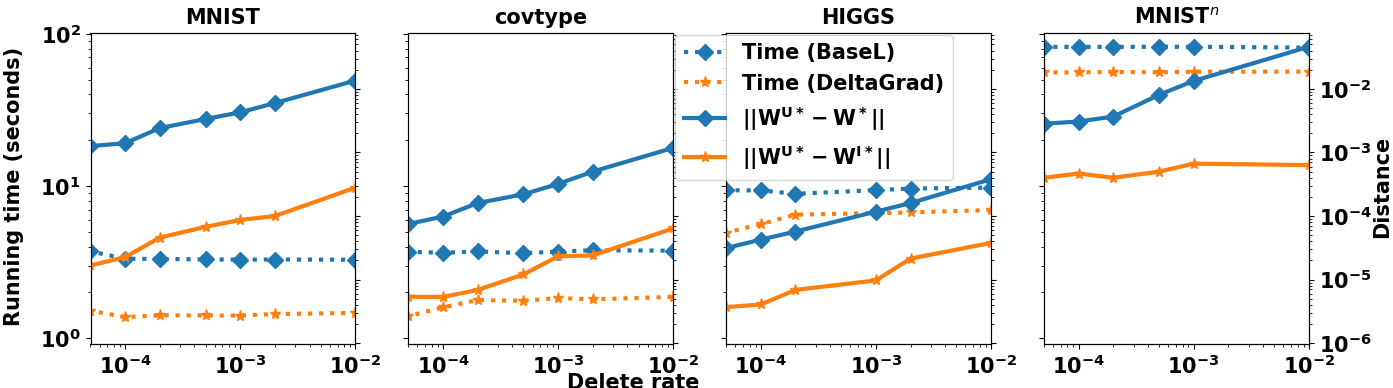}}
\caption{Running time and distance with varied delete rate}
\label{fig: Time_distance_del}
\end{center}
\end{figure*}

To test the robustness and efficiency of \Increm\ in batch deletion, we vary the \textit{Delete} and \textit{Add rate} from 0 to 0.01. The first three sub-figures in Figures \ref{fig: Time_distance_add} and \ref{fig: Time_distance_del} along with Figure \ref{fig: e} show the running time of \Basel\ and \Increm\ (blue and red dotted lines, resp.) and the two distances, $\|\ouw - \ow\|$ and $\|\ouw - \oiw\|$ (blue and red solid lines, resp.) over the four datasets using regularized logistic regression. The results on the use of 2-layer DNN over \MNIST\ are presented in the last sub-figures in Figures \ref{fig: Time_distance_add} and \ref{fig: Time_distance_del}, which are denoted by \MNIST$^n$.

The running time of \Basel\ and \Increm\ is almost constant regardless of the delete or add rate, confirming the time complexity analysis of \Increm\ in Section \ref{sec: complexity}. The theoretical running time is free of the number of removed samples $r$, when $r$ is small. For any given delete/add rate, \Increm\ achieves significant speed-ups (up to 2.6x for \MNIST, 2x for \covtype, 1.6x for \higgs, 6.5x for \rcv) compared to \Basel. On the other hand, the distance between $\ouw$ and $\oiw$ is quite small;  it is less than $0.0001$ even when up to 1\% of samples are removed or added.
When the delete or add rate is close to 0, $\|\ouw - \oiw\|$ is of magnitude $10^{-6}$ ($10^{-8}$ for \rcv), indicating that the approximation brought by $\oiw$ is negligible. Also, $\|\ouw - \oiw\|$ is at least one order of magnitude smaller than $\|\ouw - \ow\|$, confirming our theoretical analysis comparing the bound of $\|\ouw - \oiw\|$ to that of $\|\ouw - \ow\|$.

\begin{table}[t]
\centering
\small
\caption{Prediction accuracy of \Basel\ and \Increm\ with batch addition/deletion. \MNIST$^n$ refers to \MNIST\ with a neural net.}
\begin{tabular}[!h]{|>{\centering\arraybackslash}p{1cm}|>{\centering\arraybackslash}p{1cm}|>{\centering\arraybackslash}p{2.4cm}|>{\centering\arraybackslash}p{2.4cm}|} \hline
\multicolumn{2}{|c|}{Dataset} & \Basel (\%) & \Increm (\%) \\ \hline
\multirow{4}{*}{\makecell{Add\\(0.005\%)}}& \MNIST & $87.530 \pm 0.0025$ & 
$87.530 \pm 0.0025$ \\ \hhline{~---} %
& \MNIST$^n$ & $92.340\pm 0.002$ & $92.340\pm 0.002$ \\ \hhline{~---} 
& \covtype & $62.991 \pm 0.0027$& $62.991 \pm 0.0027$\\ \hhline{~---}
& \higgs & $55.372 \pm 0.0002$ & $55.372 \pm 0.0002 $ \\\hhline{~---}
& \rcv & $92.222 \pm 0.00004$ & $92.222 \pm 0.00004$ \\  \hline
\multirow{4}{*}{\makecell{Add\\(1\%)}}& \MNIST & $87.540 \pm 0.0011$ & 
$87.542 \pm 0.0011$ \\ \hhline{~---} %
& \MNIST$^n$ & $92.397\pm 0.001$ & $92.397\pm 0.001$ \\ \hhline{~---} 
& \covtype &$63.022 \pm 0.0008$ & $63.022 \pm 0.0008$\\ \hhline{~---}
& \higgs &$55.381\pm 0.0007$& $55.380 \pm 0.0007$\\ \hhline{~---}
& \rcv & $92.233 \pm 0.00010$ & $92.233 \pm 0.00010$\\  \hline
\multirow{4}{*}{\makecell{Delete \\ (0.005\%)}} & \MNIST &$86.272 \pm 0.0035 $& $86.272 \pm 0.0035 $\\ \hhline{~---}
& \MNIST$^n$ & $92.203\pm 0.004$ & $92.203\pm 0.004$ \\ \hhline{~---}
& \covtype& $62.966 \pm 0.0017 $ & $62.966 \pm 0.0017$ \\ \hhline{~---}
& \higgs &$52.950 \pm 0.0001 $& $52.950 \pm 0.0001 $\\\hhline{~---}
& \rcv & $92.241 \pm 0.00004$ & $92.241 \pm 0.00004$ \\ \hline
\multirow{4}{*}{\makecell{Delete\\ (1\%)}}& \MNIST &$86.082 \pm 0.0046$ & $86.074 \pm 0.0048$\\ \hhline{~---}
& \MNIST$^n$ & $92.373\pm 0.003$ & $92.370\pm 0.003$ \\ \hhline{~---}
& \covtype & $62.943 \pm 0.0007 $ & $62.943 \pm 0.0007$ \\ \hhline{~---}
& \higgs &$52.975 \pm 0.0002 $ & $52.975 \pm 0.0002$\\\hhline{~---}
& \rcv & $92.203 \pm 0.00007$ & $92.203 \pm 0.00007$ \\ \hline
\end{tabular}
\label{Table: Test_acc_comparison_batch_del_add}
\end{table}

To investigate whether the tiny difference between $\ouw$ and $\oiw$ will lead to any difference in prediction behavior, the prediction accuracy using $\ouw$ and $\oiw$ is presented in Table \ref{Table: Test_acc_comparison_batch_del_add}. Due to space limitations, only results on a very small ($0.005\%$) and the largest ($1\%$) add/delete rates are presented. Due to the randomness in SGD, the standard deviation for the prediction accuracy is also presented. In most cases, the models produced by \Basel\ and \Increm\ end up with effectively the same prediction power. There are a few cases where the prediction results of $\ouw$ and $\oiw$ are not exactly the same (e.g. Add (1\%) over \MNIST), their confidence intervals overlap, so that statistically $\ouw$ and $\oiw$ provide the same prediction results.

\yw{For the 2-layer net model where strong convexity does not hold, we use the variant of \Increm\~ mentioned above, i.e. Algorithm \ref{alg: update_algorithm_general}. See the last sub-figures in Figure \ref{fig: Time_distance_add} and \ref{fig: Time_distance_del}.}
The figures show that \Increm\ achieves about 1.4x speedup compared to \Basel\ while maintaining a relatively small difference between $\oiw$ and $\ouw$.
This suggests that it may be possible to extend our analysis for \Increm\ beyond strong convexity; this is left for future work.

\begin{table*}
\centering
\small
\caption{Distance and prediction performance of \Basel\ and \Increm\ in online deletion/addition}
\begin{tabular}[!h]{|>{\centering\arraybackslash}p{2.5cm}|>{\centering\arraybackslash}p{1.8cm}|>{\centering\arraybackslash}p{1.8cm}|>{\centering\arraybackslash}p{2.5cm}|>{\centering\arraybackslash}p{2.5cm}|} \hline
\multirow{2}{*}{\makecell{Dataset}} & \multicolumn{2}{c|}{\makecell{Distance}} & \multicolumn{2}{c|}{\makecell{Prediction accuracy (\%)}}\\ \hhline{~----}
 & $\|\ouw - \ow\|$&$\|\oiw - \ouw\|$ & \Basel& \Increm \\ \hline
\makecell{\MNIST\ (Addition)} & $5.7 \times 10^{-3}$& $2\times 10^{-4}$& $87.548 \pm 0.0002$ & $87.548 \pm 0.0002$
\\ \hline %
\makecell{\MNIST\ (Deletion)}&$5.0 \times 10^{-3}$& $1.4\times 10^{-4}$& $87.465 \pm 0.002$ & $87.465 \pm 0.002$
\\ \hline %
\makecell{\covtype\ (Addition)} & $8.0 \times 10^{-3}$&$2.0 \times 10^{-5}$ & $63.054\pm 0.0007$& $63.054\pm 0.0007$\\ \hline
\makecell{\covtype\ (Deletion)} & $7.0 \times 10^{-3}$&$2.0 \times 10^{-5}$ & $62.836\pm 0.0002$& $62.836\pm 0.0002$\\ \hline
\makecell{\higgs\ (Addition)} & $2.1 \times 10^{-5}$&$1.4 \times 10^{-6}$ & $55.303 \pm 0.0003$ & $55.303 \pm 0.0003$\\ \hline
\makecell{\higgs\ (Deletion)} & $2.5 \times 10^{-5}$&$1.7 \times 10^{-6}$ & $55.333 \pm 0.0008$ & $55.333 \pm 0.0008$\\ \hline
\makecell{\rcv\ (Addition)} & $0.0122$& $3.6\times 10^{-6}$ & $92.255 \pm 0.0003$ & $92.255 \pm 0.0003$ \\ \hline
\makecell{\rcv\ (Deletion)} & $0.0119$ &$3.5\times 10^{-6}$ & $92.229 \pm 0.0006$ & $92.229 \pm 0.0006$ \\ \hline
\end{tabular}
\label{Table: online_add_del_acc_comparison}
\end{table*}

\tocless\subsubsection{Online addition/deletion.}

To simulate deletion and addition requests over the training data continuously in an on-line setting, 100 random selected samples are added or deleted sequentially. Each triggers model updates by either \Basel\ or \Increm. The running time comparison between the two approaches in this experiment is presented in Figure \ref{fig: Time_continuous_deletion_addition}, which shows that \Increm\ is about 2.5x, 2x, 1.8x and 6.5x faster than \Basel\ on \MNIST, \covtype, \higgs\ and \rcv\ respectively. The accuracy comparison is shown in Table \ref{Table: online_add_del_acc_comparison}. There is essentially no prediction performance difference between $\ouw$ and $\ow$.

\textbf{Discussion.} By comparing the speed-ups brought by \Increm\ and the choice of $T_0$, we found that the theoretical speed-ups are not fully achieved. One reason is that in the approximate L-BFGS computation, a series of small matrix multiplications are involved. Their computation on GPU vs CPU cannot bring about very significant speed-ups compared to the larger matrix operations\footnote{See the matrix computation benchmark on GPU with varied matrix sizes: https://developer.nvidia.com/cublas}, which indicates that the overhead of L-BFGS is non-negligible compared to gradient computation. Besides, although $r$ is far smaller than $n$, to compute the gradients over the $r$ samples, other overhead becomes more significant: copying data from CPU DRAM to GPU DRAM, the time to launch the kernel on GPU, etc. This leads to non-negligible explicit gradient computation cost over the $r$ samples. It would be interesting to explore how to adjust \Increm\ to fully utilize the computation power of GPU in the future.

\begin{figure}[h]
\begin{center}
\centerline{\includegraphics[width=1\columnwidth,height = 0.4\columnwidth]{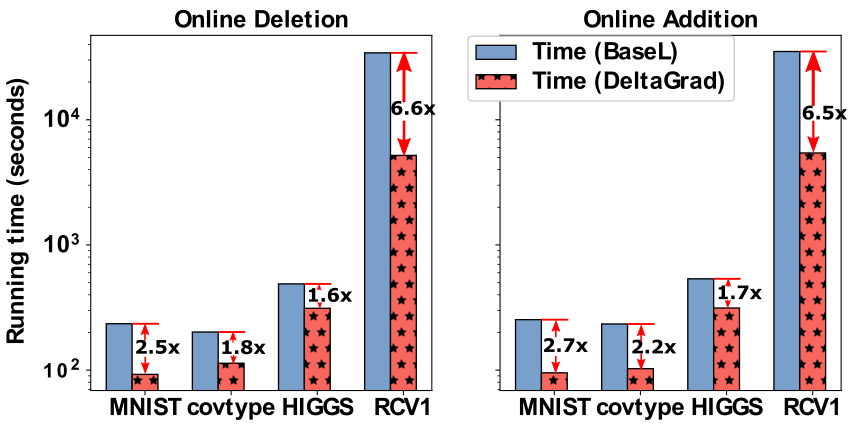}}
\caption{Running time comparison of \Basel\ and \Increm\ with 100 continuous deletions/addition}
\label{fig: Time_continuous_deletion_addition}
\end{center}
\end{figure}

Other experiments with \Increm\ are in the Appendix (Section \ref{sec: supple_exp}): evaluations with larger delete rate (i.e. when $r \ll n$ may not hold), comparisons with state-of-the-art work and studies on the effect of mini-batch sizes and hyper-parameters etc.

%% file: Sections/Applications.tex
\tocless\section{Applications}\label{sec: app}
Our algorithm has many applications, including privacy related data deletion, continuous model updating, robustness, bias reduction, and uncertainty quantification (predictive inference). Some of these applications are quite direct, and so for space limitations we only briefly describe them. \yw{Some initial experimental results on how our method can accelerate some of those applications such as robust learning are included in Appendix \ref{sec: exp_applications}.}

\tocless\subsection{Privacy related data deletion}

By adding a bit of noise one can often guarantee differential privacy, the impossibility to distinguish the presence or absence of a datapoint from the output of an algorithm \citep{dwork2014algorithmic}. We leverage and slightly extend a closely related notion, approximate data deletion, \cite{ginart2019making} to guarantee private deletion.

We will consider learning algorithms $A$ that take as input a dataset $D$, and output a model $A(D)$ in the hypothesis space $\mathcal{H}$. With the $i$-th sample removed, the resulting model is thus $A(D_{-i})$. A data deletion operation $R_A$ maps $D$, $A(D)$ and the index of the removed sample $i$ to the model $R_A(D, A(D), i)$.
We call $R_A$ an $\epsilon$-approximate deletion if for all $D$ and measurable subsets $S \subset \mathcal{H}$:
\vspace{-1mm}
\begin{align*}
    \begin{split}
    |\log\frac{P(A(D_{-i}) \in S | D_{-i})}{P(R_A(D,A(D),i)\in S | D_{-i})}| \leq \epsilon        
    \end{split}
\end{align*}
Here if either of the two probabilities is zero, the other must be zero too.
Using the standard Laplace mechanism \citep{dwork2014algorithmic}, we can make the output of our algorithm an $\epsilon$-approximate deletion. We add independent $Laplace$ $(\delta/\epsilon)$ noise to each coordinate of $\w^*$, $\uw^*$ and $\iw^*$, where 
\vspace{-1mm}
$$\delta 
= \frac{\sqrt{p}AM_1^2r^2}
{\eta(\frac{1}{2}\mu - \frac{r}{n-r}\mu - \frac{c_0 M_1r}{2n})^2(n-r)(n/2-r)}$$ 
is an upper bound on $p^{1/2}\|\uw^* - \iw^*\|$. See the appendix for details.

\tocless\subsection{Continous model updating}

Continous model updating is a direct application. In many cases, machine learning models run in production need to be retrained on newly acquired data. DeltaGrad can be used to update the models. Similarly, if there are changes in the data, then we can run DeltaGrad twice: first to remove the original data, then to add the changed data.

\tocless\subsection{Robustness}

Our method has applications to robust statistical learning. The basic idea is that we can identify outliers by fitting a preliminary model. Then we can prune them and re-fit the model. Methods based on this idea are some of the most statistically efficient ones for certain problems, see e.g., the review \citet{yu2017robust}.

\yw{
\tocless\subsection{Data valuation}
Our method can be also used to evaluate the importance of training samples (see \citet{cook1977detection} and the follow-up works such as \citet{ghorbani2019data}). One common method to do this is the {\em leave-one-out} test, i.e. comparing the difference of the model parameters before and after the deletion of one single training sample of interest. Our method is thus useful to speed up evaluating the model parameters after the deletion operations.
}

\tocless\subsection{Bias reduction}

Our algorithms can be used directly to speed up existing techniques for bias correction. There are many different techniques based on \emph{subsampling} \citep{politis1999subsampling}. A basic one is the \emph{jackknife} \citep{quenouille1956notes}. Suppose we have an estimator $\hat f_n$ computed based on $n$ training datapoints, and defined for both $n$ and $n-1$. The jackknife bias-correction is 
$\smash{\hat f_{jack}=\hat f_n-\hat{b}(\hat f_n)}$
where $\hat{b}(\hat f_n)$ is the jackknife estimator of the bias $b(\hat f_n)$ of the estimator $\hat f_n$. This is constructed as
$\hat{b}(\hat f_n) = (n-1)\left(n^{-1}\sum_{i=1}^n\hat f_{-i}-\hat f_n\right)$
where $\hat f_{-i}$ is the estimator $\hat f_{n-1}$ computed on the training data removing the $i$-th data point. Our algorithm can be used to recompute the estimator on all subsets  of size $n-1$ of the training data. To validate that this works, a good example may be logistic regression with $n$ not much larger than $p$, which will have bias \citep{sur2018modern}.

\tocless\subsection{Uncertainty quantification / Predictive inference}

Our algorithm has applications to uncertainty quantification and predictive inference. These are fundamental problems of wide applicability. Techniques based on conformal prediction \citep[e.g.,][]{shafer2008tutorial} rely on retraining models on subsets of the data. As an example, in cross-conformal prediction \citep{vovk2015cross} we have a predictive model $\hat f$ that can be trained on any subset of the data. We can split the data into $K$ subsets of roughly equal size. We can train $\hat f_{-S_k}$ on the data excluding $S_k$, and compute the cross-validation residuals $R_i = |\textbf{y}_i- \hat f_{-S_k}(\textbf{x}_i)|$ for $i\in S_i$. Then for a test datapoint $\textbf{x}_{n+1}$, we form a prediction set $C(\textbf{x}_{n+1})$ with all $\textbf{y}$ overlapping $n-(1-\alpha)(n+1)$ of the intervals $\hat f_{-S_k}(\textbf{x}_{n+1}) \pm R_i$. This forms a valid $\beta= 1-2\alpha-2K/n$ level prediction set in the sense that $P(\textbf{y}_{n+1} \in C(\textbf{x}_{n+1}))\ge\beta$ over the randomness in all samples. The "best" (shortest) intervals arise for large $K$, which means a small number of samples is removed to find $f_{-S_k}$. Thus our algorithm is applicable.


%% file: Sections/Conclusion.tex
\tocless\section{Conclusion}

\yw{
In this work, we developed the efficient  \Increm\ retraining algorithm after slight changes (deletions/additions) of the training dataset by differentiating the optimization path with Quasi-Newton method. This is provably more accurate than the baseline of retraining from scratch. Its performance advantage has been empirically demonstrated with some medium-scale public datasets, revealing its great potential in constructing data deletion/addition machine learning systems for various applications. The code for replicating our experiments is available on Github: \url{https://github.com/thuwuyinjun/DeltaGrad}. Adjusting \Increm\ to handle smaller mini-batch sizes in SGD and more complicated ML models without strong convexity and smoothness guarantees is important future work.
}

%% file: supplement_2.tex
\begin{center}
\textbf{\large Appendix for DeltaGrad: Rapid retraining of machine learning models}
\end{center}

%
%
%
%

\setcounter{tocdepth}{3}
\renewcommand*{\thesection}{\Alph{section}}
\tableofcontents
\section{Mathematical details}\label{sec: math}

The main result for \Increm\ with GD is Theorem \ref{main10}, proved in Section \ref{mainpf1}.

\subsection{Additional notes on setup, preliminaries}

\subsubsection{Classical results on GD convergence, SGD convergence}

\begin{lemma}[GD convergence, folklore, e.g., \cite{boyd2004convex}]\label{lemma: gd_convergence_strong_convex}
Gradient descent over a strongly convex objective function with fixed step size $\eta_t = \eta \leq \frac{2}{L+\mu}$ has exponential convergence rate, i.e.:
\begin{equation}
    F\left(\w_{t}\right) - F\left(\w^*\right) \leq c^t \frac{L}{2} ||\w_{0} - \w^*||^2
\end{equation}
where $c:=(L-\mu)/(L+\mu)<1$.
\end{lemma}

Recall also that the eigenvalues of the "contraction operator" $\textbf{I} - \eta_t \bH\left(\w\right)$ are bounded as follows.


\begin{lemma}[Classical bound on eigenvalues of the "contraction operator"]\label{eq: identity_hessian_bound}
Under the convergence conditions of gradient descent with fixed step size, i.e. $\eta_t = \eta \leq \frac{2}{\mu + L}$, the following inequality holds for any parameter $\w$:
\begin{equation}
        \|\textbf{I} - \eta \bH\left(\w\right)\| \leq 1.
\end{equation}
\end{lemma}
This lemma follows directly, because the eigenvalues of $I-\eta \bH$ are bounded between $-1\le 1-\eta L \le 1-\eta \mu \le 1$.

\begin{lemma}[SGD convergence, see e.g., \cite{bottou2018optimization}]\label{lemma: sgd}
Suppose that the stochastic gradient estimates are correlated with the true gradient, and bounded in the following way. There exist two scalars $J_1 \geq J_2 > 0$ such that for arbitrary $\miniB_{t}$, the following two inequalities hold:

\begin{align}
    & \nabla F\left(\w_{t}\right)^T \E\frac{1}{B_{t}}\sum_{i \in \miniB_{t}} \nabla F_i\left(\w_{t}\right) \geq J_2 \|\nabla F\left(\w_{t}\right)\|^2, \label{eq: sgd_exp_bound1}
\end{align}    
\begin{align*}
\|\E\frac{1}{B_{t}}\sum_{i \in \miniB_{t}} \nabla F_i\left(\w_{t}\right)\| \leq J_1 \|\nabla F\left(\w_{t}\right)\|.
\end{align*}

Also, assume that for two scalars $J_3, J_4 \geq 0$, we have:
\begin{align}
    Var\left(\frac{1}{B_{t}}\sum_{i \in \miniB_{t}} \nabla F_i\left(\w_{t}\right)\right) \leq J_3 + J_4 \|\nabla F\left(\w_{t}\right)\|^2.\label{eq: sgd_var_bound}
\end{align}

By combining equations \eqref{eq: sgd_exp_bound1}-\eqref{eq: sgd_var_bound}, the following inequality holds:

\begin{align*}
    \begin{split}
        \E\|\frac{1}{B_{t}}\sum_{i \in \miniB_{t}} \nabla F_i\left(\w_{t}\right)\|^2 \leq J_3 + J_5 \|\nabla F\left(\w_{t}\right)\|^2 
    \end{split}
\end{align*}

where $J_5 = J_4 + J_1^2 \geq J_2^2 \geq 0$.

Then stochastic gradient descent with fixed step size $\eta_t = \eta \leq \frac{J_2}{L J_5}$ has the convergence rate:

\begin{align*}
\begin{split}
    & \E\left[F\left(\w_{t}\right) - F\left(\w^*\right)\right]\leq \frac{\eta L J_3}{2\mu J_2} + \left(1-\eta\mu J_2\right)^{t-1}\left(F\left(\w_{1}\right) - F\left(\w^*\right) - \frac{\eta L J_3}{2\mu J_2}\right) \rightarrow \frac{\eta L J_3}{2\mu J_2}.
\end{split}
\end{align*}

\end{lemma}

If the gradient estimates are unbiased, then $\E\frac{1}{B_{t}}\sum_{i \in \miniB_{t}}$ $\nabla F_i\left(\w_{t}\right) $ $= \frac{1}{n}$  $\sum_{i=1}^n$ $\nabla F_i\left(\w_{t}\right)$ $ = \nabla F\left(\w_{t}\right)$ and thus $J_1 = J_2 = 1$. Moreover, $J_3\sim 1/m$, where $m$ is the minibatch size, because $J_2$ is the variance of the stochastic gradient.

So the convergence condition for fixed step size becomes $\eta_t = \eta \leq \frac{1}{LJ_5}$, in which $J_5 = J_4 +J_1^2 = J_4 + 1 \geq 1$. So $\eta_t = \eta \leq \frac{1}{LJ_5} \leq \frac{1}{L}$ suffices to ensure convergence.

\subsubsection{Notations for \Increm\ with SGD}
The SGD parameters trained over the full dataset, explicitly trained over the remaining dataset and incrementally trained over the remaining dataset are denoted by $\sw$, $\usw$ and $\isw$ respectively. Then given the mini-batch size $B$, mini-batch $\miniB_t$, the number of removed samples from each mini-batch $\Delta B_t$ and the set of removed samples $R$, the update rules for the three parameters are:

\begin{align}
         \sw_{t+1}& = \sw_t - \eta \frac{1}{B}\sum_{i\in \miniB_t} \nabla F_i(\sw_t) = \sw_t - \eta \sgrad(\sw_t),\label{eq: sw_update}\\
        \begin{split}
         \usw_{t+1}& = \usw_{t} - \eta \frac{1}{B - \Delta B_t} \sum_{i\in \miniB_t, i \not\in R} \nabla F_i(\usw_{t})\\
        & = \usw_{t} - \eta \sugrad(\usw_t),
        \end{split}\label{eq: usw_update}\\
        \isw_{t+1} = &
        \begin{cases}
        \isw_{t} - \frac{\eta }{B-\Delta B_t} \sum_{i\in \miniB_t, i\not\in R}\nabla F(\isw_t) & \parbox{4cm}{$(t-j_0) \mod T_0 = 0$ \\ or $t \leq j_0$} \\
        \parbox{10cm}{$\isw_{t} - \frac{\eta }{B-\Delta B_t}\{B[\B_{j_m}(\isw_t - \sw_t)+\frac{1}{B}\sum_{i\in \miniB_t} \nabla F_i(\sw_t)] - \sum_{i\in R, i\in \miniB_t} \nabla F(\isw_t)\}$} & otherwise \label{eq: isw_update}
        \end{cases}
\end{align}

in which $\sgrad(\sw_t)$ and $\sugrad(\usw_t)$ represent the average gradients over the minibatch $\mathcal{B}_t$ before and after removing samples.

We assume that the minibatch randomness of $\usw$ and $\isw$ is the same as $\sw$. By following Lemma \ref{lemma: sgd}, we assume that the gradient estimates of SGD are unbiased, i.e. $\E\left(\frac{1}{B_{t}}\sum_{i \in \miniB_{t}} \nabla F_i\left(\w\right)\right)$ $= \frac{1}{n} \sum_{i=1}^n \nabla F_i\left(\w\right) = \nabla F\left(\w\right)$ for any $\w$, which indicates that:

\begin{align*}
    & \E\left(\frac{1}{B}\sum_{i\in \miniB_t} \nabla F_i(\sw_t)\right) = \frac{1}{n} \sum_{i=1}^n \nabla F_i\left(\sw_t\right) = \nabla F\left(\sw_t\right),\\
    & \E\left(\frac{1}{B - \Delta B_t} \sum_{i\in \miniB_t, i \not\in R} \nabla F_i(\usw_{t})\right) = \frac{1}{n-\Delta n} \sum_{i \not\in R} \nabla F_i\left(\usw_t\right) = \nabla \uF\left(\usw_t\right).
\end{align*}

\subsubsection{Classical results for random variables}
To analyze \Increm\ with SGD, Bernstein’s inequality \citep{oliveira2009concentration, tropp2012user, tropp2016expected} is necessary. Both its scalar version and matrix version are stated below.

\begin{lemma}[Bernstein’s inequality for scalars]\label{lemma: Bernstein_inequality1}
Consider a list of independent random variables, $\textbf{S}_1,\textbf{S}_2,\dots,\textbf{S}_k$ satisfying $\E(\textbf{S}_i) = \textbf{0}$ and $|\textbf{S}_i| \leq J$, and their sum $\textbf{Z} = \sum_{i=1}^k \textbf{S}_i$. 
Then the following inequality holds:
\begin{align*}
    \begin{split}
        Pr(\|\textbf{Z}\| \geq x) \leq \exp\left({\frac{-x^2}{\sum_{i=1}^k \E(\textbf{S}_i^2) + \frac{Jx}{3}}}\right) , \forall x \geq 0.
\end{split}
\end{align*}

\end{lemma}

\begin{lemma}[Bernstein’s inequality for matrices]\label{lemma: Bernstein_inequality2}
Consider a list of independent $d_1\times d_2$ random matrices, $\textbf{S}_1,\textbf{S}_2,\dots,\textbf{S}_k$ satisfying $E(\textbf{S}_i) = \textbf{0}$ and $\|\textbf{S}_i\| \leq J$, and their sum $\textbf{Z} = \sum_{i=1}^k \textbf{S}_i$. Define the deterministic "varianc surrogate":
\begin{align}
    \begin{split}
        V(\textbf{Z}) = \max\left(\|\sum_{i=1}^k \E(\textbf{S}_i\textbf{S}_i^*)\|, \|\sum_{i=1}^k \E(\textbf{S}_i^*\textbf{S}_i)\|\right).
    \end{split}
\end{align}
Then the following inequalities hold:
\begin{align}
    \begin{split}\label{eq: Bernstein_ineq_1}
        Pr(\|\textbf{Z}\| \geq x) \leq (d_1 + d_2)\exp\left({\frac{-x^2}{V(\textbf{Z}) + \frac{Jx}{3}}}\right) , \forall x \geq 0,
\end{split}\\
\begin{split}\label{eq: Bernstein_ineq_2}
    \E(\|\textbf{Z}\|) \leq \sqrt{2 V(\textbf{Z})\log(d_1 + d_2)} + \frac{1}{3}J \log\left(d_1 + d_2\right).
\end{split}
\end{align}
\end{lemma}

\subsection{Results for deterministic gradient descent}\label{sec: gd_proof}

The main result for \Increm\ with GD is Theorem \ref{main10}, proved in Section \ref{mainpf1}.

\subsubsection{Quasi-Newton}\label{sec: quasi_newton}
\label{qn}
By following equations 1.2 and 1.3 in \cite{byrd1994representations}, the Quasi-Hessian update can be written as:
\begin{align}\label{eq: hessian_update}
    \begin{split}
        \B_{t+1} = \B_{t} - \frac{\B_{t}\Dw_{t} \Dw^{T}_{t} \B_{t}}{\Dw^{T}_{t} \B_{t} \Dw_{t}} + \frac{\Dg_{t} \Dg^{T}_{t}}{\Dg^{T}_{t} \Dw_{t}}.
    \end{split}
\end{align}

We have used the indices $k$ to index the Quasi-Hessians $\B_{j_k}$. This allows us to see that they correspond to the appropriate parameter gap $\Dw_{j_k}$ and gradient gap $\Dg_{j_k}$. The indices $j_k$ \emph{depend on the iteration number $t$} in the main algorithm, and they are updated by removing the ``oldest'' entry, and adding $T_0$ at every period. 

\Increm\ uses equation \eqref{eq: hessian_update} on the prior updates:
\begin{align}\label{eq: hessian_update2}
    \begin{split}
        \B_{j_{k+1}} = \B_{j_k} - \frac{\B_{j_k}\Dw_{j_k} \Dw_{j_k}^{T} \B_{j_k}}{\Dw_{j_k}^{T} \B_{j_k} \Dw_{j_k}} + \frac{\Dg_{j_k} \Dg_{j_k}^{T}}{\Dg_{j_k}^{T} \Dw_{j_k}},
    \end{split}
\end{align}
where the initialized matrix $\B_{j_{0}}$ is
$\B_{j_{0}} $ = $\Dg_{i_0}^{T} \Dw_{j_{0}}$ $/[\Dw_{i_0}^{T}\Dw_{j_{0}}] \textbf{I}$.

We use formulas 3.5 and 2.25 from \cite{byrd1994representations} for the Quasi-Newton method, with the caveat that they use slightly different notation.

For the update rule of $\B_{j_k}$, i.e.:
\begin{equation}\label{eq: B_formula}
    \B_{j_{k+1}} = \B_{j_k} - \frac{\B_{j_k}\Dw_{j_k} \Dw_{j_k}^{T} \B_{j_k}}{\Dw_{j_k}^{T} \B_{j_k} \Dw_{j_k}} + \frac{\Dg_{j_k} \Dg_{j_k}^{T}}{\Dg_{j_k}^{T} \Dw_{j_k}}.
\end{equation}

There is an equivalent expression for the inverse of $\B_{j_k}$ as below:
\begin{equation}\label{eq: B_inverse_formula}
    \B_{j_{k+1}}^{-1} = \left(\textbf{I} - \frac{\Dw_{j_k}\Dg_{j_k}^{T}}{\Dg_{j_k}^{T} \Dw_{j_k}}\right)\B_{j_k}^{-1}\left(\textbf{I} - \frac{\Dg_{j_k} \Dw_{j_k}^{T}}{\Dg_{j_k}^{T} \Dw_{j_k}}\right) + \frac{\Dw_{j_k}\Dw_{j_k}^{T}}{\Dg_{j_k}^{T} \Dw_{j_k}}.
\end{equation}

See Algorithm \ref{alg: lbfgs_algorithm} for an overview of the L-BFGS algorithm.
 
\begin{algorithm}[h!]
\small
 \SetKwInOut{Input}{Input}
 \SetKwInOut{Output}{Output}
 \Input{The sequence of the model parameter differences $\Delta W = \{\Dw_{0}, \Dw_{1}, \dots, \Dw_{m-1}\}$, the sequence of the gradient differences $\Delta G = \{\Dg_{0}, \Dg_{1}, \dots, \Dg_{m-1}\}$, a vector $\textbf{v}$, history size $m$}
 
 \Output{Approximate results of $\bH(w_m)\textbf{v}$ at point $w_m$, and for some $\textbf{v}$, such that $\Dw_{i}\approx w_i-w_{i-1}$ for all $i$}
 
 Compute $\Delta W^T \Delta W$
 
 Compute $\Delta W^T \Delta G$, get its diagonal matrix $D$ and its lower triangular submatrix $L$
 
 Compute $\sigma = \Dg_{m-1}^{T} \Dw_{m-1}/\left(\Dw_{m-1}^{T}\Dw_{m-1}\right)$
 
 Compute the Cholesky factorization for $\sigma \Delta W^{T}\Delta W + LDL^{T}$ to get $JJ^T$
 
 Compute $p = {\begin{bmatrix}-D^{\frac{1}{2}} & D^{-\frac{1}{2}}L^{T}\\
\textbf{0} & J^{T}
 \end{bmatrix}}^{-1}{\begin{bmatrix}D^{\frac{1}{2}} & \textbf{0}\\
D^{-\frac{1}{2}}L^{T}& J^{T}
 \end{bmatrix}}^{-1}\begin{bmatrix}\Delta G^{T}\textbf{v} \\ \sigma \Delta W^{T}\textbf{v}\end{bmatrix}$
 
 \Return $\sigma\textbf{v}-\begin{bmatrix}\Delta G & \sigma \Delta W \end{bmatrix} p$
 
\caption{Overview of L-BFGS algorithm}
 \label{alg: lbfgs_algorithm}
 \end{algorithm}

\subsubsection{Proof that Quasi-Hessians are well-conditioned}

We show that the Quasi-Hessian matrices computed by L-BFGS are well-conditioned.
\begin{lemma}[Bounds on Quasi-Hessians]\label{assp: B_K_product_bound}
The Quasi-Hessian matrices $\B_{j_k}$ are well-conditioned. There exist two positive constants $K_1$ and $K_2$ (depending on the problem parameters $\mu,L$, etc) such that for any $t$, any vector $\textbf{z}$, and all $k\in \{0,1$,$\ldots$,$m\}$, the following inequality holds:
$$
    K_1 \|\textbf{z}\|^2 \leq  \textbf{z}^T\B_{j_k}\textbf{z} \leq K_2 \|\textbf{z}\|^2.
$$
\end{lemma}


\begin{proof}

We start with the lower bound. Based on equation \eqref{eq: B_inverse_formula}, $\|\B_{j_k}^{-1}\|$ can be bounded by:
\begin{align}\label{eq: B_inverse_bound}
    \begin{split}
    \|\B_{j_{k+1}}^{-1}\|& \leq \|\textbf{I} - \frac{\Dw_{j_k}\Dg_{j_k}^{T}}{\Dg_{j_k}^{T} \Dw_{j_k}}\|\cdot \|\B_{j_k}^{-1}\|
     \cdot \|\textbf{I} - \frac{\Dg_{j_k} \Dw_{j_k}^{T}}{\Dg_{j_k}^{T} \Dw_{j_k}}\| + \|\frac{\Dw_{j_k}\Dw_{j_k}^{T}}{\Dg_{j_k}^{T} \Dw_{j_k}}\|.
    \end{split}
\end{align}

in which by using the mean value theorem, $\|\textbf{I} - \frac{\Dw_{j_k}\Dg_{j_k}^{T}}{\Dg_{j_k}^{T} \Dw_{j_k}}\|$ can be bounded as:

\begin{align}\label{eq: B_inverse_bound_deriv_1}
    \begin{split}
        & \|\textbf{I} - \frac{\Dw_{j_k}\Dg_{j_k}^{T}}{\Dg_{j_k}^{T} \Dw_{j_k}}\| \leq 1 + \frac{\|\Dw_{j_k}\Dg_{j_k}^{T}\|}{\Dg_{j_k}^{T} \Dw_{j_k}} \\
        & = 1 + \frac{\|\Dw_{j_k}(\bH_{j_k}\Dw_{j_k})^T\|}{\Dw_{j_k}^T\bH_{j_k}\Dw_{j_k}} \leq 1 + \frac{\|\Dw_{j_q}\|\|\bH_{j_k}\|\|\Dw_{j_q}\|}{\mu \|\Dw_{j_q}\|^2} \leq 1 + \frac{L}{\mu}.
    \end{split}
\end{align}

In addition, $\|\frac{\Dw_{j_k}\Dw_{j_k}^{T}}{\Dg_{j_k}^{T} \Dw_{j_k}}\|$ can be bounded as:

\begin{align}\label{eq: B_inverse_bound_deriv_2}
    \begin{split}
        \|\frac{\Dw_{j_k}\Dw_{j_k}^{T}}{\Dg_{j_k}^{T} \Dw_{j_k}}\| = \|\frac{\Dw_{j_k}\Dw_{j_k}^{T}}{\Dw_{j_k}^{T} \bH_{j_k} \Dw_{j_k}}\| \leq \|\frac{\Dw_{j_k}^{T}\Dw_{j_k}}{\mu  \Dw_{j_k}^{T}\Dw_{j_k}}\| = \frac{1}{\mu }.
    \end{split}
\end{align}

So by combining Equation \eqref{eq: B_inverse_bound_deriv_1} and Equation \eqref{eq: B_inverse_bound_deriv_2}, Equation \eqref{eq: B_inverse_bound} can be bounded by:

\begin{align*}
    \begin{split}
        \|\B_{j_{k+1}}^{-1}\|
    & \leq (1+\frac{L}{\mu})^2\|\B_{j_k}^{-1}\| + \frac{1}{\mu } \leq (1+\frac{L}{\mu})^{2k}\|\B_{j_{0}}^{-1}\| + \frac{1-(1+\frac{L}{\mu})^{2k}}{1-(1+\frac{L}{\mu})^2} \frac{1}{\mu}\\
    & = (1+\frac{L}{\mu})^{2k}\frac{L}{\mu} + \frac{1-(1+\frac{L}{\mu})^{2k}}{1-(1+\frac{L}{\mu})^2} \frac{1}{\mu}.
    \end{split}
\end{align*}

which thus implies that $\|\B_{j_k}\| \geq K_1:=\frac{1}{(1+\frac{L}{\mu})^{2k}\frac{L}{\mu} + \frac{1-(1+\frac{L}{\mu})^{2k}}{1-(1+\frac{L}{\mu})^2} \frac{1}{\mu}}$ where $0 \leq k \leq m$. Recall that $m$ is small, (set as $m=2$ in the experiments). So the lower bound will not approach zero.

Then based on Equation \eqref{eq: hessian_update}, we derive an upper bound for $\|\B_{j_k}\|$ as follows:
\begin{align*}
        & \textbf{z}^T \B_{j_{k+1}} \textbf{z} = \textbf{z}^T \B_{j_k} \textbf{z} - \frac{\textbf{z}^T \B_{j_k}\Dw_{j_k} \Dw_{j_k}^{T} \B_{j_k} \textbf{z}}{\Dw_{j_k}^{T} \B_{j_k} \Dw_{j_k}} + \frac{\textbf{z}^T \Dg_{j_k} \Dg_{j_k}^{T}\textbf{z}}{\Dg_{j_k}^{T} \Dw_{j_k}}  \\
        & \leq \textbf{z}^T \B_{j_k} \textbf{z} + \frac{\textbf{z}^T \Dg_{j_k} \Dg_{j_k}^{T}\textbf{z}}{\Dg_{j_k}^{T} \Dw_{j_k}} = \textbf{z}^T \B_{j_k} \textbf{z} + \frac{\textbf{z}^T \bH_{j_k} \Dw_{j_k} \Dw_{j_k}^{T}\bH_{j_k}\textbf{z}}{\Dw_{j_k}^{T} \bH_{j_k} \Dw_{j_k}} \\
        & \leq \textbf{z}^T \B_{j_k} \textbf{z} + \frac{\textbf{z}^T \bH_{j_k} \textbf{z} \Dw_{j_k}^{T}\bH_{j_k}\Dw_{j_k}}{\Dw_{j_k}^{T} \bH_{j_k} \Dw_{j_k}} = \textbf{z}^T \B_{j_k} \textbf{z} + \textbf{z}^T \bH_{j_k} \textbf{z} \\
        & \leq \textbf{z}^T \B_{j_k} \textbf{z} + L \|\textbf{z}\|^2. 
\end{align*}

The first inequality uses the fact that $\textbf{z}^T \B_{j_k}\Dw_{j_k} \Dw_{j_k}^{T} \B_{j_k} \textbf{z} = \left(\textbf{z}^T \B_{j_k}\Dw_{j_k}\right)^2 \geq 0$ and $\Dw_{j_k}^{T} \B_{j_k} \Dw_{j_k}$ $\geq 0$, due to the positive definiteness of $\B_{j_k}$. The second inequality uses the Cauchy-Schwarz inequality for the Quasi-Hessian, i.e.:
$$\left(\textbf{a}^T \bH_{j_k}\textbf{b}\right)^2 \leq \left(\textbf{a}^T \bH_{j_k}\textbf{a}\right) \left(\textbf{b}^T \bH_{j_k}\textbf{b}\right).$$

By applying the formula above recursively, we get $\textbf{z}^T \B_{j_{k+1}} \textbf{z} \leq (k+1) L \|\textbf{z}\|^2$ where $0 \leq k \leq m$. Again, as $m$ is bounded, so we have $(k+1) L \le K_2:=(m+1)L$. This finishes the proof.
\end{proof}

\subsubsection{Proof preliminaries}

First of all, we provide the bound on $\delta_t$, which is defined as:
\begin{lemma}[Upper bound on $\delta_t$]\label{lemma: delta_t_bound} By defining $$\delta_t = - 
        \frac{\eta}{n-r} 
        \left(\frac{r}{n}\sum_{i=1}^n \nabla F_i\left(\uw_{t}\right) 
        - \sum_{\substack{i \in R}}\nabla F_i\left(\uw_{t}\right)\right),$$ 
we then have $\|\delta_{t}\| \leq 2c_2 \frac{r\eta}{n}.$
\label{updel}
\end{lemma}

\begin{proof}
Based on the definition of $\delta_{t}$, we can rearrange it a little bit as:
\begin{align*}
    \begin{split}
        & \|\delta_{t}\| = \|-\frac{\eta r}{n\left(n-r\right)}\sum_{i=1}^n \nabla F_i\left(\uw_{t}\right) + \frac{\eta}{n-r} \sum_{\substack{i \in R}} \nabla F_i\left(\uw_{t}\right)\|\\
        & = \|-\frac{\eta r}{n\left(n-r\right)}[\sum_{i=1}^n \nabla F_i\left(\uw_{t}\right) - \sum_{\substack{i \in R}} \nabla F_i\left(\uw_{t}\right)]+ (\frac{\eta}{n-r} - \frac{\eta r}{n(n-r)}) \sum_{\substack{i \in R}} \nabla F_i\left(\uw_{t}\right)\|\\
        & = \|-\frac{\eta r}{n\left(n-r\right)}\sum_{i\not\in R} \nabla F_i\left(\uw_{t}\right) + \frac{\eta}{n} \sum_{\substack{i \in R}} \nabla F_i\left(\uw_{t}\right)\|.
    \end{split}
\end{align*}

Then by using the triangle inequality and Assumption \ref{assp: gradient_upper_bound} (bounded gradients), the formula above can be bounded as:
\begin{align*}
    \begin{split}
        & \leq \frac{\eta r}{n\left(n-r\right)}\sum_{i\not\in R} \|\nabla F_i\left(\uw_{t}\right)\| + \frac{\eta}{n} \sum_{\substack{i \in R}} \|\nabla F_i\left(\uw_{t}\right)\|\leq \frac{\eta r}{n}c_2 + \frac{\eta r}{n} c_2 = \frac{2\eta r}{n}c_2 
    \end{split}
\end{align*}


\end{proof}



Notice that Algorithm \ref{alg: lbfgs_algorithm} requires $2m$ vectors as the input, i.e. $[\Dw_{j_{0}}$, $\Dw_{j_1}$,$\dots$,$ \Dw_{j_{m-1}}]$ and $[\Dg_{j_{0}}, \Dg_{j_1}$, $\dots$, $\Dg_{j_{m-1}}]$ to approximate the product of the Hessian matrix $\bH(w_t)$ and the input vector $\Dw_{t}$ at the $t_{th}$ iteration where $j_{m-1} \leq t \leq j_{m-1} + T_0$.

Note that by multiplying $\Dw_{j_k}$ on both sides of the Quasi-Hessian update Equation \eqref{eq: hessian_update2}, we have the classical \emph{secant equation} that characterizes Quasi-Newton methods as below:
\begin{equation}\label{eq: b_k_g_k_equation} 
    \B_{j_{k+1}}\Dw_{j_k} = \Dg_{j_k}.
\end{equation}

Then we give a bound on the quantity $\|\Dg_{j_k} - \B_{j_q}\Dw_{j_k}\|$ where the intermediate index $q$ is in between the "correct" index $k+1$ and the final index $m$,  so  $m \geq q \geq k+1$. This characterizes the error by using a different Quasi-Hessian at some iteration. Its proof borrows ideas from \cite{conn1991convergence}. Unlike \cite{conn1991convergence}, our proof relies on a preliminary estimate on the bound on $\|\w_t - \iw_t\|$, which is at the level of $O(\frac{r}{n})$. The proof of the bound will be presented later.


\begin{theorem}\label{theorem: delta_model_para_bound}
Suppose that the preliminary estimate is: $\|\w_{j_k} - \iw_{j_k}\| \leq \frac{1}{\frac{1}{2}-\frac{r}{n}}M_1\frac{r}{n}$, where $k=1,2,\dots,m$ and $M_1 = \frac{2c_2}{\mu}$. Let $e = \frac{L(L+1) + K_2 L}{\mu  K_1}$, for the upper and lower bounds $K_1,K_2$ on the eigenvalues of the quasi-Hessian from Lemma \ref{assp: B_K_product_bound}, for the upper bounds $c_2$ on the gradient from Assumption \ref{assp: gradient_upper_bound} and for the Lipshitz constant $c_0$ of the Hessian. For $1 \leq k+1 \leq q \leq m$,  we have:
$$\|\bH_{j_k} - \bH_{j_q}\| \leq c_0 d_{j_k, j_q} + c_0 \frac{1}{\frac{1}{2}-\frac{r}{n}}M_1 \frac{r}{n}$$ and 
$$\|\Dg_{j_k} - \B_{j_q}\Dw_{j_k}\| \leq \left[(1+e)^{q-k-1} - 1\right]\cdot c_0 (d_{j_k,j_q} + \frac{1}{\frac{1}{2}-\frac{r}{n}}M_1\frac{r}{n}) \cdot s_{j_1,j_m},$$
 where $s_{j_1,j_m} = \max\left(\|\Dw_{a}\|\right)_{a=j_1,j_{2},\dots,j_m}$ and $d$ is defined as the maximum gap between the steps of the algorithm over the iterations from $j_k$ to $j_q$:
\beq\label{d}
d_{j_k,j_q} = \max\left(\|\w_{a} - \w_{b}\|\right)_{j_k \leq a \leq b \leq j_q}.
\eeq
\end{theorem}

\begin{proof}
Let $v_q=\Dg_{j_k} - \B_{j_{q+1}}\Dw_{j_k}$, $b_q=\|v_{q}\|$ and $f=c_0 (d_{j_1,j_m +T_0 -1} + \frac{1}{\frac{1}{2}-\frac{r}{n}}M_1\frac{r}{n})s_{j_1,j_m}$.

Let us bound the difference between the averaged Hessians $\|\bH_{j_k} - \bH_{j_q}\|$, where $1 \leq k < q \leq m$, using their definition, as well as using Assumption \ref{assp: hessian_continuous} on the Lipshitzness of the Hessian:
\begin{align}\label{eq: hessian_diff_bound0}
    \begin{split}
       & \|\bH_{j_k} - \bH_{j_q}\|\\
       & = \|\int_0^1 [\bH(\w_{j_k} + x (\iw_{j_k} - \w_{j_k}))]dx - \int_0^1 [\bH(\w_{j_q} + x (\iw_{j_q} - \w_{j_q}))]dx\|\\
       & = \| \int_0^1 [\bH(\w_{j_k} + x (\iw_{j_k} - \w_{j_k}))-\bH(\w_{j_q} + x (\iw_{j_q} - \w_{j_q}))]dx \|\\
       & \leq c_0 \int_0^1 \|\w_{j_k} + x (\iw_{j_k} - \w_{j_k}) - [\w_{j_q} + x (\iw_{j_q} - \w_{j_q})]\| dx \\
       & \leq c_0\| \w_{j_k} - \w_{j_q}\| + \frac{c_0}{2}\|\iw_{j_k} - \w_{j_k} - (\iw_{j_q} - \w_{j_q})\| \\
       & \leq c_0 \|\w_{j_k} - \w_{j_q}\| + \frac{c_0}{2}\|\w_{j_q}- \iw_{j_q}\| + \frac{c_0}{2}\|\iw_{j_k} - \w_{j_k}\|\\
       & \leq c_0 d_{j_k, j_q} + \frac{c_0}{\frac{1}{2}-\frac{r}{n}}M_1 \frac{r}{n} \leq c_0 d_{j_1, j_m +T_0 - 1} + \frac{c_0}{\frac{1}{2}-\frac{r}{n}}M_1 \frac{r}{n}.
    \end{split}
\end{align}
On the last line, we used the definition of $d_{j_k, j_q}$, and the assumption on the boundedness of $\|\iw_{j_k} - \w_{j_k}\|$.

Then, when $q=k$, according to Equation \eqref{eq: b_k_g_k_equation}, the secant equation $\Dg_{j_k} = \B_{j_{k+1}}\Dw_{j_k}$ holds. So $\|\Dg_{j_k} - \B_{j_{k+1}}\Dw_{j_k}\| = 0$, which proves the claim when $q=k$. So $v_{q}=b_{q}=0$.

Next, let $u_q = \Dg_{j_q} - \B_{j_q}\Dw_{j_q}$. This quantity is closely related to $v_{q-1}=\Dg_{j_k} - \B_{j_{q}}\Dw_{j_k}$, and the difference is that in $u_q$, the $\Dg, \Dw$ terms are defined at $q$, as opposed to the base one at $k$.   Then $|u_q^T \Dw_{j_k}|$, where $q> k$, can be bounded as:
\begin{align}\label{eq: delta_w_product}
\begin{split}
    &|u_q^T \Dw_{j_k}| \\
    & = |\Dg_{j_q}^{T}\Dw_{j_k} - \Dg_{j_k}^{T}\Dw_{j_q} + \Dg_{j_k}^{T}\Dw_{j_q} - \Dw_{j_q}^{T}\B_{j_q}\Dw_{j_k} | \\
    & \leq |\Dg_{j_q}^{T}\Dw_{j_k} - \Dg_{j_k}^{T}\Dw_{j_q}| + |\Dw_{j_q}^{T}v_{q-1}| \\
    & \leq |\Dg_{j_q}^{T}\Dw_{j_k} - \Dg_{j_k}^{T}\Dw_{j_q}| + \|\Dw_{j_q}\|\cdot b_{q-1} \\
    & = |\Dw_{j_q}^{T} \bH_{j_q}\Dw_{j_k} - \Dw_{j_k}^{T} \bH_{j_k} \Dw_{j_q}|+ \|\Dw_{j_q}\|\cdot b_{q-1} \\
    & = |\Dw_{j_q}^{T} \left(\bH_{j_q} -\bH_{j_k}\right)\Dw_{j_k}| + \|\Dw_{j_q}\|\cdot b_{q-1} \\
    & \leq \|\Dw_{j_q}\|\cdot \|\bH_{j_q} -\bH_{j_k}\|\cdot \|\Dw_{j_k}\| + \|\Dw_{j_q}\|\cdot b_{q-1} \\
    & \leq \left(f + b_{q-1}\right)\|\Dw_{j_q}\|,
\end{split}
\end{align}

in which the first inequality uses the triangle inequality, the second inequality uses the Cauchy-Schwarz inequality, and the subsequent equality uses the Cauchy mean value theorem. Finally, the third inequality uses Assumption \ref{assp: hessian_continuous} and equation \eqref{eq: hessian_diff_bound0}. We also use the following bounds, which hold by definition (notice that $k,q\le m$):
\begin{align*}
\|w_{j_k} -w_{j_q}\| \le d_{j_k,j_q}\qquad\qquad
\|\Dw_{j_q}\| \le s_{j_1,j_m}.
\end{align*}

The argument on the upper bound of $b_q$ will proceed by induction. The claim is true for the base case $q=k$. Assuming that the claim is true for $q-1$, we want to prove it for $q$, which is bounded as below:

\begin{align}\label{eq: delta_gap_1}
    \begin{split}
        &b_{q} 
        =\|\Dg_{j_k} - \left(\B_{j_q} - \frac{\B_{j_q}\Dw_{j_q} \Dw_{j_q}^{T} \B_{j_q}}{\Dw_{j_q}^{T} \B_{j_q} \Dw_{j_q}} + \frac{\Dg_{j_q} \Dg_{j_q}^{T}}{\Dg_{j_q}^{T} \Dw_{j_q}}\right)\Dw_{j_k}\|.
    \end{split}
\end{align}
By using the triangle inequality, we obtain the following upper bound:
\begin{align*}
    \begin{split}
        & \leq b_{q-1} + \|\left(\frac{\Dg_{j_q} \Dg_{j_q}^{T}}{\Dg_{j_q}^{T} \Dw_{j_q}} - \frac{\B_{j_q}\Dw_{j_q} \Dw_{j_q}^{T} \B_{j_q}}{\Dw_{j_q}^{T} \B_{j_q} \Dw_{j_q}}\right)\Dw_{j_k}\|.
    \end{split}
\end{align*}
Now we come to a key and nontrivial step of the argument. By bringing fractions to the common denominator in the second term, adding and subtracting $\Dg_{j_q}\Dg_{j_q}^{T} \Dw_{j_q}^{T}\Dg_{j_q}$ and $\Dg_{j_q}(\B_{j_q}\Dw_{j_q})^T\Dw_{j_q}^{T} \Dg_{j_q}$, and rearranging to factor out the term $-u_q$ in the numerator of each summand, the formula above can be rewritten as:
\begin{align*}
    \begin{split}
        & = b_{q-1} + \frac{\|[-\Dg_{j_q}\Dg_{j_q}^{T} \Dw_{j_q}^{T}u_q + \Dg_{j_q} u_q^T\Dw_{j_q}^{T} \Dg_{j_q}
        + u_q\Dw_{j_q}^{T}\B_{j_q}\Dw_{j_q}^{T}\Dg_{j_q}]\Dw_{j_k}\| }{\Dg_{j_q}^{T} \Dw_{j_q}\Dw_{j_q}^{T} \B_{j_q} \Dw_{j_q}}.
    \end{split}
\end{align*}
Next, using the Cauchy mean value theorem, and the fact that the smallest eigenvalues of $\bH_{j_q},\B_{j_q}$ are lower bounded by $\mu,K_1$ respectively, the formula above is bounded as:
\begin{align*}
    \begin{split}
        & \leq b_{q-1} + \frac{\|[-\Dg_{j_q}\Dg_{j_q}^{T} \Dw_{j_q}^{T}u_q + \Dg_{j_q} u_q^T\Dw_{j_q}^{T} \Dg_{j_q}
         + u_q\Dw_{j_q}^{T}\B_{j_q}\Dw_{j_q}^{T}\Dg_{j_q}]\Dw_{j_k}\|}
         {\mu K_1 \|\Dw_{j_q}\|^4}\\
        & \leq b_{q-1} + 
        (\|\Dg_{j_q}\|^2 \cdot \|\Dw_{j_q}^{T}u_q\Dw_{j_k}\|
         + \|\Dg_{j_q}\|\cdot \|u_q^T\Dw_{j_q}^{T} \Dg_{j_q}\Dw_{j_k}\|\\
        &+ \|u_q\Dw_{j_q}^{T}\B_{j_q}\Dw_{j_k}\Dw_{j_q}^{T}\Dg_{j_q} \|)/
        \mu K_1 \|\Dw_{j_q}\|^4.
    \end{split}
\end{align*}
Now we want to bound the last three terms one by one. First of all, $\|\Dg_{j_q}\|^2\|\Dw_{j_q}^{T}u_q\Dw_{j_k}\|$ can be bounded as:
\begin{align*}
    & \|\Dg_{j_q}\|^2\cdot \|\Dw_{j_q}^{T}u_q\Dw_{j_k}\|  = \|\bH_{j_q} \Dw_{j_q}\|^2 \cdot |\Dw_{j_q}^{T}u_q |\cdot \|\Dw_{j_q}\| \\
    & \leq L \|\Dw_{j_q}\|^3 \cdot |\Dw_{j_q}^{T}u_q |  \leq L \left(f + b_{q-1}\right) \|\Dw_{j_q}\|^4,
\end{align*}

in which the first equality uses the Cauchy mean value theorem, the subsequent inequality uses Assumption \ref{assp: gradient_upper_bound} and the last inequality uses equation \eqref{eq: delta_w_product}, the upper bound on $|\Dw_{j_q}^{T}u_q|$.

Then for $\|\Dg_{j_q}\|\cdot \|u_q^T\Dw_{j_q}^{T} \Dg_{j_q}\Dw_{j_k}\|$, we have a very similar argument. The only difference is that we factor out the scalar $\Dw_{j_q}^{T} \Dg_{j_q}$, and bound it by $L \|\Dw_{j_q}\|^2$, i.e.:

\begin{align*}
        & \|\Dg_{j_q}\|\cdot \|u_q^T\Dw_{j_q}^{T} \Dg_{j_q}\Dw_{j_k}\|\\
        & = \|\bH_{j_q}\Dw_{j_q}\|\cdot |\Dw_{j_q}^{T} \Dg_{j_q}| \cdot |u_q^T\Dw_{j_k}| \\
        & \leq L^2 \left(f + b_{q-1}\right) \|\Dw_{j_q}\|^4,
\end{align*}

in which the first equality uses Cacuhy mean value theorem and the fact that $\Dw_{j_q}^{T} \Dg_{j_q}$ is a scalar and the last inequality uses Assumption \ref{assp: gradient_upper_bound} and Equation \eqref{eq: delta_w_product}.

In terms of the bound on $\|u_q\Dw_{j_q}^{T}\B_{j_q}\Dw_{j_k}\Dw_{j_q}^{T}\Dg_{j_q}\|$, it is derived as:

\begin{align*}
    \begin{split}
        & \|u_q\Dw_{j_q}^{T}\B_{j_q}\Dw_{j_k}\Dw_{j_q}^{T}\Dg_{j_q}\| \\
        & = \|u_q\Dw_{j_q}^{T}\B_{j_q}\Dw_{j_k}\Dw_{j_q}^{T}\Dg_{j_q}\| \\
        & \leq \|u_q\Dw_{j_q}^{T}\|\cdot |\Dw_{j_q}^{T}\B_{j_q}\Dw_{j_k}| \cdot\|\Dg_{j_q}\|\\
        & \leq \left(f + b_{q-1}\right)\|\Dw_{j_q}\|\cdot |\Dw_{j_q}^{T}\B_{j_q}\Dw_{j_k}|\cdot\|\bH_{j_q}\Dw_{j_q}\|\\
        & \leq \left(f + b_{q-1}\right)\|\Dw_{j_q}\|\cdot K_2 \|\Dw_{j_q}\|^2 \cdot L \|\Dw_{j_q}\| \\
        & = K_2 L \left(f + b_{q-1}\right) \|\Dw_{j_q}\|^4 ,
    \end{split}
\end{align*}

in which the first inequality uses the Cauchy Schwarz inequality, the second inequality uses equation \eqref{eq: delta_w_product} and the third inequality uses Assumption \ref{assp: B_K_product_bound}.

In summary, for all $j\ge t+1$, Equation \eqref{eq: delta_gap_1} is bounded by:
\begin{align*}
    \begin{split}
        &b_q
         \leq b_{q-1}
        + \frac{L(L+1) + K_2 L}{\mu  K_1 \|\Dw_{j_q}\|^4} \left(f + b_{q-1}\right) \|\Dw_{j_q}\|^4 \\
        & = (1+e)b_{q-1} + ef.
    \end{split}
\end{align*}
By recursion and using the fact that $b_{k}=0$, this can be bounded as:
\begin{align}\label{eq: delta_w_product_final}
    \begin{split}
        & \leq \left(1+e\right)^{q-k} b_{k+1} 
        + \sum_{i=0}^{q-k - 1} \left(1+e\right)^{i} e \cdot f \\
        & = \frac{(1+e)^{q-k} - 1}{e} \cdot e f  = [(1+e)^{q-k} - 1] f.
    \end{split}
\end{align}
This proves the required claim $b_q \le [(1+e)^{q-k} - 1] f$ and finishes the proof.
\end{proof}

\begin{corollary}[Approximation accuracy of quasi-Hessian to mean Hessian]\label{corollary: approx_hessian_real_hessian_bound}
Suppose that $\|\w_{j_s} - \iw_{j_s}\| \leq \frac{1}{\frac{1}{2}-\frac{r}{n}}M_1\frac{r}{n}$ and $\|\w_{t} - \iw_{t}\| \leq \frac{1}{\frac{1}{2}-\frac{r}{n}}M_1\frac{r}{n}$ where $s = 1,2,\dots, m$. Then for $ j_m \le t \leq j_m + T_0 - 1,$
\beq\label{xi2}
\|\bH_{t} - \B_{j_m}\| \leq \xi_{j_1,j_m}:= A d_{j_1,j_m + T_0 - 1}
+ A \frac{1}{\frac{1}{2}-\frac{r}{n}}M_1\frac{r}{n},
\eeq 
where recall again that $c_0$ is the Lipschitz constant of the Hessian, $d_{j_1,j_m + T_0 - 1}$ is the maximal gap between the iterates of the GD algorithm on the full data from $j_1$ to $j_m+T_0-1$ (see equation \eqref{d}), which goes to zero as $t\to\infty$) and $A = \frac{c_0\sqrt{m}[(1+e)^{m}-1]}{c_1} + c_0$ in which $e$ is a problem dependent constant defined in Theorem \ref{theorem: delta_model_para_bound}, $c_1$ is the ``strong independence'' constant from \eqref{assp: singular_lower_bound}.
\end{corollary}

\begin{proof}

Based on Theorem \ref{theorem: delta_model_para_bound},
$b_{q-1} = \|\bH_{j_q}\Dw_{j_k} - \B_{j_q}\Dw_{j_k}\| \leq \left[(1+e)^{q-k-1} - 1\right]f$.

Then based on the ``strong linear independence'' in Assumption \ref{assp: singular_lower_bound}, the matrix $\Delta W_{j_1,j_2,\dots, j_m} = [\frac{\Dw_{j_1}}{s_{j_1,j_m}}$ , $\frac{\Dw_{j_2}}{s_{j_1,j_m}}$, $\dots, \frac{\Dw_{j_m}}{s_{j_1,j_m}}]$ has its smallest singular value lower bounded by $c_1>0$. Then $\|\bH_{j_m} - \B_{j_m}\|$ can be bounded as below:

\begin{align}\label{eq: hessian_approx_bound}
    \begin{split}
        & \|\bH_{j_m} - \B_{j_m}\| \leq \frac{1}{c_1}\|\left(\bH_{j_m} - \B_{j_m}\right)\Delta W_{j_1, j_2,\dots, j_m}\|\\
        & \leq  \sqrt{m}[(1+e)^{m}-1]\frac{c_0}{c_1} \left(d_{j_1,j_m +T_0-1} + \frac{1}{\frac{1}{2}-\frac{r}{n}}M_1\frac{r}{n}\right)
    \end{split}
\end{align}
The second inequality uses the bound $\|M\|\le \sqrt{m} \max_i \|m_i\|$, where $M$ is a matrix with the $m$ columns $m_i$.


So by combining the results from equation \eqref{eq: hessian_approx_bound}, we can upper bound $\|\bH_t - \B_{j_m}\|$ where $j_m \le t \leq j_m + T_0 - 1$, i.e.:

\begin{align}\label{eq: approx_hessian_bound2}
    \begin{split}
    & \|\bH_t - \B_{j_m}\| = \|\bH_t - \bH_{j_m} + \bH_{j_m} + \B_{j_m}\|\\
    & \leq \|\bH_t - \bH_{j_m}\| + \|\bH_{j_m} - \B_{j_m}\|\\
    & \leq c_0 (d_{j_m,t} + M_1\frac{r}{n}) + \sqrt{m}[(1+e)^{m}-1]\frac{c_0}{c_1}\left(d_{j_1,j_m +T_0-1} + \frac{1}{\frac{1}{2}-\frac{r}{n}}M_1\frac{r}{n}\right)\\
    & \leq A d_{j_1,j_m+T_0 -1} + A\frac{1}{\frac{1}{2}-\frac{r}{n}}M_1\frac{r}{n}\\
    \end{split}
\end{align}


This finishes the proof.
\end{proof}

Note that in the upper bound on $\|\bH_t - \B_{j_m}\|$, there is one term $d_{j_1,j_m+T_0-1}$. So we need to do some analysis of this term:

\begin{lemma}[Contraction of the GD iterates]\label{Lemma: bound_d_j}
Recall the definition of $d_{j_k, j_q}$ from Theorem \ref{theorem: delta_model_para_bound}:
$$d_{j_k,j_q} = \max\left(\|\w_{a} - \w_{b}\|\right)_{j_k \leq a \leq b \leq j_q}.$$
Then $d_{j_k, j_q} \leq d_{j_k-z,j_q-z}$ for any positive integers $z$ and $d_{j_k,j_q} \leq (1-\mu\eta)^{j_k}d_{0, j_q-j_k} $ for any $0 \leq j_k \leq j_q$.
\end{lemma}

\begin{proof}
To prove the two inequalities, we should look at $d_{j_k,j_q}$ and $d_{j_k-z, j_q-z}$ where $z$ is a positive integer. For any given $j_k \leq a \leq b \leq j_q$, the upper bound on $\|w_a - w_b\|$ can be derived as below:

\begin{align*}
    \begin{split}
        & \|\w_{a} - \w_{b}\| = \|\w_{a-1} - \eta \nabla F(\w_{a-1}) - (\w_{b-1} - \eta \nabla F(\w_{b-1})\| \\
        & = \|\w_{a-1} - \w_{b-1} - \eta(\nabla F(\w_{a-1}) - \nabla F(\w_{b-1}))\|\\
        & = \|\w_{a-1} - \w_{b-1} -\\
        & \eta \frac{1}{n}\left(\int_{0}^1\sum_{i=1}^n \bH_i(\w_{a-1} + x(\w_{b-1} - \w_{a-1})) dx \right) (\w_{a-1} - \w_{b-1})\|\\
        & = \|\left(\textbf{I} - \frac{\eta}{n}\left(\int_{0}^1\sum_{i=1}^n \bH_i(\w_{a-1} + x(\w_{b-1} - \w_{a-1})) dx \right)\right) (\w_{a-1} - \w_{b-1})\|.
    \end{split}
\end{align*}

The derivation above uses the update rule of gradient descent and Cauchy mean-value theorem. Then according to Cauchy Schwarz inequality and strong convexity, it can be further bounded as 
$\|\w_{a} - \w_{b}\| \leq (1-\eta \mu)\|\w_{a-1} - \w_{b-1}\|$.

This can be used iteratively, which ends up with the following inequality:

\begin{align}\label{eq: w_diff_bound}
    \begin{split}
        \|\w_{a} - \w_{b}\| \leq (1-\eta \mu)^z\|\w_{a-z} - \w_{b-z}\|
    \end{split}
\end{align}

which indicates that $d_{j_k, j_q} \leq (1-\eta\mu)^zd_{j_k-z, j_q-z}$ and thus $d_{j_k,j_q} \leq d_{j_k-z, j_q-z}$. So by replacing 
$z$ with $j_k$, we will have: 
$d_{j_k, j_q} \leq (1-\mu\eta)^{j_k}d_{0,j_q-j_k}.$



\end{proof}

\subsubsection{Main recursions}
We bound the difference between $\iw_{t}$ and $\uw_{t}$. The proofs of the theorems stated below are in the following sections.

Our proof starts out with the usual approach of trying to show a contraction for the gradient updates, see e.g., \cite{bottou2018optimization}.  First we bound $\|\w_{t} - \uw_{t}\|$, i.e.:

\begin{theorem}[Bound between iterates on full and the leave-$r$-out dataset]\label{wu}
$\|\w_{t} - \uw_{t}\| \leq M_1 \frac{r}{n}$ where $M_1=\frac{2}{\mu }c_2$ is some positive constant that does not depend on $t$.
\end{theorem}

To show that the preliminary estimate on the bound on $\|\iw_t - \w_t\|$ used in Theorem \ref{theorem: delta_model_para_bound} and Corollary \ref{corollary: approx_hessian_real_hessian_bound} holds, the proof is provided as below:

\begin{theorem}[Bound between iterates on full data and incrementally updated ones]\label{bfi}
Consider an iteration $t$ indexed with $j_m$ for which $j_m \leq t < j_m +T_0 -1$, and suppose that we are at the $x$-th iteration of full gradient updates, so  $j_1 = j_0 + xT_0$, $j_m = j_0 + (m - 1 + x) T_0$. Suppose that we have the bounds $\|\bH_{t} - \B_{j_m}\| \leq \xi_{j_1,j_m} = Ad_{j_1,j_m + T_0 - 1} + \frac{1}{\frac{1}{2}-\frac{r}{n}}AM_1 \frac{r}{n}$ (where we recalled the definition of $\xi$)  and $\xi_{j_1,j_m} \leq \frac{\mu}{2}$ for all iterations $x$.
Then 
$$\|\iw_{t+1} - \w_{t+1} \| \leq \frac{2r c_2/n}{(1-r/n)\mu - \xi_{j_0,j_0+(m-1)T_0}} \leq \frac{1}{\frac{1}{2}-\frac{r}{n}}M_1 \frac{r}{n}.$$ 
Recall that $c_0$ is the Lipshitz constant of the Hessian, $M_1$ and $A$ are defined in Theorem \ref{wu} and Corollary \ref{corollary: approx_hessian_real_hessian_bound} respectively, which do not depend on $t$, 
\end{theorem}

For this theorem, note that this inequality depends on the condition $\|\bH_{t} - \B_{j_m}\| \leq \xi_{j_1,j_m}$ while in Theorem \ref{theorem: delta_model_para_bound}, to prove $\|\bH_{t} - \B_{j_m}\| \leq \xi_{j_1,j_m}$, we need to use the inequality in Theorem \ref{bfi}, i.e. $\|\iw_{t+1} - \w_{t+1} \| \leq \frac{1}{\frac{1}{2}-\frac{r}{n}}M_1 \frac{r}{n}$. In what follows, we will show that both inequalities hold for all the iterations $t$ without relying on other conditions.

We can select hyper-parameters $T_0,j_0$ such that 
$$A(1-\eta\mu)^{j_0-m+1}d_{0,(m-1)T_0} + \frac{1}{\frac{1}{2}-\frac{r}{n}}AM_1\frac{r}{n}
<\min(\frac{\mu}{2}, (1-\frac{r}{n})\mu - \frac{c_0M_1 r(n-r)}{2n^2}),$$ 
e.g. when $m=2$ and $T_0 = 5$, which is what we used in our experiments. It is enough that
$$j_0 > \max(\frac{\log(\frac{1}{Ad_{0,5}}[\frac{\mu}{2}-\frac{1}{\frac{1}{2}-\frac{r}{n}}AM_1\frac{r}{n})]}{\log(1-\eta\mu)}, \frac{\log(\frac{1}{Ad_{0,5}}[(1-\frac{r}{n})\mu-\frac{1}{\frac{1}{2}-\frac{r}{n}}AM_1\frac{r}{n})]}{\log(1-\eta\mu)})+m-1.$$ 
This holds for small enough $r/n$:

$$j_0 > \frac{\log(\frac{1}{Ad_{0,5}}[\frac{\mu}{2}-\frac{1}{\frac{1}{2}-\frac{r}{n}}AM_1\frac{r}{n})]}{\log(1-\eta\mu)}+m-1$$

Then the following two theorems hold.

\begin{theorem}[Bound between iterates on full data and incrementally updated ones (all iterations)]\label{bfi2}
For any $j_m < t < j_m +T_0 - 1$, 
$\|\iw_t - \w_t\| \leq \frac{1}{\frac{1}{2}-\frac{r}{n}}M_1\frac{r}{n}$ and $\|\bH_{t} - \B_{j_m}\| \leq \xi_{j_1,j_m}$. 
\end{theorem}

Then we have the following bound for $\|\uw_t - \iw_t\|$, which is our main result.

\begin{theorem}[Convergence rate of \Increm]\label{main10}
For all iterations $t$, the result $\iw_t$ of \Increm, Algorithm \ref{alg: update_algorithm}, approximates the correct iteration values $\uw_t$ at the rate
$$\|\uw_t - \iw_t\| = o(\frac{r}{n}).$$
So $\|\uw_t - \iw_t\|$ is of a lower order than $\frac{r}{n}$.
\end{theorem}
This is proved in Section \ref{mainpf1}.

\subsubsection{Proof of Theorem \ref{wu}}

\begin{proof}
By subtracting the GD update from equation \eqref{eq: update_rule_naive}, we have:
\begin{align}\label{eq: model_para_diff_1}
    \begin{split}
        &\uw_{t+1} - \w_{t+1} = \uw_{t} - \w_{t}\\
        &- \eta \left(\frac{1}{n-r}\left(\sum_{i=1}^n \nabla F_i\left(\uw_{t}\right) - \sum_{\substack{i \in R}} \nabla F_i\left(\uw_{t}\right)\right) - \frac{1}{n}\sum_{i=1}^n \nabla F_i\left(\w_{t}\right)\right)
    \end{split}
\end{align}

in which the right-hand side can be rewritten as:
\begin{align*}
    \begin{split}
        & \uw_{t} - \w_{t} -  \eta \left(\nabla F(\uw_t) - \nabla F(\w_t)\right)\\
        & - \eta \left(\frac{1}{n-r}\left(\sum_{i=1}^n \nabla F_i\left(\uw_{t}\right) - \sum_{\substack{i \in R}} \nabla F_i\left(\uw_{t}\right)\right) - \frac{1}{n}\sum_{i=1}^n \nabla F_i\left(\uw_{t}\right)\right) \\
        & = \uw_{t} - \w_{t} -  \eta \left(\nabla F(\uw_t) - \nabla F(\w_t)\right)\\
        & - 
        \frac{\eta}{n-r} 
        \left(\frac{r}{n}\sum_{i=1}^n \nabla F_i\left(\uw_{t}\right) 
        - \sum_{\substack{i \in R}}\nabla F_i\left(\uw_{t}\right)\right)\\
        &=\uw_{t} - \w_{t} -  \eta \left(\nabla F(\uw_t) - \nabla F(\w_t)\right) + \delta_{t}.
    \end{split}
\end{align*}

Then by applying Cauchy mean value theorem, the triangle inequality, Cauchy schwarz inequality and Lemma \ref{lemma: delta_t_bound} respectively, we have:
\begin{align*}
    \begin{split}
    	& \|\w_{t+1}-\uw_{t+1}\|\\
        & \leq \|\w_{t} - \uw_{t} - \eta (\int_0^1 \bH\left(\w_{t} + x\left(\uw _{t}-\w_t\right)\right)dx) \left(\w_{t} - \uw_{t}\right)\| + \|\delta_t\|\\
        & \leq \|\textbf{I}-\eta\int_0^1 \bH\left(\w_{t} + x\left(\uw _{t}-\w_t\right)\right)dx \|\|\w_{t} - \uw_{t}\| + \frac{2c_2r\eta}{n}
    \end{split}
\end{align*}

Then by applying the triangle inequality over integrals and Lemma \ref{eq: identity_hessian_bound}, the formula can be further bounded as:
\begin{align*}
    \begin{split}
        & \leq \|\int_0^1(\textbf{I}-\eta\bH\left(\w_{t} + x\left(\uw _{t}-\w_t\right)\right)dx)\|\|\w_{t} - \uw_{t}\| + \frac{2c_2r\eta}{n}\\
        & \leq (1-\eta \mu)\|\w_{t} - \uw_{t}\| + \frac{2c_2r\eta}{n}
    \end{split}
\end{align*}

Then by applying this formula iteratively, we get:
\begin{align*}
    & \|\w_{t+1}-\uw_{t+1}\| \leq \frac{1}{\eta\mu}\frac{2c_2r\eta}{n} = \frac{2c_2}{\mu}\frac{r}{n} \coloneqq M_1 \frac{r}{n}
\end{align*}


\end{proof}

\subsubsection{Proof of Theorem \ref{bfi}}
\begin{proof}
The updates for the iterations $j_m \leq t \leq j_m + T_0 - 1$ follow the Quasi-Hessian update. We proceed in a similar way as before, by expanding the recursion as below:
\begin{align}\label{eq: iw_bound0}
    \begin{split}
        & \|\iw_{t+1} - \w_{t+1} \|\\
        & = \|\iw_{t} - (\w_{t} - \eta \nabla F\left(\w_t\right)) \\
        & - \frac{\eta}{n-r}(n \left[\B_{j_m}\left(\iw_{t} - \w_{t}\right) + \nabla F\left(\w_{t}\right)\right]- \sum_{\substack{i \in R}} \nabla F_i\left(\iw_{t}\right))\|\\
        & = \|(\textbf{I} - \eta \frac{n}{n-r} \B_{j_m})(\iw_t - \w_t) - \frac{r\eta}{n-r}\nabla F(\w_t) + \frac{\eta}{n-r} \sum_{i\in R} \nabla F_i(\iw_{t})  \|\\
    \end{split}
\end{align}
By rearranging the formula above and using the triangle inequality, we get:
\begin{align}\label{eq: w_i_w_bound_2}
    \begin{split}
        & = \|(\textbf{I} - \eta \frac{n}{n-r} \B_{j_m})(\iw_t - \w_t) - \frac{r\eta}{n-r}\nabla F(\w_t)\\
        & + \frac{\eta}{n-r} \sum_{i\in R} (\bH_{t,i}\mmop(\iw_t - \w_t) + \nabla F_i(\w_{t}))\|\\
        & \leq \|(\textbf{I} - \eta \frac{n}{n-r} \B_{j_m})(\iw_t - \w_t) + \frac{\eta}{n-r} \sum_{i\in R} \bH_{t,i}\mmop(\iw_t - \w_t)\|\\
        & + \|\frac{r\eta}{n-r}\nabla F(\w_t)\| + \|\frac{\eta}{n-r} \sum_{i\in R}\nabla F_i(\w_{t})\|\\
    \end{split}
\end{align}

in which we use $\bH_{t,i}$ to denote $\int_{0}^1\bH_i(\w_t + x(\iw_t - \w_t))dx$ (recall that $\bH_i$ represents the Hessian matrix evaluated at the $i_{th}$ sample). Then the terms in the first absolute value are rewritten as:
\begin{align*}
    \begin{split}
        & [\textbf{I} - \eta \frac{n}{n-r} (\B_{j_m} - \bH_{t} + \bH_{t})](\iw_t - \w_t) + \frac{\eta}{n-r} \sum_{i\in R} \bH_{t,i}\mmop(\iw_t - \w_t)\\
        & = [\textbf{I} - \eta \frac{n}{n-r} (\B_{j_m} - \bH_{t})](\iw_t - \w_t) - \frac{\eta}{n-r} \sum_{i \not\in R} \bH_{t,i}\mmop(\iw_t - \w_t)
            \end{split}
\end{align*}

which uses the fact that $\bH_t = \sum_{i=1}^n \bH_{t,i} = \sum_{i \not\in R} \bH_{t,i} + \sum_{i \in R} \bH_{t,i}$. Then Formula \eqref{eq: w_i_w_bound_2} can be further bounded as:
\begin{align}\label{eq: iw_bound_last}
    \begin{split}
        & \leq \|[\textbf{I} - \frac{\eta}{n-r}\sum_{i \not\in R} \bH_{t,i}](\iw_t - \w_t)\| + \frac{n\eta}{n-r}\|(\B_{j_m} - \bH_{t})(\iw_t - \w_t)\|\\
        & + \|\frac{r\eta}{n-r}\nabla F(\w_t)\| + \|\frac{\eta}{n-r} \sum_{i\in R}\nabla F_i(\w_{t})\|\\
        & \leq (1 - \eta \mu + \eta \frac{n}{n-r}\xi_{j_1,j_m})\|\iw_t - \w_t\| + \frac{r\eta c_2}{n-r} + \frac{\eta r c_2}{n-r}\\
    \end{split}
\end{align}

Then according to Lemma \ref{Lemma: bound_d_j}, $d_{j_1,j_m + T_0 - 1} = d_{j_0+xT_0,j_0 + (x+m)T_0 - 1}$ decreases with increasing $x$, and thus $\xi_{j_1,j_m} = Ad_{j_1,j_m + T_0 - 1} + \frac{1}{\frac{1}{2}-\frac{r}{n}}AM_1 \frac{r}{n}$ is also decreasing with increasing $x$. So the formula above can be further bounded as:
\begin{align*}
    \begin{split}
        \leq (1-\eta\mu + \eta\frac{n}{n-r}\xi_{j_0,j_0+(m-1)T_0})\|\iw_t-\w_t\| + \frac{2r\eta c_2}{n-r}
    \end{split}
\end{align*}

This shows a recurrent inequality for $\|\iw_t - \w_t\|$. Next, notice that the conditions for deriving the above inequality hold for all $j_m \le t \leq j_m + T_0 - 1$. 

Then, when we reach $t=j_m$, we have an iteration where the gradient is computed exactly.   For these iterations we have
$\iw_{t+1} = \iw_{t} - \frac{\eta}{n-r} \sum_{i\not\in R}\nabla F(\iw_t)$
as well as $\w_{t+1} = \w_{t} - \eta \nabla F\left(\w_t\right)$. Using the same argument as in the bound for $\w_{t} - \uw_{t}$ we can get:
\begin{align*}
    & \|\w_{t+1} - \iw_{t+1}\|
    \leq \left[1-\eta \mu \right]\|\w_{t} - \iw_{t}\| + \frac{2c_2r \eta}{n}.
\end{align*}
Therefore, we effectively have $\xi=0$ for these iterations. We then continue with $t\gets t-1$, and use the appropriate bound among the two derived above. This recursive process works until we reach $t=1$.

As long as $\xi_{j_0,j_0+(m-1)T_0} \leq \frac{\mu}{2}$, $- \eta \mu + \eta \frac{n}{n-r}\xi_{j_0,j_0+(m-1)T_0} < - \eta \mu + \eta \frac{n}{n-r}\frac{\mu}{2}<0$. Then we get the following inequality:
\begin{align*}
\|\iw_{t} - \w_t \| &\le \frac{2\frac{\eta r c_2}{n-r}}{\eta \mu - \eta \frac{n}{n-r}\xi_{j_0,j_0+(m-1)T_0}}\\
&= \frac{2r c_2/n}{(1-r/n)\mu - \xi_{j_0,j_0+(m-1)T_0}}
\end{align*}

As long as $\xi_{j_0,j_0+(m-1)T_0} \leq \frac{\mu}{2}$, then 
$$\|\iw_{t} - \w_t \| \leq \frac{2r c_2/n}{(1-r/n)\mu - \xi_{j_0,j_0+(m-1)T_0}} \leq \frac{2r c_2/n}{(1-r/n)\mu - \frac{\mu}{2}} = \frac{1}{\frac{1}{2}-\frac{r}{n}}M_1 \frac{r}{n}.$$
The last step uses the fact that $M_1 = \frac{2c_2}{\mu}$.


\end{proof}

\subsubsection{Proof of Theorem \ref{bfi2}}

\textbf{Architecture of the proof.} To visualize the recursive proof process, we draw a picture as:

\tikzstyle{decision} = [diamond, draw, fill=blue!20, 
    text width=5em, text badly centered, node distance=3cm, inner sep=0pt]
\tikzstyle{iterations} = [rectangle, draw, dashed, text width=3em, text centered, minimum height=2em, font=\footnotesize]
\tikzstyle{iterations2} = [rectangle, draw, dashed, text width=18em, text centered, minimum height=3em, font=\footnotesize]
\tikzstyle{block} = [rectangle, draw, fill=blue!20, 
    text width=5.5em, text centered, rounded corners, minimum height=3em, font=\footnotesize]
\tikzstyle{empty} = [rectangle, draw, text width=5em, text centered, minimum height=3em, font=\footnotesize]
\tikzstyle{line} = [draw]
\tikzstyle{arrow} = [thick,->,>=stealth]

    
\begin{tikzpicture}[node distance = 1.7cm, auto]
    \node [block] (wu_diff0) {$\|\uw_{t} - \w_t\|$};
    \node [iterations, left of=wu_diff0, xshift=-0.5cm] (iter_0) {$t \leq j_0$};
    \node [iterations2, below of=iter_0,xshift=1cm, yshift=-1cm] (iter_1) {$j_0 < t \leq j_0 + T_0$ \\ $\{j_1,j_2,\dots,j_m\}$\\$= \{j_0-m+1,j_0-m+2,\dots,j_0-2,j_0-1,j_0\}$};
    \node [iterations2, below of=iter_1, yshift=-1cm] (iter_2) {$j_0 + T_0 < t \leq j_0 + 2T_0$\\ $\{j_1,j_2,\dots,j_m\}$\\$= \{j_0-m+2,\dots,j_0-2,j_0-1, j_0,j_0 +T_0\}$};
    \node [block, right of=wu_diff0, xshift=2.5cm] (w_diff0) {$\|\iw_t - \w_t\|$};
    \node [block, below of=w_diff0, yshift=-1cm] (w_diff1) {$\|\iw_t - \w_t\|$};
    \node [block, right of=w_diff1, xshift=2.5cm] (h_diff0) {$\|\bH_{t-1} - \B_{j_m}\|$};
    \node [block, below of=h_diff0, yshift=-1cm] (h_diff1) {$\|\bH_{t-1} - \B_{j_m}\|$};
    \node [block, below of=w_diff1, yshift=-1cm] (w_diff2) {$\|\iw_t - \w_t\|$};
    \node[text width=2cm, below of=h_diff1, yshift=-1cm](dot_text) {$\dots \dots$};
    \node[text width=2cm, below of=w_diff2, yshift=-1cm](dot_text2) {$\dots \dots$};
    \node [block, below of=h_diff1, yshift=-3cm] (h_diff2) {$\|\bH_{t-1} - \B_{j_m}\|$};
    \node [block, below of=w_diff2, yshift=-3cm] (w_diff3) {$\|\iw_t - \w_t\|$};
    \node [iterations2, below of=iter_2, yshift=-3cm] (iter_3) {$j_0 + (x+m-1)T_0 < t \leq j_0 + (x+m)T_0$\\ $\{j_1,j_2,\dots,j_m\}$\\$= \{j_0+xT_0,j_0+(x+1)T_0,\dots,j_0 +(x+m-1)T_0\}$};
    \node[text width=2cm, below of=w_diff3,yshift=-1cm](dot_text3) {$\dots \dots$};
    \node[text width=2cm, below of=h_diff2,yshift=-1cm](dot_text4) {$\dots \dots$};
    \draw [arrow] (wu_diff0) --node{Theorem \ref{wu}} (w_diff0);
    \draw [arrow] (w_diff0) --node[sloped, anchor=center, above]{Corollary \ref{corollary: approx_hessian_real_hessian_bound}} (h_diff0);
    \draw [-{Latex[right]}] (h_diff0) --node[anchor=center, above]{Theorem \ref{bfi}} (w_diff1);
    \draw [-{Latex[right]}] (w_diff1) --node[anchor=center, below]{Corollary \ref{corollary: approx_hessian_real_hessian_bound}} (h_diff0);
    \draw [arrow] (w_diff1) --node[sloped, anchor=center, above]{Corollary \ref{corollary: approx_hessian_real_hessian_bound}} (h_diff1);
    \draw [-{Latex[right]}] (h_diff1) --node[sloped, anchor=center, above]{Theorem \ref{bfi}} (w_diff2);
    \draw [-{Latex[right]}] (w_diff2) --node[anchor=center, below]{Corollary \ref{corollary: approx_hessian_real_hessian_bound}} (h_diff1);
    \draw [arrow] (w_diff2) --node[sloped, anchor=center, above]{Corollary \ref{corollary: approx_hessian_real_hessian_bound}} (dot_text);
    \draw [arrow] (dot_text2) --node[sloped, anchor=center, above]{Corollary \ref{corollary: approx_hessian_real_hessian_bound}} (h_diff2);
    \draw [arrow] (w_diff3) --node[sloped, anchor=center, above]{Corollary \ref{corollary: approx_hessian_real_hessian_bound}} (dot_text4);
    \draw [-{Latex[right]}] (h_diff2) --node[sloped, anchor=center, above]{Theorem \ref{bfi}} (w_diff3);
    \draw [-{Latex[right]}] (w_diff3) --node[anchor=center, below]{Corollary \ref{corollary: approx_hessian_real_hessian_bound}} (h_diff2);
    \draw[dashed] ($ (wu_diff0) - (3,1.5) $) -- ($ (w_diff0) - (-5,1.5) $);
    \draw[dashed] ($ (w_diff1) - (7,1.5) $) -- ($ (h_diff0) - (-1,1.5) $);
    \draw[dashed] ($ (w_diff2) - (7,1.5) $) -- ($ (h_diff1) - (-1,1.5) $);
    \draw[dashed] ($ (w_diff2) - (7,3.5) $) -- ($ (h_diff1) - (-1,3.5) $);
    \draw[dashed] ($ (w_diff3) - (7,1.5) $) -- ($ (h_diff2) - (-1,1.5) $);
\end{tikzpicture}

\begin{proof}
First of all, in terms of the bound on $\xi_{j_1,j_m}$ which is required in Theorem \ref{bfi}, i.e. $\xi_{j_1,j_m} \leq \frac{\mu}{2}$, we do the analysis below to show that we can adjust the value of $j_0$ and $T_0$ such that it can hold for all $t$. When $j_1 \geq j_0$, i.e. $j_1 = j_0 + xT_0$, then 
$$\{j_2,j_3,\dots, j_m\} = \{j_0 + (x+1)T_0, j_0 + (x+2)T_0, \dots, j_0 + (x+m-1)T_0\},$$ 
thus 
$\xi_{j_1,j_m} = \xi_{j_0 + x T_0,j_0 + (x+m-1)T_0} = Ad_{j_0 + xT_0,j_0 + (x+m)T_0 - 1} + \frac{1}{\frac{1}{2}-\frac{r}{n}}AM_1 \frac{r}{n}.$
Here $d_{j_0 + xT_0,j_0 + (x+m)T_0 - 1}$ decreases with $x$, and so does $
\xi_{j_1,j_m} = $ $\xi_{j_0 + x T_0,j_0 + (x+m-1)T_0}$. So the following inequality holds:
$$d_{j_0 + x T_0,j_0 + (x+m)T_0-1} \leq d_{j_0,j_0 + mT_0-1} \leq (1-\mu\eta)^{j_0}d_{0,mT_0-1}.$$ 
When $j_1 < j_0$, there are only $m$ different choices for $\{j_1,j_2,\dots,j_m\}$, in which the smallest $j_1$ used for approximation is $j_0-m+1$. Then, the following inequality holds: 
$$d_{j_1,j_m} \leq (1-\eta\mu)^{j_1} d_{0,j_m-j_1} \leq (1-\eta\mu)^{j_0-m+1} d_{0,j_m-j_1}.$$ 
For those $j_1,j_2,\dots,j_m$, we have $j_m - j_1 \leq (m-1)T_0$ and thus 
$$d_{j_1,j_m} \leq (1-\eta\mu)^{j_1} d_{0,j_m-j_1} \leq (1-\eta\mu)^{j_0 - m + 1} d_{0,j_m-j_1} \leq (1-\eta\mu)^{j_0 - m + 1} d_{0,(m-1)T_0}.$$ 
So $\xi_{j_1,j_m}$ is bounded by $A(1-\eta\mu)^{j_0-m+1}d_{0,(m-1)T_0} + \frac{1}{\frac{1}{2}-\frac{r}{n}}AM_1\frac{r}{n}$. To make sure $\xi_{j_1,j_m} \leq \frac{\mu}{2}$, we can adjust $j_0, m, T_0$ to make $A(1-\eta\mu)^{j_0-m+1}d_{0,(m-1)T_0} + \frac{1}{\frac{1}{2}-\frac{r}{n}}AM_1\frac{r}{n}$ smaller than $\frac{\mu}{2}$.

Then when $t \leq j_0$, the gradient is evaluated explicitly, which means that $\uw_t = \iw_t$, so the bound clearly holds, i.e., from Theorem \ref{wu}, we have $\|\uw_t - \w_t\| \leq \frac{M_1 r}{n}$ and thus $\|\iw_t - \w_t\| = \|\uw_t - \w_t\| \leq \frac{M_1r}{n} \leq \frac{1}{\frac{1}{2}-\frac{r}{n}}M_1 \frac{r}{n}$.

When $j_0 < t < j_0 + T_0$, in order to compute $\iw_t$, we need to use the history information $\{\Dw_{j_1}$, $\Dw_{j_2}$,$\dots$, $\Dw_{j_m}\}$, $\{\Dg_{j_1}, \Dg_{j_2}$, $\dots$, $\Dg_{j_m}\}$ and the corresponding quasi-Hessian matrices $\{\B_{j_1}, \B_{j_2}$, $\dots$, $\B_{j_m}\}$ where $\{j_1,j_2,\dots,j_m\} = \{j_0 - m + 1, j_0 - m + 2, \dots, j_0\}$ (we suppose $m < j_0$, which is a natural assumption). Since $\|\iw_t - \w_t\| \leq \frac{M_1r}{n}$ for any $t \leq j_0$, the conditions of Corollary \ref{corollary: approx_hessian_real_hessian_bound}  (used here with the $j_1, \ldots, j_m$ described above) hold up to $j_0$, so when $t=j_0+1$, 
$\|\bH_{t-1} - \B_{j_m}\| \leq \xi_{j_1,j_m}$
where $$\xi_{j_1,j_m} = \xi_{j_0-m+1, j_0}= A d_{j_1,j_m + T_0 - 1} + A M_1\frac{r}{n} = A d_{j_0-m+1,j_0+T_0} + A M_1\frac{r}{n}.$$ 


Plus, according to Theorem \ref{bfi}, $\|\iw_t - \w_t\| \leq \frac{2r c_2/n}{(1-r/n)\mu - \xi_{j_1,j_m}} = \frac{2r c_2/n}{(1-r/n)\mu - \xi_{j_0 - m + 1,j_0}}$. When $\xi_{j_0-m+1,j_0} \leq \frac{\mu}{2}$, then $$\|\iw_t - \w_t\| \leq \frac{2r c_2/n}{(1-r/n)\mu - \xi_{j_1,j_m}} = \frac{2r c_2/n}{(1-r/n)\mu - \frac{\mu}{2}} = \frac{1}{\frac{1}{2}-\frac{r}{n}}M_1\frac{r}{n}.$$ 
So the bound on $\|\iw_t - \w_t\|$ holds for all $t \leq j_0 + 1$. Then according to the conditions of Corollary \ref{corollary: approx_hessian_real_hessian_bound}, when $t = j_0+2$, 
$\|\bH_{t-1} - \B_{j_m}\| \leq \xi_{j_1,j_m}$ holds. This can proceed recursively until $t=j_0 +T_0$,
in which the gradients are explicitly evaluated according to Theorem \ref{bfi}, i.e.:
$$\|\iw_{j_0+T_0} - \w_{j_0 + T_0}\| \leq \frac{2r c_2/n}{(1-r/n)\mu - \xi_{j_1,j_m}} = \frac{2r c_2/n}{(1-r/n)\mu - \frac{\mu}{2}} = \frac{1}{\frac{1}{2}-\frac{r}{n}}M_1\frac{r}{n}.$$

Next when $j_0 + T_0 < t < j_0 + 2T_0$, $j_m$ is updated as $j_0 + T_0$ while $\{j_1, j_2,\dots, j_{m-1}\}$ is updated as $\{j_0 - m + 2, j_0 - m+3,\dots, j_0\}$ and we know that $\|\iw_{j_k} - \w_{j_k}\| \leq \frac{1}{\frac{1}{2}-\frac{r}{n}}M_1\frac{r}{n}$. So based on Corollary \ref{corollary: approx_hessian_real_hessian_bound}, the following inequality holds:
$$\|\bH_{t} - \B_{j_m}\| \leq \xi_{j_1,j_m} = \xi_{j_0 - m +2, j_0 +T_0} = A d_{j_1,j_m + T_0 - 1} + A M_1\frac{r}{n}$$
$$ = A d_{j_0-m+2,j_0+2T_0} + A M_1\frac{r}{n}$$

This process can proceed recursively. 

When $j_0 + xT_0 < t < j_0 + (x+1) T_0$, we know that:
$$\|\bH_{t} - \B_{j_m}\| \leq \xi_{j_1,j_m} = A d_{j_1,j_m + T_0 - 1} + A M_1\frac{r}{n}.$$

Then based on Theorem \ref{bfi}, $\|\iw_t - \w_t\| \leq \frac{1}{\frac{1}{2} - \frac{r}{n}}M_1\frac{r}{n}$. Then at iteration $j_0 + (x+1)T_0$, we update $j_1,j_2,\dots,j_m$ as:
$j_m \leftarrow j_0 + (x+1)T_0$ and $j_{i-1} \leftarrow j_{i}$ ($i=2,3,\dots, m$) and thus $$\|\iw_{j_k} - \w_{j_k}\| \leq \frac{1}{\frac{1}{2} - \frac{r}{n}}M_1\frac{r}{n}$$ still holds for all $k=1,2,\dots,m$.

So when $j_0 + (x+1)T_0 < t < j_0 + (x+2) T_0$, Corollary \ref{corollary: approx_hessian_real_hessian_bound} and Theorem \ref{bfi} are applied alternatively. Then the following two inequalities hold for all iterations $t$ satisfying $j_0 + (x+1)T_0 < t < j_0 + (x+2) T_0$:
$$\|\bH_{t} - \B_{j_m}\| \leq \xi_{j_1,j_m} = A d_{j_1,j_m + T_0 - 1} + A M_1\frac{r}{n},$$
$$\|\iw_{j_k} - \w_{j_k}\| \leq \frac{1}{\frac{1}{2} - \frac{r}{n}}M_1\frac{r}{n}.$$

So in the end, we know that:
$$\|\iw_t - \w_t\| \leq \frac{1}{\frac{1}{2}-\frac{r}{n}}M_1\frac{r}{n}$$ and $$\|\bH_{t} - \B_{j_m}\| \leq \xi_{j_1,j_m}$$
hold for all $t$.

\end{proof}

\subsubsection{Proof of Theorem \ref{main10}}
\label{mainpf1}
\begin{proof}

The proof is by induction.

When $t \leq j_0$, the gradient is evaluated explicitly, which means that $\uw_t = \iw_t$, so the bound clearly holds.


From iteration $T_0$ to iteration $t$, the difference between $\iw_{t}$ and $\uw_{t}$ can be bounded as follows. In these equations, we use the definition of the update formula  $ \iw_{t+1} =  \iw_{t} - \frac{\eta}{n-r}[n(\B_{j_m}(\iw_t - \w_t)) - \sum_{i\in R} \nabla F(\iw_t)]$.
By rearranging terms appropriately, we get:
\begin{align}\label{eq: w_U_w_IU_diff}
\begin{split}
    \|\iw_{t+1} - \uw_{t+1}\|& = \|\iw_{t} - \uw_{t} - \frac{n\eta}{n-r}\left[\B_{j_m}(\iw_{t} - \w_{t}) + \nabla F(\w_{t})\right]\\
     & + \frac{\eta}{n-r}\sum_{\substack{i \in R}} \nabla F_i(\iw_{t}) + \frac{\eta}{n-r}\sum_{\substack{i \not\in R}} \nabla F_i(\uw_{t})\|
\end{split}
\end{align}

Then by bringing in $\bH_{t}$ into the expression above, it is rewritten as: 

\begin{align}\label{eq: w_i_w_u_gap_deriv1}
\begin{split}
 &= \|\iw_{t} - \uw_{t} - \frac{n\eta}{n-r}\left[\left(\B_{j_m} - \bH_{t}\right)(\iw_{t} - \w_{t}) + \bH_{t}\mmop(\iw_{t} - \w_{t}) + \nabla F(\w_{t})\right]\\
     & + \frac{\eta}{n-r}\sum_{\substack{i \in R}} (\nabla F_i(\iw_{t}) - \nabla F_i(\w_{t}) + \nabla F_i(\w_{t})) + \frac{\eta}{n-r}\sum_{\substack{i \not\in R}} \nabla F_i(\uw_{t})\|
\end{split}
\end{align}

In the formula above, we will try to make sure there is no confusion between $\bH_{t}(\w)$ (Hessian as a function evaluated at $\w$) and $\bH_{t}\times (\w)$ (Hessian times a vector). Then by applying the Cauchy mean value theorem over each individual $\nabla F_i(\iw_{t}) - \nabla F_i(\w_{t})$ and by denoting the corresponding Hessian matrix as $\bH_{t,i}$ (note that $\sum_{i=1}^n \bH_{t,i} = n\bH_t$), the expression becomes:

\begin{align*}
\begin{split}
& = \|\iw_{t} - \uw_{t} - \frac{n\eta}{n-r}\left[\left(\B_{j_m} - \bH_{t}\right)(\iw_{t} - \w_{t}) + \bH_{t}\mmop(\iw_{t} - \w_{t}) + \nabla F(\w_{t})\right]\\
     & + \frac{\eta}{n-r}\sum_{\substack{i \in R}} (\bH_{t,i}\mmop(\iw_{t} - \w_t) + \nabla F_i(\w_{t})) + \frac{\eta}{n-r}\sum_{\substack{i \not\in R}} \nabla F_i(\uw_{t})\| \\
\end{split}
\end{align*}

Then by using the fact that $\sum_{i\in R} \nabla F_i(\w_t) + \sum_{i \not\in R} \nabla F_i(\w_t) = n \nabla F(\w_t)$ and $\sum_{i\in R} \bH_{t,i} + \sum_{i \not\in R} \bH_{t,i} = n \bH_t$, the expression can be rearranged as:

\begin{align*}
\begin{split}
& = \|\iw_{t} - \uw_{t} - \frac{\eta}{n-r}\sum_{\substack{i \not\in R}}\bH_{t,i}\mmop(\iw_{t} - \w_t) -  \frac{n\eta}{n-r}\left[\left(\B_{j_m} - \bH_{t}\right)(\iw_{t} - \w_{t})\right]\\
     & -\frac{\eta}{n-r}\sum_{\substack{i \not\in R}} \nabla F_i(\w_{t}) + \frac{\eta}{n-r}\sum_{\substack{i \not\in R}} \nabla F_i(\uw_{t})\| \\
\end{split}
\end{align*}

in which $\frac{\eta}{n-r}\sum_{\substack{i\in R}} \nabla F_i(\w_{t})$ is canceled out. Then by adding and subtracting $\uw_{t}$ in the first part, we get:

\begin{align*}
\begin{split}
& = \|\iw_{t} - \uw_{t} - \frac{\eta}{n-r}\sum_{\substack{i \not\in R}}\bH_{t,i}\mmop(\iw_{t} - \uw_t) -  \frac{n\eta}{n-r}\left[\left(\B_{j_m} - \bH_{t}\right)(\iw_{t} - \uw_{t})\right]\\
& - \frac{\eta}{n-r}\sum_{\substack{i \not\in R}}\bH_{t,i}\mmop(\uw_{t} - \w_t) - \frac{n\eta}{n-r}\left[\left(\B_{j_m} - \bH_{t}\right)(\uw_{t} - \w_{t})\right]\\
     & -\frac{\eta}{n-r}\sum_{\substack{i \not\in R}} \nabla F_i(\w_{t}) + \frac{\eta}{n-r}\sum_{\substack{i \not\in R}} \nabla F_i(\uw_{t})\| 
\end{split}
\end{align*}

We apply Cauchy mean value theorem over $-\frac{\eta}{n-r}\sum_{\substack{i \not\in R}} \nabla F_i(\w_{t}) + \frac{\eta}{n-r}\sum_{\substack{i \not\in R}} \nabla F_i(\uw_{t})$, i.e.:
\begin{align*}
&-\frac{\eta}{n-r}\sum_{\substack{i \not\in R}} \nabla F_i(\w_{t}) + \frac{\eta}{n-r}\sum_{\substack{i \not\in R}} \nabla F_i(\uw_{t})\\ 
&= \frac{\eta}{n-r}[\sum_{\substack{i \not\in R}} \int_0^1 \bH_i(\w_t + x(\uw_t - \w_t)) dx](\uw_t - \w_t).
\end{align*}
In addition, note that $\bH_{t,i} = \int_0^1 \bH_i(\w_t + x(\iw_t - \w_t)) dx$. So the formula above becomes:

\begin{align*}
\begin{split}    
& = \|\iw_{t} - \uw_{t} - \frac{\eta}{n-r}\sum_{\substack{i \not\in R}}\bH_{t,i}\mmop(\iw_{t} - \uw_t) -  \frac{n\eta}{n-r}\left[\left(\B_{j_m} - \bH_{t}\right)(\iw_{t} - \uw_{t})\right]\\
& - \frac{\eta}{n-r}\sum_{\substack{i \not\in R}}(\int_{0}^1 \bH_i(\w_{t} + x(\iw_t - \w_t))dx)(\uw_{t} - \w_t) - \frac{n\eta}{n-r}\left[\left(\B_{j_m} - \bH_{t}\right)(\uw_{t} - \w_{t})\right]\\
     & + \frac{\eta}{n-r}\sum_{\substack{i \not\in R}} (\int_0^1 \bH_i(\w_t + x(\uw_t-\w_t))dx)(\uw_{t} - \w_t)\|.
\end{split}
\end{align*}

Then by applying the triangle inequality and rearranging the expression appropriately, the expression can be bounded as:

\begin{align*}
\begin{split}
& \leq \|(\textbf{I} - \frac{\eta}{n-r}\sum_{\substack{i \not\in R}}\bH_{t,i})(\iw_{t}-\uw_{t})\| + \|\frac{n\eta}{n-r}\left[\left(\B_{j_m} - \bH_{t}\right)(\iw_{t} - \uw_{t})\right]\|\\
& + \|\frac{\eta}{n-r}[\sum_{i\not\in R}\int_0^1 \bH_i(\w_t + x(\uw_t-\w_t))dx - \int_0^1 \bH_i(\w_t + x(\iw_t-\w_t))dx](\uw_{t}-\w_t)\|\\
& + \|\frac{n\eta}{n-r}\left[\left(\B_{j_m} - \bH_{t}\right)(\uw_{t} - \w_{t})\right]\|,
\end{split}
\end{align*}

in which the first term is the main contraction component which always appears in the analyses of gradient descent type algorithms. The remaining terms are error terms due to the various sources of error: using a quasi-Hessian, not having a quadratic objective (implicitly assumed by the local models at each step), using the iterate $\iw$ for our update instead of the correct $\uw$.

Then by using the following facts:
\benum 
\item $\|\textbf{I} - \eta \bH_{t,i}\| \le 1 - \eta \mu$;
\item from Theorem \ref{bfi2} on the approximation accuracy of the quasi-Hessian to mean Hessian, we have the error bound $\|\bH_t - \B_{j_m}\|\le\xi_{j_1,j_m}$;
\item we can bound the difference of integrated Hessians using the strategy from equation \eqref{eq: hessian_diff_bound0};
\item from Theorem \ref{wu}, we have the error bound $\|\uw_{t}-\w_t\|\le M_1 \frac{r}{n}$ (and this requires no additional assumptions),
\eenum
the expression above can be bounded as follows:

\begin{align}\label{eq: w_U_w_IU_diff_2}
\begin{split}
& \leq (1-\eta\mu + \frac{n\eta}{n-r}\xi_{j_1,j_m})\|\iw_t - \uw_t\| + \frac{\eta c_0}{2}\|\uw_{t} - \iw_{t}\|\|\uw_{t}-\w_t\|\\
& + \frac{n\eta}{n-r} \xi_{j_1,j_m} \|\uw_t-\w_t\|\\
& \leq (1-\eta\mu + \frac{n\eta}{n-r}\xi_{j_1,j_m}+\frac{c_0M_1r\eta}{2n})\|\iw_t - \uw_t\| + \frac{M_1r\eta}{n-r} \xi_{j_1,j_m} 
    \end{split}
\end{align}

Recall from Corollary \ref{corollary: approx_hessian_real_hessian_bound} that $\xi_{j_1,j_m} = \xi_{j_0 + xT_0,j_0 + (m+x-1)T_0} = A d_{j_0 + xT_0,j_0 + (m+x)T_0 - 1} + A\frac{1}{\frac{1}{2}-\frac{r}{n}} M_1 \frac{r}{n}$ decreases with the increasing $x$. So the formula above can be bounded as:
\begin{align}\label{eq: w_U_w_IU_diff_3}
    \begin{split}
        & \leq (1-\eta\mu + \frac{n\eta}{n-r}\xi_{j_0,j_0+(m-1)T_0}+\frac{c_0M_1r\eta}{2n})\|\iw_t - \uw_t\| + \frac{M_1r\eta}{n-r} \xi_{j_1,j_m}.
    \end{split}
\end{align}

Also by plugging the formula for $\xi$ into the formula above and using Lemma \ref{Lemma: bound_d_j} (contraction of GD updates), we get:

\begin{align}\label{eq: w_u_w_i_gap_deriv_2}
    \begin{split}
        & \leq (1-\eta\mu + \frac{n\eta}{n-r}\xi_{j_0,j_0+(m-1)T_0}+\frac{c_0M_1r\eta}{2n})\|\iw_t - \uw_t\|\\
        & + \frac{M_1r\eta}{n-r} (A d_{j_0 + xT_0,j_0 + (m+x)T_0 - 1} + A\frac{1}{\frac{1}{2}-\frac{r}{n}} M_1 \frac{r}{n})\\
        & \leq (1-\eta\mu + \frac{n\eta}{n-r}\xi_{j_0,j_0+(m-1)T_0}+\frac{c_0M_1r\eta}{2n})\|\iw_t - \uw_t\|\\
        & + \frac{M_1r\eta}{n-r} (A(1-\eta\mu)^{j_0+xT_0} d_{0,mT_0 - 1} + A\frac{1}{\frac{1}{2}-\frac{r}{n}} M_1 \frac{r}{n})\\
    \end{split}
\end{align}

Now, we will argue that it is pssible to choose hyperparameters such that $\xi_{j_0,j_0+(m-1)T_0} \leq (1-\frac{r}{n})\mu - \frac{c_0M_1 r(n-r)}{2n^2}$. Then $1-\eta\mu + \frac{n\eta}{n-r}\xi_{j_0,j_0+(m-1)T_0}+\frac{c_0M_1r\eta}{2n}$ is a constant for all $t$ and smaller than 1. By denoting $\mu- \frac{n}{n-r}\xi_{j_0,j_0+(m-1)T_0}-\frac{c_0M_1r}{2n}$ as $C$, the formula above can be written as:

\begin{align*}
    \begin{split}
        & = (1-\eta C) \|\iw_t - \uw_t\| + \frac{M_1r\eta}{n-r} (A(1-\eta\mu)^{j_0+xT_0} d_{0,mT_0 - 1} + A\frac{1}{\frac{1}{2}-\frac{r}{n}} M_1 \frac{r}{n}).
    \end{split}
\end{align*}

This can be used recursively until iteration $j_m = j_0 + (x+m)T_0 - 1$, i.e.:

\begin{align*}
    \begin{split}
        & \leq (1-\eta C)^{t-(j_0+(x+m-1)T_0)-1}\|\iw_{j_0 + (x+m-1)T_0+1} - \uw_{j_0 + (x+m-1)T_0+1}\|\\
        & + \frac{1-(1-\eta C)^{t-(j_0+(x+m-1)T_0)}}{\eta C}\frac{M_1r\eta}{n-r} (A(1-\eta\mu)^{j_0+xT_0} d_{0,mT_0 - 1} + A\frac{1}{\frac{1}{2}-\frac{r}{n}} M_1 \frac{r}{n})\\
        & \leq (1-\eta C)^{t-(j_0+(x+m-1)T_0)-1}\|\iw_{j_0 + (x+m-1)T_0+1} - \uw_{j_0 + (x+m-1)T_0+1}\|\\
        & + \frac{M_1r}{C(n-r)} (A(1-\eta\mu)^{j_0+xT_0} d_{0,mT_0 - 1} + A\frac{1}{\frac{1}{2}-\frac{r}{n}} M_1 \frac{r}{n})
    \end{split}
\end{align*}

We can set $t = j_0 + (y+m)T_0$ and for any $y=1,2,\dots,x-1$, the formula above can be rewritten as:
\begin{align*}
    \begin{split}
        &\|\iw_{j_0 + (y+m)T_0} - \uw_{j_0 + (y+m)T_0}\|\\ & \leq (1-\eta C)^{T_0-1}\|\iw_{j_0 + (y+m-1)T_0+1} - \uw_{j_0 + (y+m-1)T_0+1}\|\\
        & + \frac{M_1r}{C(n-r)} (A(1-\eta\mu)^{j_0+yT_0} d_{0,mT_0 - 1} + A\frac{1}{\frac{1}{2}-\frac{r}{n}} M_1 \frac{r}{n})\\
    \end{split}
\end{align*}

Then at the iteration $t=j_0+(y+m-1)T_0$, the gradient is explicitly evaluated, which means that:
$$\|\iw_{j_0 + (y+m-1)T_0+1} - \uw_{j_0 + (y+m-1)T_0+1}\| \leq (1-\eta\mu)\|\iw_{j_0 + (y+m-1)T_0} - \uw_{j_0 + (y+m-1)T_0}\|.$$ 
Since $C=\mu- \frac{n}{n-r}\xi_{j_0,j_0+(m-1)T_0}-\frac{c_0M_1r}{2n}$, then $1-\eta\mu < 1-\eta C$ and thus 
$$\|\iw_{j_0 + (y+m-1)T_0+1} - \uw_{j_0 + (y+m-1)T_0+1}\| \leq (1-\eta C)\|\iw_{j_0 + (y+m-1)T_0} - \uw_{j_0 + (y+m-1)T_0}\|,$$ which can be plugged into the formula above:
\begin{align*}
    \begin{split}
        &\|\iw_{j_0 + (y+m)T_0} - \uw_{j_0 + (y+m)T_0}\|\\ & \leq (1-\eta C)^{T_0}\|\iw_{j_0 + (y+m-1)T_0} - \uw_{j_0 + (y+m-1)T_0}\|\\
        & + \frac{M_1r}{C(n-r)} (A(1-\eta\mu)^{j_0+yT_0} d_{0,mT_0 - 1} + A\frac{1}{\frac{1}{2}-\frac{r}{n}} M_1 \frac{r}{n}).
    \end{split}
\end{align*}
This can be used recursively over $y=x-1,x-2,\dots,2,1$:
\begin{align}\label{eq: bound_wu_wi_diff0}
    \begin{split}
        &\|\iw_{j_0 + (y+m)T_0} - \uw_{j_0 + (y+m)T_0}\|\\ & \leq (1-\eta C)^{y T_0}\|\iw_{j_0 + mT_0} - \uw_{j_0 + mT_0}\|\\
        & + \sum_{p=1}^y(1-\eta C)^{(y-p) T_0}\frac{M_1r}{C(n-r)} (A(1-\eta\mu)^{j_0+pT_0} d_{0,mT_0 - 1} + A\frac{1}{\frac{1}{2}-\frac{r}{n}} M_1 \frac{r}{n})\\
        & = (1-\eta C)^{y T_0}\|\iw_{j_0 + mT_0} - \uw_{j_0 + mT_0}\| \\
        & + \sum_{p=1}^y(1-\eta C)^{(y-p) T_0}\frac{M_1r}{C(n-r)} (A(1-\eta\mu)^{j_0+pT_0} d_{0,mT_0 - 1})\\
        & + \sum_{p=1}^y(1-\eta C)^{(y-p) T_0}\frac{AM_1^2r^2}{C(n-r)(n/2-r)},
    \end{split}
\end{align}

in which $$\sum_{p=1}^y(1-\eta C)^{(y-p) T_0}\frac{M_1r}{C(n-r)} (A(1-\eta\mu)^{j_0+pT_0} d_{0,mT_0 - 1})$$
$$=\frac{AM_1r\eta}{C(n-r)}(1-\eta C)^{yT_0}(1-\eta\mu)^{j_0}d_{0,mT_0 - 1}\sum_{p=1}^y(1-\eta C)^{-pT_0}(1-\eta\mu)^{pT_0}.$$

Recall that since $1-\eta C > 1-\eta \mu$, then the formula above can be bounded as:
$$\sum_{p=1}^y(1-\eta C)^{(y-p) T_0}\frac{M_1r}{C(n-r)} (A(1-\eta\mu)^{j_0+pT_0} d_{0,mT_0 - 1})$$
$$\leq \frac{AM_1r\eta}{C(n-r)}(1-\eta C)^{yT_0}(1-\eta\mu)^{j_0}d_{0,mT_0 - 1}\frac{1}{1-(\frac{1-\eta\mu}{1-\eta C})^{T_0}}.$$

Also $\sum_{p=1}^y(1-\eta C)^{(y-p) T_0}\frac{AM_1^2r^2}{C(n-r)(n/2-r)}$ can be simplified to:
$$\sum_{p=1}^y(1-\eta C)^{(y-p) T_0}\frac{AM_1^2r^2}{C(n-r)(n/2-r)} = \sum_{p=0}^{y-1}(1-\eta C)^{pT_0}\frac{AM_1^2r^2}{C(n-r)(n/2-r)}$$
$$\leq \frac{1}{1-(1-\eta C)^{T_0}}\frac{AM_1^2r^2}{C(n-r)(n/2-r)}.$$

So equation \eqref{eq: bound_wu_wi_diff0} can be further bounded as:
\begin{align}\label{eq: w_U_w_IU_diff_last}
    \begin{split}
        & \|\iw_{j_0 + (y+m)T_0} - \uw_{j_0 + (y+m)T_0}\|\\ & \leq (1-\eta C)^{y T_0}\|\iw_{j_0 + mT_0} - \uw_{j_0 + mT_0}\|\\
        & + \frac{AM_1r}{C(n-r)}(1-\eta C)^{yT_0}(1-\eta\mu)^{j_0}d_{0,mT_0 - 1}\frac{1}{1-(\frac{1-\eta\mu}{1-\eta C})^{T_0}}\\
        & + \frac{1}{1-(1-\eta C)^{T_0}}\frac{AM_1^2r^2}{C(n-r)(n/2-r)}.
    \end{split}
\end{align}

When $t \rightarrow \infty$ and thus $y \rightarrow \infty$, $(1-\eta C)^{y T_0} \rightarrow 0$ and thus $$\|\iw_{j_0 + (y+m)T_0} - \uw_{j_0 + (y+m)T_0}\| = o(\frac{r}{n}).$$

\end{proof}

\input{Sections/SGD_analysis.tex}

\section{Details on applications}

\subsection{Privacy related data deletion}

The notion of Approximate Data Deletion from the training dataset is proposed in  \cite{ginart2019making}:
\begin{definition}
A data deletion operation $R_A$ is a $\delta-$deletion for algorithm $A$ if, for all datasets $D$ and for all measurable subset $S$, the following inequality holds:
$$Pr[A(D_{-i})\in S| D_{-i}] \geq \delta Pr[R_A(D, A(D), i)\in S|D_{-i}],$$

where $D$ is the full training dataset, $D_{-i}$ is the remaining dataset after the $i_{th}$ sample is removed, $A(D)$ and $A(D_{-i})$ represent the model trained over $D$ and $D_{-i}$ respectively. Also $R_A$ is an approximate model update algorithm, which updates the model after the sample $i$ is removed.

\end{definition}

This definition mimics the classical definition of differential privacy \citep{dwork2014algorithmic}:

\begin{definition}
A mechanism $M$ is $\epsilon$-differentially private, where $\epsilon \geq 0$
,
if for all neighboring databases $D_0$ and $D_1$, i.e., for databases differing in only one record, and for all sets $S \in [M]$, where $[M]$ is the range of $M$, the following inequality holds: 
$$Pr [M(D_0) \in S] \leq e^{\epsilon} Pr [M(D_1) \in S].$$
\end{definition}

By borrowing the notations from \cite{ginart2019making}, we define a version of approximate data deletion, which is slightly more strict than the one from \cite{ginart2019making}:
\begin{definition}
$R_A$ is an $\epsilon-$approximate deletion for $A$ if for all $D$ and measurable subset $S \subset \mathcal{H}$:
$$P(A(D_{-i}) \in S | D_{-i}) \leq e^{\epsilon}P(R_A(D,A(D),i)\in S | D_{-i})$$
and 
$$P(R_A(D,A(D),i)\in S | D_{-i}) \leq e^{\epsilon}P(A(D_{-i}) \in S | D_{-i}).$$
\end{definition}

To satisfy this definition for gradient descent, necessary randomness is added to the output of the \Basel\ and \Increm. One simple way is the Laplace mechanism \citep{dwork2014algorithmic}, also following the idea from \cite{chaudhuri2009privacy} where noise following the Laplace distribution, i.e. $$Lap(x|\frac{2}{n\epsilon\lambda}) = \frac{1}{\frac{2}{n\epsilon\lambda}}\exp(-\frac{|x|}{\frac{2}{n\epsilon\lambda}}),$$ is added to the each coordinate of the output of the regularized logistic regression. Here $p$ is the number of the parameters, $\lambda$ is the regularization rate and $\frac{2}{n\lambda}$ is the \textit{sensitivity} of logistic regression (see \cite{chaudhuri2009privacy} for more details). 

We can add even smaller noise to $\w^*$, $\uw^*$ and $\iw^*$, which follows the distribution $Lap(\frac{\delta}{\epsilon})$ for each coordinate of $\w^*$, $\uw^*$ and $\iw^*$ and is independent across different coordinates. Here $\delta > \sqrt{p} \delta_0$ and $$\delta_0 = \frac{1}{\eta(\frac{1}{2}\mu - \frac{r}{n-r}\mu - \frac{c_0 M_1r}{2n})^2}\frac{M_1r}{n-r}(A\frac{1}{\frac{1}{2}-\frac{r}{n}} M_1 \frac{r}{n})$$ (which is an upper bound on $\|\uw^* - \iw^*\|$), such that the randomized \Increm\ preserves $\epsilon-$approximate deletion.

\begin{proof}
We denote the model parameters after adding the random noise over $\rw$, $\ruw$ and $\riw$, and $\textbf{v}_i$ as the value of $\textbf{v}$ in the $i_{th}$ coordinate. We have:
$$\w^* - \rw^*, \uw^* - \ruw^*, \iw^* - \riw^* \sim Lap(\frac{\delta}{\epsilon})$$

Given an arbitrary vector $\textbf{z} = [z_1,z_2,\dots,z_p]$, the probability density ratio between $Pdf(\ruw^*=\textbf{z})$ and $Pdf(\riw^*=\textbf{z})$ can be calculated as
\begin{align*}
\begin{split}
    &\frac{Pdf(\ruw^*=\textbf{z})}{Pdf(\riw^*=\textbf{z})} = \frac{\Pi_{i=1}^p \frac{\epsilon}{\delta}\exp(-\frac{\epsilon|\textbf{z} - \uw^*|}{\delta}) }{\Pi_{i=1}^p\frac{\epsilon}{\delta}\exp(-\epsilon\frac{|\textbf{z}_i - \iw^*_i|}{\delta})}\\
    & = \Pi_{i=1}^p \exp(\frac{\epsilon (|\textbf{z}_i - \uw^*| - |\textbf{z}_i - \iw^*_i|)}{\delta})\\
    & \leq \Pi_{i=1}^p \exp(\frac{\epsilon (|\iw^*_i - \uw^*_i|)}{\delta})\\
    & = \exp(\frac{\epsilon (\|\iw^* - \uw^*\|_1)}{\delta})
\end{split}
\end{align*}

Since $$\|\iw^* - \uw^*\|_1 \leq \sqrt{p}\|\iw^* - \uw^*\|_2 = \sqrt{p}\|\iw^* - \uw^*\|$$

Then,

\begin{align*}
    \begin{split}
        & \frac{Pdf(\ruw^*=\textbf{z})}{Pdf(\riw^*=\textbf{z})} \leq \exp(\frac{\epsilon (\|\iw^* - \uw^*\|)}{\delta})\\
        & \leq \exp(\frac{\epsilon \sqrt{p}\delta_0}{\delta}) \leq \exp(\epsilon)
    \end{split}
\end{align*}

Similarly, we can also prove $\frac{Pdf(\ruw^*=\textbf{z})}{Pdf(\riw^*=\textbf{z})} \geq \exp(\epsilon)$ by symmetry.

\end{proof}


\input{Sections/other_alg.tex}

\input{Sections/other_experiements.tex}

%% file: Sections/SGD_analysis.tex
\subsection{Results for stochastic gradient descent}\label{sec: sgd_proof}
\subsubsection{Quasi-Newton}
We modify Equations \eqref{eq: B_formula} and \eqref{eq: hessian_update2} to SGD versions:
\begin{align}\label{eq: sgd_hessian_updates}
    \begin{split}
        \sB_{j_{k+1}} = \sB_{j_k} - \frac{\sB_{j_k}\Dsw_{j_k} \Dsw_{j_k}^{T} \sB_{j_k}}{\Dsw_{j_k}^{T} \sB_{j_k} \Dsw_{j_k}} + \frac{\Dsg_{j_k} {\Dsg_{j_k}}^{T}}{{\Dsg_{j_k}}^{T} \Dsw_{j_k}}
    \end{split}\\
    \begin{split}\label{eq: sgd_B_inverse_formula}
    \sB_{j_{k+1}}^{-1} = \left(\textbf{I} - \frac{\Dsw_{j_k}\Dsg_{j_k}^{T}}{\Dsg_{j_k}^{T} \Dsw_{j_k}}\right)\sB_{j_k}^{-1}\left(\textbf{I} - \frac{\Dsg_{j_k} \Dsw_{j_k}^{T}}{\Dsg_{j_k}^{T} \Dsw_{j_k}}\right) + \frac{\Dsw_{j_k}\Dsw_{j_k}^{T}}{\Dsg_{j_k}^{T} \Dsw_{j_k}}
    \end{split}
\end{align}
This iteration
 has the same initialization as $\B_{j_k}$ and $\B_{j_k}^{-1}$ but relies on the history information collected from the SGD-based training process $[\Dsw_{j_{0}}$, $\Dsw_{j_1}$,$\dots$,$ \Dsw_{j_{m-1}}]$ and $[\Dsg_{j_{0}}, \Dsg_{j_1}$, $\dots$, $\Dsg_{j_{m-1}}]$ where $\Dsw_{j_{x}} = \sw_{j_{x}} - \isw_{j_{x}}$ and $\Dsg_{j_{x}} = \sgrad(\isw_{j_x}) - \sgrad(\sw_{j_x})$ ($x = 0,1,2,\dots m-1$). By the same argument as the proof of Lemma \ref{assp: B_K_product_bound}, the following inequality holds:
\begin{align}\label{sgdcond}
    \begin{split}
        K_1 \|\textbf{z}\|^2 \leq \textbf{z}^T \sB_{j_k}\textbf{z} \leq K_2 \|\textbf{z}\|^2
    \end{split}
\end{align}

where $K_1:=\frac{1}{(1+\frac{L}{\mu})^{2m}\frac{L}{\mu} + \frac{1-(1+\frac{L}{\mu})^{2m}}{1-(1+\frac{L}{\mu})^2} \frac{1}{\mu}}$ and $K_2:=(m+1)L$, which are both positive values representing a lower bound and an upper bound on the eigenvalues of $\sB_{j_k}$.

\subsubsection{Proof preliminaries}

Similar to the argument for the GD-version of \Increm, we can give an upper bound on $\delta_{t,S}$:
\begin{lemma}[Upper bound on $\delta_{t,S}$]\label{lemma: delta_t_zeta_bounds} Define $\delta_{t,S} = \sgrad(\usw_t) - \sugrad(\usw_t)$.
Then $\|\delta_{t,S}\| \leq 2c_2 \frac{\Delta B_t}{B}$.
Moreover, with probability higher than $1-t\times 2\exp(-2\sqrt{B})$, 
$$\|\delta_{t',S}\| \leq 2c_2(\frac{r}{n} + \frac{1}{B^{1/4}})$$ uniformly over all iterations $t'\leq t$.
\label{updels}
\end{lemma}

\begin{proof}

Recall that $$\sgrad(\usw_t) = \frac{1}{B}\sum_{i\in \miniB_t} \nabla F_i(\usw_t)$$ and $$\sugrad(\usw_t) = \frac{1}{B - \Delta B_t}\sum_{i\in \miniB_t, i \not\in R} \nabla F_i(\usw_{t}).$$ By subtracting $\sgrad(\usw_t)$ from $\sugrad(\usw_t)$, we have:
\begin{align*}
\begin{split}
& \|\sugrad(\usw_t) - \sgrad(\usw_t)\| \\
    & = \|\frac{1}{B}\sum_{i\in \miniB_t} \nabla F_i(\usw_t) - \frac{1}{B - \Delta B_t}\sum_{i\in \miniB_t, i \not\in R} \nabla F_i(\usw_{t})\|\\
    & = \|\frac{1}{B}\sum_{i\in \miniB_t, i \in R} \nabla F_i(\usw_t) + (\frac{1}{B}-\frac{1}{B - \Delta B_t})\sum_{i\in \miniB_t, i \not\in R} \nabla F_i(\usw_{t})\|
\end{split}
\end{align*}

Then by using the triangle inequality and the fact that $\|\nabla F_i(\usw_{t})\| \leq c_2$ (Assumption \ref{assp: gradient_upper_bound}), the formula above can be bounded by $\frac{2\Delta B_t c_2}{B}$.

Because of the randomness from SGD, the $r$ removed samples can be viewed as uniformly distributed among all $n$ training samples. Each sample is included in a mini-batch according to the outcome of a Bernoulli($r/n$) random variable $\textbf{S}_i$.  Within a single mini-batch $\miniB_{t'}$ at the iteration $t'$, we get $\E(\sum_{i\in \miniB_{t'}} \textbf{S}_i) = \E(\Delta B_{t'}) = B\frac{r}{n}$ and $\var(\sum_{i\in \miniB_{t'}} \textbf{S}_i) = B \frac{r}{n}(1-\frac{r}{n})$. So in terms of the random variable $\frac{\Delta B_{t'}}{B}$, its expectation and variance will be $\E(\frac{\Delta B_{t'}}{B}) = \frac{r}{n}$ and $\var(\frac{\Delta B_{t'}}{B}) = \frac{r}{Bn}(1-\frac{r}{n})$.

Then based on Hoeffding's inequality, the following inequality holds:

\begin{align*}
    \Pr(|\frac{\Delta B_{t'}}{B} - \frac{r}{n}| \leq \epsilon) \geq 1-2\exp(-2\epsilon^2B). 
\end{align*}

Then by setting $\epsilon = \frac{1}{B^{1/4}}$
the formula above can be written as:
\begin{align*}
    \Pr(|\frac{\Delta B_{t'}}{B} - \frac{r}{n}| \geq \frac{1}{B^{1/4}}) \leq 2\exp(-2\sqrt{B})
\end{align*}

Then by taking the union for all the iterations before $t$, we get: with probability higher than $1-t\times 2\exp(-2\sqrt{B})$, $$|\frac{\Delta B_{t'}}{B} - \frac{r}{n}| \leq \frac{1}{B^{1/4}}$$ and thus 
\begin{align}\label{eq: delta_b_bound}
\frac{\Delta B_{t'}}{B} \leq \frac{r}{n} + \frac{1}{B^{1/4}}
\end{align}
 for all $t' \leq t$.
\end{proof}
In what follows, we use $\Psi_1$ to represent $\Psi_1:=2\exp(-2\sqrt{B})$, which goes to 0 with large $B$.

Next we provide a bound for the sum of random sampled Hessian matrices within a minibatch in SGD.

\begin{theorem}[Hessian matrix bound in SGD]\label{theorem: stochastic_hessian_bound}
With probability higher than $$1-\prob,$$ for a given iteration $t$, $\|\left(\frac{1}{B}\sum_{i \in \miniB_t} \bH_i(\sw_t)\right) - \bH(\sw_t)\| \leq \rhbound$ where $p$ represents the number of model parameters.
\end{theorem}

\begin{proof}
We consider using the matrix Bernstein inequality, Lemma \ref{lemma: Bernstein_inequality2}. We define the random matrix $\textbf{S}_i = \frac{\bH_i(\w) - \bH(\w)}{B}$ ($i\in \miniB_t$). Due to the randomness from SGD, 
we know that $\E(\textbf{S}_i) = \E(\frac{\bH_i(\w) - \bH(\w)}{B}) = \textbf{0}$. Using the sum $\textbf{Z}$ as required in Lemma \ref{lemma: Bernstein_inequality2}, $\textbf{Z} = \left(\frac{1}{B}\sum_{i \in \miniB_t} \bH_i(\w)\right) - \bH(\w)$. Also note that $\bH(\w)$ and $\bH_i(\w)$ are both $p\times p$ matrices, so $d_1=d_2=p$ in Lemma \ref{lemma: Bernstein_inequality2}.

Furthermore, for each $\textbf{S}_i = \frac{\bH_i(\w) - \bH(\w)}{B}$, its norm is bounded by $\frac{2L}{B}$ based on the smoothness condition, which means that $J = \frac{2L}{B}$ in Lemma \ref{lemma: Bernstein_inequality2}. Then we can explicitly calculate the upper bound on $E(\textbf{S}_i\textbf{S}_i^*)$ and $V(\textbf{Z})$:
\begin{align*}
    \begin{split}
        \|\E(\textbf{S}_i\textbf{S}_i^*)\| \leq \E(\|\textbf{S}_i\textbf{S}_i^*\|) \leq \E(\|\textbf{S}_i\|\|\textbf{S}_i^*\|) \leq J^2 = \frac{4L^2}{B^2},
    \end{split}\\
    \begin{split}
        V(Z) \leq \sum_{i\in \miniB_t} \frac{4L^2}{B^2} = \frac{4L^2}{B}.
    \end{split}
\end{align*}

Thus by plugging the above expression into equation \eqref{eq: Bernstein_ineq_1} and \eqref{eq: Bernstein_ineq_2}, we get:
\begin{align}
\begin{split}\label{eq: Bernstein_ineq_eval_1}
        &P(\|\textbf{Z}\| \geq x) = \Pr(\|\left(\frac{1}{B}\sum_{i \in \miniB_t} \bH_i(\w)\right) - \bH(\w)\| \geq x)\\ 
        &\leq (d_1 + d_2)\exp\left({\frac{-x^2}{\frac{4L^2}{B} + \frac{2Lx}{3B}}}\right)
        = 2p\exp\left({\frac{-x^2}{\frac{4L^2}{B} + \frac{2Lx}{3B}}}\right) , \forall x \geq 0
    \end{split}
\end{align}
\begin{align}
    \begin{split}\label{eq: Bernstein_ineq_eval_2}
        &\E(\|\textbf{Z}\|) = \E\left(\|\left(\frac{1}{B}\sum_{i \in \miniB_t} \bH_i(\w)\right) - \bH(\w)\|\right)\\ &\leq \sqrt{\frac{8L^2}{B} \log(d_1 + d_2)} + \frac{2L}{3B}\log(d_1 + d_2) = \sqrt{\frac{8L^2}{B} \log(2p)} + \frac{2L}{3B}\log(2p).
    \end{split}
\end{align}

Then by setting $x=\rhbound$, Equation \eqref{eq: Bernstein_ineq_eval_1} becomes:

\begin{align}
\begin{split}\label{eq: Bernstein_ineq_eval_3}
        &\Pr(\|\textbf{Z}\| \geq \rhbound) \\
        & = Pr\left(\|\left(\frac{1}{B}\sum_{i \in \miniB_t} \bH_i(\w)\right) - \bH(\w)\| \geq \rhbound\right)\\ 
        &\leq (2p)\exp\left({\frac{-\frac{L^2 log(2p)}{\sqrt{B}}}{\frac{4L^2}{B} + \frac{2L^2}{3B}\left(\frac{\log^2(2p)}{B}\right)^{1/4}}}\right)  = (2p)\exp\left({-\frac{\log(2p)\sqrt{B}}{4 + \frac{2}{3}\left(\frac{\log^2(2p)}{B}\right)^{1/4}}}\right).
    \end{split}
\end{align}

For large mini-batch size $B$, both $\rhbound$ and $(2p)\exp\left({-\frac{\log(2p)\sqrt{B}}{4 + \frac{2}{3}\left(\frac{\log^2(2p)}{B}\right)^{1/4}}}\right)$ are approaching 0. 
\end{proof}
In what follows, we use $\Psi_2$ to denote the probability $\Psi_2:=(2p)\exp\left({-\frac{\log(2p)\sqrt{B}}{4 + \frac{2}{3}\left(\frac{\log^2(2p)}{B}\right)^{1/4}}}\right)$.

Based on this result, we can derive an SGD version of Theorem \ref{theorem: delta_model_para_bound} as below, which also relies on a preliminary estimate on the bound on $\|\isw_t - \sw_t\|$:

\begin{theorem}[Error in mean Hessian, and in secant equation with incorrect quasi-Hessian for SGD]\label{theorem: delta_model_para_bound2s}
Suppose that $\|\sw_{t'} - \isw_{t'}\| \leq M_1\frac{1}{\frac{1}{2}-\frac{r}{n} - \frac{1}{B^{1/4}}}(\frac{r}{n} + \frac{1}{B^{1/4}})$ and 
$$\|\left(\frac{1}{B}\sum_{i \in \miniB_{t'}} \bH_i(\sw_{t'})\right) - \bH(\sw_{t'})\| \leq \rhbound$$ hold for any $t' \leq t$ where $M_1 = \frac{2c_2}{\mu}$, $\mu$ is from Assumption \ref{assp: strong_convex smooth} and $c_2$ is from Assumption \ref{assp: gradient_upper_bound}. 
Let $e = \frac{L(L+1) + K_2 L}{\mu  K_1}$ for the upper and lower bounds $K_1,K_2$ on the eigenvalues of the quasi-Hessian from equation \eqref{sgdcond} and for the Lipshitz constant $c_0$ of the Hessian. For any $t_1,t_2$ such that $1 \leq t_1 < t_2 \leq t$, 
we have:
$$\|\sbH_{t_1} - \sbH_{t_2}\| \leq 2\rhbound + c_0 d_{t_1, t_2} + 3c_0 M_1\frac{1}{\frac{1}{2}-\frac{r}{n} - \frac{1}{B^{1/4}}}(\frac{r}{n} + \frac{1}{B^{1/4}}).$$ 
For any $j_1,j_2,\dots,j_m$ such that $j_m \leq t' \leq j_m + T_0 - 1$ and $t' \leq t$, we have:
$$\|\Dsg_{j_k} - \sB_{j_q}\Dsw_{j_k}\|$$
$$\leq \left[(1+e)^{j_q-j_k-1} - 1\right]\cdot [2\rhbound + c_0d_{j_k,j_q} + \frac{3c_0M_1}{\frac{1}{2}-\frac{r}{n} - \frac{1}{B^{1/4}}}(\frac{r}{n} + \frac{1}{B^{1/4}})] \cdot s_{j_m,j_1}.$$
Here $s_{j_m,j_1} = \max\left(\|\Dsw_{a}\|\right)_{a=j_1,j_{2},\dots,j_m}$, $d_{j_k,j_q} = \max\left(\|\sw_{a} - \sw_{b}\|\right)_{j_k \leq a \leq b \leq j_q}$, $\sbH_{t}$ is the average of the Hessian matrix evaluated between $\sw_{t}$ and $\isw_t$ for the samples in mini-batch $\miniB_{t}$:
 $$\sbH_{t} = \frac{1}{B}\sum_{i \in \miniB_t} \int_{0}^1 \bH_{i}(\sw_t + x(\usw_t - \sw_t))dx.$$
\end{theorem}

\begin{proof}



First of all, let us bound $\|\sbH_{t_1} - \int_0^1 \bH(\sw_{t_1} + x(\isw_{t_1} - \sw_{t_1}))dx\|$ by adding and subtracting $\frac{1}{B}\sum_{i\in \miniB_{t_1}}\bH_i(\sw_{t_1})$ and $\bH(\sw_{t_1})$ inside the norm:
\begin{align*}
    \begin{split}
        & \|\sbH_{t_1} - \int_0^1 \bH(\sw_{t_1} + x(\isw_{t_1} - \sw_{t_1}))dx\|\\
        & = \|\int_{0}^1\frac{1}{B}\sum_{i\in \miniB_{t_1}}\bH_i(\sw_{t_1} + x(\isw_{t_1} - \sw_{t_1}))dx - \int_0^1 \bH(\sw_{t_1} + x(\isw_{t_1} - \sw_{t_1}))dx\| \\
        & = \|\int_{0}^1\frac{1}{B}\sum_{i\in \miniB_{t_1}}(\bH_i(\sw_{t_1} + x(\isw_{t_1} - \sw_{t_1})) - \bH_i(\sw_{t_1}))dx + \frac{1}{B}\sum_{i\in \miniB_{t_1}}\bH_i(\sw_{t_1}) \\
        & - \int_0^1 (\bH(\sw_{t_1} + x(\isw_{t_1} - \sw_{t_1})) - \bH(\sw_{t_1}))dx - \bH(\sw_{t_1})\|.
    \end{split}
\end{align*}

Then by using the triangle inequality and Assumption \ref{assp: hessian_continuous}, the formula above can be bounded as:
\begin{align}\label{eq: single_random_hessian_bound}
\begin{split}
    & \leq \int_{0}^1\frac{1}{B}\sum_{i\in \miniB_{t_1}}\|\bH_i(\sw_{t_1} + x(\isw_{t_1} - \sw_{t_1})) - \bH_i(\sw_{t_1})\|dx  \\
        & + \int_0^1 \|\bH(\sw_{t_1} + x(\isw_{t_1} - \sw_{t_1})) - \bH(\sw_{t_1})\|dx \\
        & + \|\int_0^1 \frac{1}{B}\sum_{i\in \miniB_{t_1}}\bH_i(\sw_{t_1}) - \bH(\sw_{t_1}) \|\\
        & \leq \frac{1}{B}(\sum_{i\in \miniB_{t_1}} \int_0^1 c_0 x \|\isw_{t_1} - \sw_{t_1}\|dx) + \int_0^1 c_0x \|\isw_{t_1} - \sw_{t_1}\|dx\\
        & + \|\frac{1}{B}\sum_{i\in \miniB_{t_1}}\bH_i(\sw_{t_1}) - \bH(\sw_{t_1}) \|\\
        & \leq c_0\|\isw_{t_1} - \sw_{t_1}\| + \|\frac{1}{B}\sum_{i\in \miniB_{t_1}}\bH_i(\sw_{t_1}) - \bH(\sw_{t_1}) \|.
\end{split}
\end{align}

Then based on the above results, we can compute the bound on $\|\sbH_{t_1} - \sbH_{t_2}\|$, for which we use the triangle inequality first:

\begin{align}\label{eq: sgd_hessian_diff_bound0}
    \begin{split}
       & \|\sbH_{t_1} - \sbH_{t_2}\|\\
       & = \|\sbH_{t_1} - \int_0^1 \bH(\sw_{t_1} + x(\isw_{t_1} - \sw_{t_1}))dx\|\\
       & + \|\int_0^1 \bH(\sw_{t_1} + x(\isw_{t_1} - \sw_{t_1}))dx - \int_0^1 \bH(\sw_{t_2} + x(\isw_{t_2} - \sw_{t_2}))dx\|\\
       & + \|\int_0^1 \bH(\sw_{t_2} + x(\isw_{t_2} - \sw_{t_2}))dx - \sbH_{t_2}\|.
    \end{split}
\end{align}


Then by using the result from Formula \eqref{eq: single_random_hessian_bound}, this term can be further bounded as:
\begin{align*}
    \begin{split}
        & \leq c_0\|\isw_{t_1} - \sw_{t_1}\| + \|\int_0^1 \frac{1}{B}\sum_{i\in \miniB_{t_1}}\bH_i(\sw_{t_1}) - \bH(\sw_{t_1}) \|\\
        & + c_0\|\isw_{t_2} - \sw_{t_2}\| + \|\int_0^1 \frac{1}{B}\sum_{i\in \miniB_{t_2}}\bH_i(\sw_{t_2}) - \bH(\sw_{t_2}) \|\\
        & + \|\int_0^1 \bH(\sw_{t_1} + x(\isw_{t_1} - \sw_{t_1}))dx - \int_0^1 \bH(\sw_{t_2} + x(\isw_{t_2} - \sw_{t_2}))dx\|.
    \end{split}
\end{align*}


Since $$\|\left(\frac{1}{B}\sum_{i \in \miniB_{t'}} \bH_i(\sw_{t'})\right) - \bH(\sw_{t'})\| \leq \rhbound$$ for any $t' \leq t$, then the formula above can be bounded as:

\begin{align*}
    \begin{split}
        & \leq 2\rhbound + c_0 \|\sw_{t_1} - \sw_{t_2}\| + \frac{c_0}{2}\|\sw_{t_1} - \isw_{t_1}\|\\
        & + \frac{c_0}{2}\|\isw_{t_2}-\sw_{t_2}\| + c_0\|\isw_{t_1} - \sw_{t_1}\| + c_0\|\isw_{t_2} - \sw_{t_2}\| \\
        & = 2\rhbound + 3c_0M_1\frac{1}{\frac{1}{2}-\frac{r}{n} -\frac{1}{B^{1/4}}}(\frac{r}{n} + \frac{1}{B^{1/4}}) + c_0 d_{t_1, t_2}.
    \end{split}
\end{align*}


This finishes the proof of the first inequality. Then by defining $$f=\left(2\rhbound + 3c_0M_1\frac{1}{\frac{1}{2}-\frac{r}{n} - \frac{1}{B^{1/4}}}(\frac{r}{n} + \frac{1}{B^{1/4}}) + c_0 d_{j_k, j_q}\right)s_{j_m,j_1}$$ and using the same argument as Equation \eqref{eq: delta_w_product}-\eqref{eq: delta_w_product_final} (except that $\Dw$ and $\Dg$ are replaced with $\Dsw$ and $\Dsg$), the following inequality thus holds:
\begin{align}\label{eq: delta_gap_sgd_1}
    \begin{split}
        &b_{j_q} 
        =\|\Dsg_{j_k} - \left(\sB_{j_q} - \frac{\sB_{j_q}\Dsw_{j_q} \Dsw_{j_q}^{T} \sB_{j_q}}{\Dsw_{j_q}^{T} \sB_{j_q} \Dsw_{j_q}} + \frac{\Dsg_{j_q} \Dsg_{j_q}^{T}}{\Dsg_{j_q}^{T} \Dsw_{j_q}}\right)\Dsw_{j_k}\| \\
        & \le [(1+e)^{j_q-j_k} - 1] f
    \end{split}
\end{align}

and thus
$$\|\Dsg_{j_k} - \sB_{j_q}\Dsw_{j_k}\| \leq [(1+e)^{j_q-j_k-1} - 1] f,$$

which finishes the proof.

For simplicity, we denote $M_1^S := M_1\frac{1}{\frac{1}{2}-\frac{r}{n} - \frac{1}{B^{1/4}}}$. So the preliminary estimate of the bound on $\|\sw_{t'} - \isw_{t'}\|$ becomes: $\|\sw_{t'} - \isw_{t'}\| \leq M_1^S(\frac{r}{n} + \frac{1}{B^{1/4}})$

\end{proof}

Similarly, we get a SGD-version of Corollary \ref{corollary: approx_hessian_real_hessian_bound}:

\begin{corollary}[Approximation accuracy of Quasi-Hessian to mean Hessian]\label{corollary: approx_hessian_real_hessian_bound_sgd}
Suppose that $\|\sw_{t'} - \isw_{t'}\| \leq M_1^S (\frac{r}{n} + \frac{1}{B^{1/4}})$ and 
$$\|\left(\frac{1}{B}\sum_{i \in \miniB_{t'}} \bH_i(\sw_{t'})\right) - \bH(\sw_{t'})\| \leq \rhbound$$ hold for any $t' \leq t$.
$M_1$ and $M_1^S$ are provided in Theorem \ref{theorem: delta_model_para_bound2s}, i.e. $M_1 = \frac{2c_2}{\mu}$ and $M_1^S = M_1\frac{1}{\frac{1}{2}-\frac{r}{n} - \frac{1}{B^{1/4}}}$. Then for any $t'$ and $j_m$ such that $ j_m \le t' \leq j_m + T_0 - 1$ and $t' \leq t$, 
the following inequality holds:
\begin{align*}
  &\|\sbH_{t'} - \sB_{j_m}\| \leq \xi_{j_1,j_m}^S\\
  &:= A( d_{j_1,j_m + T_0 - 1}
+ 3M_1^S(\frac{r}{n} + \frac{1}{B^{1/4}}) + \frac{2}{c_0}\rhbound)
\end{align*}

where recall again that $c_0$ is the Lipschitz constant of the Hessian, $d_{j_1,j_m + T_0 - 1}$ is the maximal gap between the iterates of the SGD algorithm on the full data from $j_1$ to $j_m+T_0-1$ and $A = \frac{c_0\sqrt{m}[(1+e)^{m}-1]}{c_1} + c_0$ in which 
$e$ is a problem dependent constant defined in Theorem \ref{theorem: delta_model_para_bound2s}, $c_1$ is the "strong independence" constant from Assumption \eqref{assp: singular_lower_bound}.
\end{corollary}

This proof is similar to the proof of Corollary \ref{corollary: approx_hessian_real_hessian_bound}. First of all, $\bH$, $\B$, $\xi_{j_1,j_m}$ in Corollary \ref{corollary: approx_hessian_real_hessian_bound} are replaced with $\sbH$, $\sB$, $\xi_{j_1,j_m}^S$. Second, 
Theorem \ref{theorem: delta_model_para_bound2s} holds and thus the following inequality holds:
$$\|\sbH_{t'} - \sbH_{j_m}\| \leq 2\rhbound + c_0 d_{t', j_m} + 3c_0 M_1^S(\frac{r}{n} + \frac{1}{B^{1/4}})$$

by using strong independence from Assumption \ref{assp: singular_lower_bound}, $\|\sbH_{j_m} - \sB_{j_m}\|$ can be bounded as:

\begin{align}\label{eq: hessian_approx_bound_sgd}
    \begin{split}
        & \|\sbH_{j_m} - \sB_{j_m}\| \leq  \sqrt{m}[(1+e)^{m}-1]\frac{c_0}{c_1}\cdot \left(d_{j_1,j_m + T_0 - 1}
+ 3M_1^S(\frac{r}{n} + \frac{1}{B^{1/4}}) + \frac{2}{c_0}\rhbound\right)
    \end{split}
\end{align}

Then by combining the two formulas above, we know that Corollary \ref{corollary: approx_hessian_real_hessian_bound_sgd} holds. Note that the definition of $\xi_{j_1,j_m}^S$ can be rewritten as below:
\begin{align}\label{eq: xi_simple_def}
\begin{split}
& \xi_{j_1,j_m}^S = A( d_{j_1,j_m + T_0 - 1}
+ 3M_1^S(\frac{r}{n} + \frac{1}{B^{1/4}}) + \frac{2}{c_0}\rhbound)\\
& =: Ad_{j_1,j_m + T_0 - 1} + A_1 \frac{r}{n} + A_2 \frac{1}{B^{1/4}}
\end{split}
\end{align}
in which $A_1 := 3AM_1^S$ and $A_2 := 3AM_1^S + \frac{2AL(log(2p))^{1/2}}{c_0}$.


We can do a similar analysis to Lemma \ref{Lemma: bound_d_j} by simply replacing $\w_t$ and $F(*)$ with $\sw_t$ and $\sgrad$:

\begin{lemma}\label{Lemma: bound_d_j_sgd}
Let us use the definition of $d_{k, q}$ from Theorem \ref{theorem: delta_model_para_bound2s}:
$$d_{k,q} = \max\left(\|\sw_{a} - \sw_{b}\|\right)_{k \leq a \leq b \leq q}$$
where $k < q \leq t$, 
then $d_{k, q} \leq (1-\eta\mu)^{k}d_{0,q-j} + 2c_2 \left(\frac{(\log(p+1))^2}{B}\right)^{1/4}$ holds  with probability higher than $1- t(p+1)\exp\left({-\frac{\log(p+1)\sqrt{B}}{4 + \frac{2}{3}\left(\frac{(\log(p+1))^2}{B}\right)^{1/4}}}\right) $.
\end{lemma}
\begin{proof}

According to Lemma \ref{lemma: Bernstein_inequality2}, we can define a random matrix $\textbf{S}_i = \frac{1}{B}(\nabla F_{i}(\sw_a) - \nabla F(\sw_a)$ where recall that $\nabla F(\sw_a) = \frac{1}{n}\sum_{i=1}^n \nabla F_{i,a}(\sw_a)$ ($i\in \miniB_t$). 
Due to the randomness from SGD, we know that $\E(\textbf{S}_i) = \textbf{0}$. Based on the definition of $\textbf{Z}$ in Lemma \ref{lemma: Bernstein_inequality2}, $\textbf{Z} = \frac{1}{B}\sum_{i \in \miniB_t}\nabla F_{i}(\sw_a) - \nabla F(\sw_a)$. Also note that $\nabla F_{i}(\sw_a)$ and $\nabla F(\sw_a)$ are both $p\times 1$ matrices, so $d_1=p$ and $d_2 = 1$ in Lemma \ref{lemma: Bernstein_inequality2}.

Moreover according to Assumption \ref{assp: gradient_upper_bound}, $\|\nabla F_{i}(\sw_a)\| \leq c_2$. Then we know that $V(\textbf{Z}) \leq \frac{4c_2^2}{B}$ and $\|\textbf{S}_i\| \leq \frac{2c_2}{B}$. So according to Lemma \ref{lemma: Bernstein_inequality2}, the following inequality holds:

\begin{align}
\begin{split}\label{eq: Bernstein_ineq_eval_4}
        &P(\|\textbf{Z}\| \geq x) = \Pr(\|\frac{1}{B}\sum_{i \in \miniB_t}\nabla F_{i}(\sw_a) - \nabla F(\sw_a)\| \geq x)\\ 
        &\leq (d_1 + d_2)\exp\left({\frac{-x^2}{\frac{4c_2^2}{B} + \frac{2c_2x}{3B}}}\right)= (p+1)\exp\left({\frac{-x^2}{\frac{4c_2^2}{B} + \frac{2c_2x}{3B}}}\right) , \forall x \geq 0
    \end{split}
\end{align}

By setting $x = c_2 \left(\frac{(\log(p+1))^2}{B}\right)^{1/4}$, the formula above is evaluated as:
\begin{align*}\label{eq: miniB_gradient_full_gradient_diff}
    \begin{split}
        & \Pr(\|\frac{1}{B}\sum_{i \in \miniB_t}\nabla F_{i}(\sw_a) - \nabla F(\sw_a)\| \geq c_2 \left(\frac{(\log(p+1))^2}{B}\right)^{1/4}) \\
        & \leq  (p+1)\exp\left({-\frac{\log(p+1)\sqrt{B}}{4 + \frac{2}{3}\left(\frac{(\log(p+1))^2}{B}\right)^{1/4}}}\right)
    \end{split}
\end{align*}

So by taking the union for the first $t$ iterations, then with probability higher than $1- t(p+1)\exp\left({-\frac{\log(p+1)\sqrt{B}}{4 + \frac{2}{3}\left(\frac{(\log(p+1))^2}{B}\right)^{1/4}}}\right) $, the following inequality holds for all $t' \leq t$:
\begin{align}
    \|\frac{1}{B}\sum_{i \in \miniB_{t'}}\nabla F_{i}(\sw_a) - \nabla F(\sw_a)\| \leq c_2 \left(\frac{(\log(p+1))^2}{B}\right)^{1/4}
\end{align}

Then by using the similar arguments to Lemma \ref{Lemma: bound_d_j}, we get:
$$\|\sw_a-\sw_{b}\| \leq (1-\eta\mu)^z \|\sw_{a-z}-\sw_{b-z}\| + \frac{2c_2}{\mu} \left(\frac{(\log(p+1))^2}{B}\right)^{1/4} = (1-\eta\mu)^z \|\sw_{a-z}-\sw_{b-z}\| + M_1 \left(\frac{(\log(p+1))^2}{B}\right)^{1/4}$$ and thus $d_{k, q} \leq (1-\eta\mu)^{k}d_{0,q-k} + M_1 \left(\frac{(\log(p+1))^2}{B}\right)^{1/4}$ holds with probability higher than $1- t(p+1)\exp\left({-\frac{\log(p+1)\sqrt{B}}{4 + \frac{2}{3}\left(\frac{(\log(p+1))^2}{B}\right)^{1/4}}}\right) $. In what follows, we use $\Psi_3$ to denote $(p+1)\exp\left({-\frac{\log(p+1)\sqrt{B}}{4 + \frac{2}{3}\left(\frac{(\log(p+1))^2}{B}\right)^{1/4}}}\right)$. 

\end{proof}

Then by using the definition of $\xi_{j_1,j_m}^S$, the following inequality holds with probability higher than $1- t\Psi_3$ for any $x$ such that for $j_0+(x+m-1)T_0 \leq t$, the following inequality holds:
\begin{align}\label{eq: xi_bound}
\begin{split}
&\xi_{j_1,j_m}^S = \xi_{j_0+xT_0,j_0+(x+m-1)T_0}^S\leq (1-\eta\mu)^{xT_0}Ad_{j_0, j_0+mT_0-1}
\\
& +A_1\frac{r}{n} + A_2\frac{1}{B^{1/4}}+A\sgdb
\end{split}
\end{align}

\subsubsection{Main recursions}
We bound the difference between $\isw_{t}$ and $\usw_{t}$. First we bound $\|\sw_{t} - \usw_{t}\|$:

\begin{theorem}[Bound between iterates on full and the leave-$r$-out dataset]\label{wu_sgd} 
When $$\frac{\Delta B_{t'}}{B}\leq \frac{r}{n} + \frac{1}{B^{1/4}}$$ holds for all $t' < t$, $\|\sw_{t} - \usw_{t}\| \leq \frac{2c_2}{\mu}(\frac{r}{n} + \frac{1}{B^{1/4}})$.
Since with probability higher than $1-t\times \Psi_1$, $$\frac{\Delta B_{t'}}{B}\leq \frac{r}{n} + \frac{1}{B^{1/4}}$$ holds for all $t' < t$. Then with the same probability, $\|\sw_{t'+1} - \usw_{t'+1}\| \leq M_1(\frac{r}{n} + \frac{1}{B^{1/4}})$ for all iterations $t'< t$, where recall that $M_1=\frac{2c_2}{\mu}$.
\end{theorem}

Similarly, we can bound the difference between $\iw_{t}$ and $\w_{t}$. 

\begin{theorem}[Bound between iterates on full data and incrementally updated ones]\label{bfi_sgd}
Suppose that for at some iteration $t$ and any given $t' \leq t$ such that $j_m' \leq t' \leq j_m' +T_0-1$,
  we have the following bounds:
\begin{enumerate}
\item $\|\sbH_{t'} - \sB_{j_m'}\| \leq \xi_{j_1',j_m'}^S = A d_{j_1',j_m' + T_0 - 1}
+ A \frac{3}{\frac{1}{2}-\frac{r}{n}}M_1(\frac{r}{n} + \frac{1}{B^{1/4}}) + A\frac{2}{c_0} \rhbound$;
\item $\frac{\Delta B_{t'}}{B} \leq \frac{r}{n} + \frac{1}{B^{1/4}}$;
\item Formula \eqref{eq: xi_bound} holds for any $x$ such that $j_0 + (x+m-1)T_0 \leq t$;
\item $\xi_{j_0, j_0 + (m-1)T_0}^S + A\times \sgdb \leq \frac{\mu}{2}$,\end{enumerate}  

then $$\|\isw_{t'+1} - \sw_{t'+1} \|\leq \frac{2c_2}{(\frac{1}{2}-\frac{r}{n} - \frac{1}{B^{1/4}})\mu}(\frac{r}{n} + \frac{1}{B^{1/4}}) = M_1^S (\frac{r}{n} + \frac{1}{B^{1/4}})$$ for any $t' \leq t$ 
Recall that $c_0$ is the Lipshitz constant of the Hessian, $M_1$ and $A$ are defined in Theorem \ref{wu_sgd} and Corollary \ref{corollary: approx_hessian_real_hessian_bound_sgd} respectively, and do not depend on $t$.

in particular for all $t$, the following inequality holds:
$$\|\isw_{t+1} - \sw_{t+1} \|\leq \frac{2c_2}{(\frac{1}{2}-\frac{r}{n} - \frac{1}{B^{1/4}})\mu}(\frac{r}{n} + \frac{1}{B^{1/4}}) = M_1^S (\frac{r}{n} + \frac{1}{B^{1/4}}).$$
\end{theorem}

Similarly, we will show that both inequalities $\|\sbH_{t} - \sB_{j_m}\| \leq \xi_{j_1,j_m}^S$ and $\|\isw_{t+1} - \sw_{t+1} \| \leq M_1^S (\frac{r}{n} + \frac{1}{B^{1/4}})$ hold for all iterations $t$.



\begin{theorem}[Bound between iterates on full data and incrementally updated ones (all iterations)]\label{bfi2_sgd}
Suppose that there are $T$ iterations in total for each training phase, then with probability higher than $1-T\times (\Psi_1 + \Psi_2 + \Psi_3),$ for any $t$ where $j_m < t < j_m +T_0 - 1$, 
$\|\isw_t - \sw_t\| \leq \frac{1}{\frac{1}{2}-\frac{r}{n}-\frac{1}{B^{1/4}}}M_1(\frac{r}{n} + \frac{1}{B^{1/4}})$ and $\|\sbH_{t} - \sB_{j_m}\| \leq \xi_{j_1,j_m}^S$, where $\xi_{j_1,j_m}^S$ is defined in Corollary \ref{corollary: approx_hessian_real_hessian_bound_sgd}, $\Psi_1$ is defined in Lemma \ref{lemma: delta_t_zeta_bounds}, $\Psi_2$ is defined in Theorem \ref{theorem: stochastic_hessian_bound} and $\Psi_3$ is defined in Lemma \ref{Lemma: bound_d_j_sgd}.
\end{theorem}

Then we have the following bound for $\|\uw_t - \iw_t\|$.

\begin{theorem}[Main result: Bound between true and incrementally updated iterates for SGD]\label{main10_sgd}

Suppose that there are $T$ iterations in total for each training phase, then with probability higher than $1-T\times (\Psi_1 + \Psi_2+ \Psi_3),$
the result $\isw_t$ of Algorithm \ref{alg: update_algorithm} approximates the correct iteration values $\usw_t$ at the rate
$$\|\usw_t - \isw_t\| \leq o((\frac{r}{n} + \frac{1}{B^{1/4}})).$$
 So $\|\usw_t - \isw_t\|$ is of a lower order than $(\frac{r}{n} + \frac{1}{B^{1/4}})$.

\end{theorem}



\subsubsection{Proof of Theorem \ref{wu_sgd}}

\begin{proof}
By subtracting $\sw_t - \usw_t$, taking the matrix norm and using the update rule in equation \eqref{eq: sw_update} and \eqref{eq: usw_update}, we get:

\begin{align}\label{eq: bound_w_w_u_diffs}
    \begin{split}
        & \|\sw_{t+1} - \usw_{t+1}\|\\
        & = \|\sw_{t} - \eta \sgrad(\sw_t) - \left(\usw_{t} - \eta  \sugrad(\usw_t)\right)\| \\
        & = \|\sw_{t} - \usw_{t} - \eta  \left(\sgrad(\sw_t) - \sugrad(\usw_t)\right)\| \\
        & = \|\sw_{t} - \usw_{t} - \eta  (\sgrad(\sw_t) - \sgrad(\usw_t)\\
        & + \sgrad(\usw_t) - \sugrad(\usw_t))\| \\
        & = \|\sw_{t} - \usw_{t} - \eta  \left(\sgrad(\sw_t) - \sgrad(\usw_t)\right) +\\
        & \eta  \left(\sgrad(\usw_t) - \sugrad(\usw_t)\right)\| 
    \end{split}
\end{align}

By Cauchy mean-value theorem and the triangle inequality, the above formula becomes:
\begin{align*}
    \begin{split}
        & \leq \|\sw_{t} - \usw_{t} - \eta  (\frac{1}{B}\int_0^1 \sum_{i \in \miniB_t}\bH_i\left(\sw_{t} + x\left(\usw _{t}-\sw_t\right)\right)dx) \left(\sw_{t} - \usw_{t}\right)\| + \eta  \|\delta_{t,S}\|\\
        & = \|\left(\textbf{I} - \eta  (\frac{1}{B}\int_0^1 \sum_{i \in \miniB_t}\bH_i\left(\sw_{t} + x\left(\usw _{t}-\sw_t\right)\right)dx)\right)\left(\sw_{t} - \usw_{t}\right)\| + \eta  \|\delta_{t,S}\|\\
    \end{split}
\end{align*}

Then by using the Lemma \ref{lemma: sgd} and Lemma \ref{lemma: delta_t_zeta_bounds}, the formula above can be bounded as:
\begin{align*}
    \begin{split}
        & \leq (1-\eta \mu)\|\sw_{t} - \usw_{t}\|  + \eta 2c_2 \frac{\Delta B_t}{B}\\
    \end{split}
\end{align*}

Then by using Lemma \ref{lemma: delta_t_zeta_bounds} and using the formula above recursively, we get that with probability higher than $1-t\cdot\Psi_1$, $\|\sw_{t'} - \usw_{t'}\| \leq \frac{2c_2}{\mu}(\frac{r}{n} + \frac{1}{B^{1/4}})$ holds for all iterations $t'\leq t$, which finishes the proof.

\end{proof}

\subsubsection{Proof of Theorem \ref{bfi_sgd}}
\begin{proof}
For any $t' \leq t$, by subtracting $\sw_{t'}$ by $\isw_{t'}$ and taking the same argument as equation \eqref{eq: iw_bound0}-\eqref{eq: iw_bound_last} (except that $\w_{t'}$, $\iw_{t'}$, $\bH$, $\B$, $n$, $r$ are replaced with $\sw_{t'}$, $\isw_{t'}$, $\sbH$, $\sB$, $B$, $\Delta B_{t'}$), the following equality holds due to the bound on $\|\sbH_{t'} - \sB_{j_m}\|$:
\begin{align}\label{eq: isw_sw_deriv_1}
    \begin{split}
        & \|\isw_{t'+1} - \sw_{t'+1}\|\\
        & \leq (1 - \eta \mu + \eta \frac{B}{B-\Delta B_t}\xi_{j_1,j_m}^S)\|\isw_t - \sw_t\| + \frac{2\Delta B_t\eta c_2}{B-\Delta B_t}.
    \end{split}
\end{align}

Since 
 $\frac{\Delta B_t}{B} \leq \frac{r}{n} + \frac{1}{B^{1/4}}$ for all iterations between 0 and $t$, 
  the following two inequalities hold:
\begin{equation}\label{eq: delta_B_ineq1}
    \frac{2\Delta B_t\eta c_2}{B-\Delta B_t} = \frac{2\eta c_2}{\frac{B}{\Delta B_t}-1} \leq \frac{2\eta c_2}{\frac{1}{\frac{r}{n} + \frac{1}{B^{1/4}}}-1} = \frac{2\eta c_2}{1-\frac{r}{n} - \frac{1}{B^{1/4}}}(\frac{r}{n} + \frac{1}{B^{1/4}}),
\end{equation}
\begin{equation}\label{eq: delta_B_ineq2}
\frac{B}{B-\Delta B_t} = \frac{1}{1-\frac{\Delta B_t}{B}} \leq \frac{1}{1-(\frac{r}{n} + \frac{1}{B^{1/4}})}.
\end{equation}

Moreover, since Formula \eqref{eq: xi_bound} holds and $\xi_{j_0, j_0 + (m-1)T_0}^S + A\times \sgdb \leq \frac{\mu}{2}$, then:
\begin{align*}
\begin{split}
&\xi_{j_1,j_m}^S = \xi_{j_0+xT_0,j_0+(x+m-1)T_0}^S\leq (1-\eta\mu)^{xT_0}Ad_{j_0, j_0+mT_0-1}
\\
& +A_1\frac{r}{n} + A_2\frac{1}{B^{1/4}}+A\sgdb\\
& \leq Ad_{j_0, j_0+mT_0-1} +A_1\frac{r}{n} + A_2\frac{1}{B^{1/4}} + A\sgdb \\
& = \xi_{j_0, j_0 + mT_0 - 1} + A\sgdb\leq \frac{\mu}{2}.
\end{split}
\end{align*}

%
%
%
%
Then the Formula \eqref{eq: isw_sw_deriv_1} can be bounded as:
\begin{align*}
    \begin{split}
        & \|\isw_{t'+1} - \sw_{t'+1}\|\\
        & \leq (1 - \eta \mu + \eta \frac{B}{B-\Delta B_t}(\xi_{j_0,j_0+(m-1)T_0}^S + A \times \sgdb)\|\isw_{t'} - \sw_{t'}\|+ \frac{2\Delta B_t\eta c_2}{B-\Delta B_t}\\
        & \leq (1 - \eta \mu + \eta \frac{\xi_{j_0,j_0+(m-1)T_0}^S +A \times \sgdb }{1-(\frac{r}{n} + \frac{1}{B^{1/4}})})\|\isw_{t'} - \sw_{t'}\|+ \frac{2\eta c_2}{1-\frac{r}{n} - \frac{1}{B^{1/4}}}(\frac{r}{n} + \frac{1}{B^{1/4}}),
    \end{split}
\end{align*}

which uses equation \eqref{eq: delta_B_ineq1} and \eqref{eq: delta_B_ineq2}. Then applying the formula recursively from iteration $t$ to 0, we can get:
\begin{align*}
    \begin{split}
        & \|\isw_{t'+1} - \sw_{t'+1}\|\\
        & \leq \frac{1}{\eta(\mu-\frac{\xi_{j_0,j_0+(m-1)T_0}^S + 2c_2\left(\frac{(\log(p+1))^2}{B}\right)^{1/4}}{1-(\frac{r}{n} + \frac{1}{B^{1/4}})})}\frac{2\eta c_2}{1-\frac{r}{n} - \frac{1}{B^{1/4}}}(\frac{r}{n} + \frac{1}{B^{1/4}}).
    \end{split}
\end{align*}

Then since $\xi_{j_0,j_0+(m-1)T_0}^S \leq \frac{\mu}{2}$, the formula above can be further bounded as:
\begin{align*}
    \begin{split}
        & = \frac{2c_2}{(1-\frac{r}{n} - \frac{1}{B^{1/4}})\mu-\xi_{j_0,j_0+(m-1)T_0}^S-A \times \sgdb}(\frac{r}{n} + \frac{1}{B^{1/4}})\\
        & \leq \frac{2c_2}{(\frac{1}{2}-\frac{r}{n} - \frac{1}{B^{1/4}})\mu}(\frac{r}{n} + \frac{1}{B^{1/4}}).
    \end{split}
\end{align*}

\end{proof}

\subsubsection{Proof of Theorem \ref{bfi2_sgd}}
The proof is the same as the proof of Theorem \ref{bfi2} except that $\w$, $\iw$, $n$, $r$, $\xi_{j_1,j_m}$, $\bH$, $\B$ need to be replaced by $\sw$, $\isw$, $B$, $r$, $\xi_{j_1,j_m}^S$, $\sbH$, $\sB$ and the main theorems that the proof depends on will be replaced by Theorem \ref{bfi_sgd} and Corollary \ref{corollary: approx_hessian_real_hessian_bound_sgd}. But we need some careful explanations for the probability, which is shown as:
\begin{proof}
We define the following event at a given iteration $k$:
\begin{align*}
&\Omega_0(k) = \{\|\sw_{k} - \usw_{k}\| \leq M_1(\frac{r}{n} + \frac{1}{B^{1/4}})\},\\
&\Omega_1(k) = \{\|\sw_{k} - \isw_{k}\| \leq M_1^S(\frac{r}{n} + \frac{1}{B^{1/4}})\},\\
&\Omega_2(k) = \{\|\sbH_{k-1} - \sB_{j_m}\| \leq \xi_{j_1,j_m}^S\}\,\,\,\, (j_m \leq k-1 \leq j_m +T_0 - 1),\\
&\Omega_3(k) = \{\|\left(\frac{1}{B}\sum_{i \in \miniB_{k-1}} \bH_i(\sw_{k-1})\right) - \bH(\sw_{k-1})\| \leq \rhbound\},\\
&\Omega_4(k) = \{\xi_{j_0 + xT_0, j_0 + (x+m-1)T_0}^S \leq (1-\eta\mu)^{j_0 + xT_0}A d_{0, mT_0-1}\\
& +A_1\frac{r}{n} + A_2\frac{1}{B^{1/4}}+A\sgdb\}\ \text{where}\ j_0 + (x+m-1)T_0 \leq k-1 \leq j_0 + (x+m)T_0-1,\\
&\Omega_5(k) = \{\frac{\Delta B_{k-1}}{B} \leq \frac{r}{n} + \frac{1}{B^{1/4}}\}.
\end{align*}

For all $t$, according to Corollary \ref{corollary: approx_hessian_real_hessian_bound_sgd}, the following equation holds:
$$\Pr(\bigcap_{k=1}^t\Omega_2(k)|\bigcap_{k=1}^{t-1}\Omega_1(k),\bigcap_{k=1}^t\Omega_3(k)) = 1.$$ 

in which the co-occurrence of multiple events is denoted by $\bigcap$ or ``,''. So this formula means that the probability that $\Omega_2(k)$ is true for all $k \leq t$ given that the events $\Omega_1(k)$ and $\Omega_3(k)$ are true at the same time for all $k \leq t$ is 1.

Similarly, according to Theorem \ref{bfi_sgd}, $\Pr(\bigcap_{k=1}^{t}\Omega_1(k)\Bigr|\bigcap_{k=1}^{t}\Omega_2(k), \bigcap_{k=1}^{t}\Omega_4(k), \bigcap_{k=1}^{t}\Omega_5(k)) =1$. Then we know that:
$$\Pr(\bigcap_{k=1}^{t}\Omega_1(k)\Bigr|\bigcap_{k=1}^{t}\Omega_2(k), \bigcap_{k=1}^{t}\Omega_4(k),\bigcap_{k=1}^{t}\Omega_5(k))\cdot \Pr(\bigcap_{k=1}^t\Omega_2(k)|\bigcap_{k=1}^{t-1}\Omega_1(k),\bigcap_{k=1}^t\Omega_3(k))$$
$$=\Pr(\bigcap_{k=1}^{t}\Omega_1(k), \bigcap_{k=1}^{t}\Omega_2(k) \Bigr|\bigcap_{k=1}^{t}\Omega_4(k),\bigcap_{k=1}^{t}\Omega_5(k), \bigcap_{k=1}^{t-1}\Omega_1(k),\bigcap_{k=1}^t\Omega_3(k))=1,$$ which can be multiplied by $$\Pr(\bigcap_{k=1}^{t-1}\Omega_1(k)\Bigr|\bigcap_{k=1}^{t-1}\Omega_2(k), \bigcap_{k=1}^{t-1}\Omega_4(k), \bigcap_{k=1}^{t-1}\Omega_5(k)).$$ The result is then multiplied by $$\Pr(\bigcap_{k=1}^{t-1}\Omega_2(k)|\bigcap_{k=1}^{t-2}\Omega_1(k),\bigcap_{k=1}^{t-1}\Omega_3(k)) = 1.$$ Then the following equality holds:
$$\Pr(\bigcap_{k=1}^{t}\Omega_1(k), \bigcap_{k=1}^{t}\Omega_2(k) \Bigr|\bigcap_{k=1}^{t}\Omega_4(k),\bigcap_{k=1}^{t}\Omega_5(k), \bigcap_{k=1}^{t-2}\Omega_1(k),\bigcap_{k=1}^t\Omega_3(k))=1$$
which uses the fact that $\bigcap_{k=1}^t\Omega_y(k) \bigcap \bigcap_{k=1}^{t-1}\Omega_y(k) = \bigcap_{k=1}^t\Omega_y(k)$ ($y=1,2,3,4,5$). So by repeating this until the iteration $j_0$, then the following equality holds:
\begin{align}\label{eq: isw_sw_deriv_2}
\begin{split}
\Pr(\bigcap_{k=1}^{t}\Omega_1(k), \bigcap_{k=1}^{t}\Omega_2(k) \Bigr|\bigcap_{k=1}^{t}\Omega_4(k),\bigcap_{k=1}^{t}\Omega_5(k), \bigcap_{k=1}^{j_0}\Omega_1(k),\bigcap_{k=1}^t\Omega_3(k))=1
\end{split}
\end{align}

When $t \leq j_0$, we know that $\isw_t = \usw_t$ and $M_1^S \geq M_1$, which means that if $\Omega_0(k)$ holds, then $\Omega_1(k)$ holds when $\isw_t = \usw_t$, and thus $$\Pr(\bigcap_{k=1}^{j_0}\Omega_1(k)|\bigcap_{k=1}^{j_0}\Omega_0(k)) = 1.$$ Then according to Theorem \ref{wu_sgd}, we know that:
$$\Pr(\bigcap_{k=1}^{j_0} \Omega_0(k)\Bigr|\bigcap_{k=1}^{j_0} \Omega_5(k)) = 1.$$

By multiplying the above two formulas, we get:
\begin{align*}
\begin{split}
&\Pr(\bigcap_{k=1}^{j_0}\Omega_1(k)|\bigcap_{k=1}^{j_0}\Omega_0(k)) \cdot \Pr(\bigcap_{k=1}^{j_0} \Omega_0(k)\Bigr|\bigcap_{k=1}^{j_0} \Omega_5(k))\\
& = \Pr(\bigcap_{k=1}^{j_0}\Omega_1(k), \bigcap_{k=1}^{j_0}\Omega_0(k) \Bigr|\bigcap_{k=1}^{j_0} \Omega_5(k)) = 1
\end{split}
\end{align*}

Note that since the probability of two joint events is smaller than that of either of the events, the following inequality holds:
$$\Pr(\bigcap_{k=1}^{j_0}\Omega_1(k), \bigcap_{k=1}^{j_0}\Omega_0(k) \Bigr|\bigcap_{k=1}^{j_0} \Omega_5(k)) \leq \Pr(\bigcap_{k=1}^{j_0}\Omega_1(k)\Bigr|\bigcap_{k=1}^{j_0} \Omega_5(k)) \leq 1.$$ So we know that:
$$\Pr(\bigcap_{k=1}^{j_0}\Omega_1(k) \Bigr|\bigcap_{k=1}^{j_0} \Omega_5(k)) = 1.$$

which can be multiplied by Formula \eqref{eq: isw_sw_deriv_2} and thus the following equality holds:
\begin{align}\label{eq: condition_prob_deriv_1}
\Pr(\bigcap_{k=1}^{t}\Omega_1(k), \bigcap_{k=1}^{t}\Omega_2(k) \Bigr|\bigcap_{k=1}^{t}\Omega_4(k),\bigcap_{k=1}^{t}\Omega_5(k), \bigcap_{k=1}^t\Omega_3(k))=1
\end{align}

Then we can compute the probability of the negation of the joint event $(\bigcap_{k=1}^{t+1}\Omega_1(k), \bigcap_{k=1}^{t}\Omega_2(k))$:
\begin{align*}
\begin{split}
& \Pr(\overline{\bigcap_{k=1}^{t}\Omega_1(k), \bigcap_{k=1}^{t}\Omega_2(k)})\\
& = \Pr(\overline{\bigcap_{k=1}^{t}\Omega_1(k), \bigcap_{k=1}^{t}\Omega_2(k)}\Bigr|\bigcap_{k=1}^{t}\Omega_4(k),\bigcap_{k=1}^{t}\Omega_5(k), \bigcap_{k=1}^t\Omega_3(k))\cdot \Pr(\bigcap_{k=1}^{t}\Omega_4(k),\bigcap_{k=1}^{t}\Omega_5(k), \bigcap_{k=1}^t\Omega_3(k))\\
&+\Pr(\overline{\bigcap_{k=1}^{t}\Omega_1(k), \bigcap_{k=1}^{t}\Omega_2(k)}\Bigr|\overline{\bigcap_{k=1}^{t}\Omega_4(k),\bigcap_{k=1}^{t}\Omega_5(k), \bigcap_{k=1}^t\Omega_3(k)})\cdot \Pr(\overline{\bigcap_{k=1}^{t}\Omega_4(k),\bigcap_{k=1}^{t}\Omega_5(k), \bigcap_{k=1}^t\Omega_3(k)})\\
&\leq \Pr(\overline{\bigcap_{k=1}^{t}\Omega_1(k), \bigcap_{k=1}^{t}\Omega_2(k)}\Bigr|\bigcap_{k=1}^{t}\Omega_4(k),\bigcap_{k=1}^{t}\Omega_5(k), \bigcap_{k=1}^t\Omega_3(k)) + \Pr(\overline{\bigcap_{k=1}^{t}\Omega_4(k),\bigcap_{k=1}^{t}\Omega_5(k), \bigcap_{k=1}^t\Omega_3(k)})\\
&=\Pr(\overline{\bigcap_{k=1}^{t}\Omega_4(k),\bigcap_{k=1}^{t}\Omega_5(k), \bigcap_{k=1}^t\Omega_3(k)}).
\end{split}
\end{align*}

The last two steps use the fact that $$\Pr(\overline{\bigcap_{k=1}^t\Omega_1(k),\bigcap_{k=1}^t\Omega_2(k)}\Bigr|\bigcap_{k=1}^{t}\Omega_4(k),\bigcap_{k=1}^{t}\Omega_5(k), \bigcap_{k=1}^t\Omega_3(k)) = 0$$ and $$\Pr(\overline{\bigcap_{k=1}^{t}\Omega_1(k), \bigcap_{k=1}^{t}\Omega_2(k)}\Bigr|\overline{\bigcap_{k=1}^{t}\Omega_4(k),\bigcap_{k=1}^{t}\Omega_5(k), \bigcap_{k=1}^t\Omega_3(k)}) \leq 1.$$

By further using the property of the probability of the union of multiply events, the formula above is bounded as:
$$\Pr(\overline{\bigcap_{k=1}^{t}\Omega_4(k),\bigcap_{k=1}^{t}\Omega_5(k), \bigcap_{k=1}^t\Omega_3(k)}) \leq \Pr(\overline{\bigcap_{k=1}^{t}\Omega_4(k)} \bigcup \overline{\bigcap_{k=1}^{t}\Omega_5(k)}\bigcup \overline{\bigcap_{k=1}^t\Omega_3(k)})$$
$$\leq \Pr(\overline{\bigcap_{k=1}^{t}\Omega_4(k)}) +\Pr(\overline{\bigcap_{k=1}^{t}\Omega_5(k)}) + \Pr(\overline{\bigcap_{k=1}^t\Omega_3(k)}).$$

Then by using Theorem \ref{theorem: stochastic_hessian_bound}, Formula \eqref{eq: xi_bound}, Lemma \ref{lemma: delta_t_zeta_bounds} and taking the union between iteration 0 and $t$, we get:
$$\Pr(\overline{\bigcap_{k=1}^t\Omega_3(k)}) \leq t\times \Psi_2,$$

$$\Pr(\overline{\bigcap_{k=1}^t\Omega_4(k)}) \leq t\Psi_3,$$

$$\Pr(\overline{\bigcap_{k=1}^t\Omega_5(k)}) \leq t\times \Psi_1.$$

Then we can know that: $$\Pr(\bigcap_{k=1}^{t}\Omega_1(k), \bigcap_{k=1}^{t}\Omega_2(k)) \geq 1-t(\Psi_2 + \Psi_3 + \Psi_1)$$
and thus
$$\Pr(\bigcap_{k=1}^{t}\Omega_1(k)) \geq \Pr(\bigcap_{k=1}^{t}\Omega_1(k), \bigcap_{k=1}^{t}\Omega_2(k)) \geq 1-t(\Psi_2 + \Psi_3 + \Psi_1).$$

This finishes the proof.

Similarly, from Formula \eqref{eq: condition_prob_deriv_1}, we know that for all $T$ iterations:
\begin{align}\label{eq: condition_prob_deriv_2}
\Pr(\bigcap_{k=1}^{t}\Omega_1(k), \bigcap_{k=1}^{t}\Omega_2(k),\bigcap_{k=1}^{t}\Omega_4(k),\bigcap_{k=1}^{t}\Omega_5(k)\Bigr|\bigcap_{k=1}^{t}\Omega_4(k),\bigcap_{k=1}^{t}\Omega_5(k), \bigcap_{k=1}^t\Omega_3(k))=1.
\end{align}

Through the same argument, we know that:
$$\Pr(\bigcap_{k=1}^{T}\Omega_1(k), \bigcap_{k=1}^{T}\Omega_2(k),\bigcap_{k=1}^{T}\Omega_4(k),\bigcap_{k=1}^{T}\Omega_5(k))\geq 1-T(\Psi_2 +\Psi_3 +\Psi_1).$$

\end{proof}

\subsubsection{Proof of Theorem \ref{main10_sgd}}
The proof is the same as the proof of Theorem \ref{bfi2} except that $\w$, $\iw$, $n$, $r$, $\xi_{j_1,j_m}$, $\bH$, $\B$ need to be replaced by $\sw$, $\isw$, $B$, $r$, $\xi_{j_1,j_m}^S$, $\sbH$, $\sB$ and the main theorems that the proof depends on will be replaced by Theorem \ref{bfi_sgd} and Corollary \ref{corollary: approx_hessian_real_hessian_bound_sgd}. We will show some key steps below.

First of all, according to the proofs of Theorem \ref{bfi2_sgd}, we know that the following inequalities hold with probability higher than $1-T(\Psi_2 + \Psi_3 + \Psi_1)$:
\begin{align*}
\begin{split}
\|\sw_{k} - \isw_{k}\| \leq \frac{1}{\frac{1}{2}-\frac{r}{n}-\frac{1}{B^{1/4}}}M_1(\frac{r}{n} + \frac{1}{B^{1/4}});
\end{split}\\
\begin{split}
\|\sbH_k - \sB_{j_m}\| \leq \xi_{j_1,j_m}^S;
\end{split}\\
\begin{split}
& \xi_{j_0 + xT_0, j_0 + (x+m-1)T_0}^S \leq (1-\eta\mu)^{j_0 + xT_0}A d_{0, mT_0-1} + A_1\frac{r}{n} + A_2\frac{1}{B^{1/4}} + A \sgdb\\
& \leq \xi_{j_0, j_0 +(m-1)T_0} + A \sgdb;
\end{split}\\
\begin{split}
\frac{\Delta B_{k}}{B} \leq \frac{r}{n} + \frac{1}{B^{1/4}}.
\end{split}
\end{align*}

Then by subtracting $\isw_t$ by $\usw_t$ and following the arguments from Formula \eqref{eq: w_U_w_IU_diff} to \eqref{eq: w_U_w_IU_diff_2}, the following inequality holds for $\|\isw_t - \usw_t\|$ with probability higher than $1-T\times (\Psi_2+ \Psi_1 + \Psi_3)$:
\begin{align*}
    \begin{split}
        & \|\isw_t - \usw_t\|\\
        & \leq (1-\eta\mu + \frac{B\eta}{B-\Delta B_t}\xi_{j_1,j_m}^S)\|\isw_t - \usw_t\|\\
        & + \frac{\eta c_0}{2}\|\usw_{t} - \isw_{t}\|\|\usw_{t}-\sw_t\|+ \frac{B\eta}{B-\Delta B_t} \xi_{j_1,j_m}^S \|\usw_t-\sw_t\|\\
        & \leq \left(1-\eta\mu + \frac{B\eta}{B-\Delta B_t}\xi_{j_1,j_m}^S+\frac{c_0M_1\eta}{2}(\frac{r}{n} + \frac{1}{B^{1/4}})\right)\|\isw_t - \usw_t\| + \frac{B\eta}{B-\Delta B_t} \xi_{j_1,j_m}^S M_1(\frac{r}{n} + \frac{1}{B^{1/4}}).
    \end{split}
\end{align*}

By using the fact that $\frac{\Delta B_t}{B} \leq \frac{r}{n} + \frac{1}{B^{1/4}}$ and $\xi_{j_1, j_m} \leq \xi_{j_0, j_0 +(m-1)T_0} + A \times \sgdb$, the formula above can be bounded as:
\begin{align*}
    \begin{split}
        & \|\isw_t - \usw_t\|\\
        & \leq \left(1-\eta\mu + \frac{B\eta}{B-\Delta B_t}\xi_{j_1,j_m}^S+\frac{c_0M_1\eta}{2}(\frac{r}{n} + \frac{1}{B^{1/4}})\right)\|\isw_t - \usw_t\| + \frac{B\eta}{B-\Delta B_t} \xi_{j_1,j_m}^S M_1(\frac{r}{n} + \frac{1}{B^{1/4}})\\
        & \leq [1-\eta\mu + \frac{\eta}{1-\frac{r}{n}-\frac{1}{B^{1/4}}}(\xi_{j_0,j_0+(m-1)T_0}^S + A \times \sgdb)\\
        & +\frac{c_0M_1\eta}{2}(\frac{r}{n} + \frac{1}{B^{1/4}})]\|\isw_t - \usw_t\| + \frac{\eta}{1-\frac{r}{n}-\frac{1}{B^{1/4}}} \xi_{j_0+xT_0,j_0+(x+m-1)T_0}^S M_1(\frac{r}{n} + \frac{1}{B^{1/4}}).
    \end{split}
\end{align*}

Since $\xi_{j_0,j_0+(m-1)T_0}^S + A \times \sgdb \leq \frac{\mu}{2}$ and $B$ is a large mini-batch size, then $$1 - (\eta\mu - \frac{\eta}{1-\frac{r}{n}-\frac{1}{B^{1/4}}}(\xi_{j_0,j_0+(m-1)T_0}^S + A \times  \sgdb)-\frac{c_0M_1\eta}{2}(\frac{r}{n} + \frac{1}{B^{1/4}})) < 1.$$

Then after explicitly using the definition of $\xi_{j_1,j_m}^S$ and following the argument of equation \eqref{eq: w_U_w_IU_diff_3} to \eqref{eq: w_U_w_IU_diff_last}, we get:
\begin{align}\label{eq: wu_w_i_sgd_diff_last}
    \begin{split}
        & \|\iw_{j_0 + (y+m)T_0} - \uw_{j_0 + (y+m)T_0}\|\\ & \leq (1-\eta C)^{y T_0}\|\iw_{j_0 + mT_0} - \uw_{j_0 + mT_0}\|\\
        & + \frac{M_1(\frac{r}{n} + \frac{1}{B^{1/4}})}{C(1-\frac{r}{n} - \frac{1}{B^{1/4}})}(1-\eta C)^{yT_0}(1-\eta\mu)^{j_0}d_{0,mT_0 - 1}\frac{1}{1-(\frac{1-\eta\mu}{1-\eta C})^{T_0}}\\
        & +\frac{1}{1-(1-\eta C)^{T_0}}
\frac{M_1(\frac{r}{n} + \frac{1}{B^{1/4}})}{C(1-\frac{r}{n} - \frac{1}{B^{1/4}})}(A_1\frac{r}{n} + A_2\frac{1}{B^{1/4}} + A\sgdb)
    \end{split}
\end{align}

when $t \rightarrow \infty$ and thus $y \rightarrow \infty$, $(1-\eta C)^{y T_0} \rightarrow 0$. Also with large mini-batch value $B$, $A_1\frac{r}{n} + A_2\frac{1}{B^{1/4}} + A\sgdb$ is a value of the same order as $\frac{r}{n} + \frac{1}{B^{1/4}}$. Thus $$\|\iw_{j_0 + (y+m)T_0} - \uw_{j_0 + (y+m)T_0}\| = o(\frac{r}{n} + \frac{1}{B^{1/4}})$$
and
$$\|\usw_t - \isw_t\| \leq o(\frac{r}{n} + \frac{1}{B^{1/4}}).$$

%% file: Sections/other_alg.tex

\section{Supplementary algorithm details}

In Section \ref{sec: alg}, we only provided the details of \Increm\ for deterministic gradient descent for the strongly convex and smooth objective functions in batch deletion/addition scenarios. In this section, we will provide more details on how to extend \Increm\ to handle stochastic gradient descent, online deletion/addition scenarios and non-strongly convex, non-smooth objective functions.

\subsection{Extension of \Increm\ for stochastic gradient descent}
By using the notations from equations \eqref{eq: sw_update}-\eqref{eq: isw_update}, we need to approximately or explicitly compute 
$\sgrad$, 
i.e. the average gradient for a mini-batch in the SGD version of \Increm, instead of 
$\nabla F$
, which is the average gradient for all samples. So by replacing $\w_t$, $\uw_t$, $\iw_t$, $\nabla F$, $\B$ and $\bH$ with $\sw_t$, $\usw_t$, $\isw_t$, $\sgrad$, $\sB$ and $\sbH$ in Algorithm \ref{alg: update_algorithm}, we get the SGD version of \Increm.

\subsection{Extension of \Increm\ for online deletion/addition}\label{sec: online_del_add}

In the online deletion/addition scenario, whenever the model parameters are updated after the deletion or addition of one sample, the history information should be also updated to reflect the changes. By assuming that only one sample is deleted or added each time, the online deletion/addition version of \Increm\ is provided in Algorithm \ref{alg: update_algorithm_online} and the differences relative to Algorithm \ref{alg: update_algorithm} are highlighted.

Since the history information needs to be updated every time when new deletion or addition requests arrive, we need to do some more analysis on the error bound, which is still pretty close to the analysis in Section \ref{sec: math}. 

In what follows, the analysis will be conducted on gradient descent with online deletion. Other similar scenarios, e.g. stochastic gradient descent with online addition, will be left as the future work.

\begin{algorithm}[t] 
\small
\SetKwInOut{Input}{Input}
\SetKwInOut{Output}{Output}
\Input{The full training set $\left(\textbf{X}, \textbf{Y}\right)$, model parameters cached during the training phase for the full training samples $\{\w_{0}, \w_{1}, \dots, \w_{t}\}$ and corresponding gradients $\{\nabla F\left(\w_{0}\right), \nabla F\left(\w_{1}\right), \dots, \nabla F\left(\w_{t}\right)\}$, the index of the removed training sample or the added training sample $i_r$, period $T_0$, total iteration number $T$, history size $m$, warmup iteration number $j_0$, learning rate $\eta$}
\Output{Updated model parameter $\iw_{t}$}
Initialize $\iw_{0} \leftarrow \w_{0}$

Initialize an array $\Delta G = \left[\right]$

Initialize an array $\Delta W = \left[\right]$

\For{$t=0;t<T; t++$}{

\eIf{$[((t - j_0) \mod T_0) == 0]$ or $t \leq j_0$}
{
    compute $\nabla F\left(\iw_{t}\right)$ exactly
    
    compute $\nabla F\left(\iw_{t}\right) - \nabla F\left(\w_{t}\right)$ based on the cached gradient $\nabla F\left(\w_{t}\right)$
    
    set $\Delta G\left[k\right] = \nabla F\left(\iw_{t}\right) - \nabla F\left(\w_{t}\right)$
    
    set $\Delta W\left[k\right] = \iw_{t} - \w_{t}$, based on the cached parameters $\w_{t}$
    
    $k\leftarrow k+1$
    
    compute $\iw_{t+1}$ by using exact GD update (equation \eqref{eq: update_rule_naive})
    
    \hl{$\w_t \leftarrow \iw_t$}
    
    \hl{$\nabla F(\w_t) \leftarrow \nabla F(\iw_t)$}
}
{
    Pass $\Delta W\left[-m:\right]$, $\Delta G\left[-m:\right]$, the last $m$ elements in $\Delta W$ and $\Delta G$, which are from the $j_1^{th}, j_2^{th},\dots, j_m^{th}$ iterations where $j_1 < j_2< \dots < j_m$ depend on $t$, $\textbf{v} = \iw_{t} - \w_{t}$, and the history size $m$, to the L-BFGFS Algorithm (See Supplement) to get the approximation of $\bH(\w_{t})\textbf{v}$, i.e., $\B_{j_m}\textbf{v}$
    
    Approximate $\nabla F\left(\iw_{t}\right) = \nabla F\left(\w_{t}\right) + \B_{j_m}\left(\iw_{t} - \w_{t}\right)$
    
    Compute $\iw_{t+1}$ by using the "leave-$1$-out" gradient formula, based on the approximated $\nabla F(\iw_{t})$ 
    
    \hl{$\w_t \leftarrow \iw_t$}
    
    \hl{$\nabla F(\w_t) \leftarrow \frac{\eta}{n-1}[n(\B_{j_m}(\iw_t - \w_t) + \nabla F(\w_{t})) - \nabla F_{i_r}(\w_t)]$}
}
}

\Return $\iw_{t}$
\caption{DeltaGrad (online deletion/addition)}
\label{alg: update_algorithm_online}
\end{algorithm}

\subsubsection{Convergence rate analysis for online gradient descent version of \Increm}

\textbf{Additional notes on setup, preliminaries}

Let us still denote the model parameters for the original dataset at the $t_{th}$ iteration by $\w_t$. During the model update phase for the $k_{th}$ deletion request at the $t_{th}$ iteration, the model parameters updated by \Basel\ and \Increm\ are denoted by $\uw_t(k)$ and $\iw_t(k)$ respectively where $\uw_t(0) = \iw_t(0) = \w_t$. We also assume that the total number of removed samples in all deletion requests, $r$, is still far smaller than the total number of samples, $n$. 

Also suppose that the indices of the removed samples are $\{i_1,i_2,\dots, i_{r}\}$, which are removed at the $1_{st}$, $2_{nd}$, $3_{rd},\dots,$, ${r}_{th}$ deletion request. This also means that the cumulative number of samples up to the $k_{th}$ deletion request ($k\leq r$) is $n-k$ for all $1 \leq k \leq r$ and thus the objective function at the $k_{th}$ iteration will be:
\begin{align*}
    F^k(\w) = \frac{1}{n-k}\sum_{i \not\in R_{k}} F_i(\w).
\end{align*}

where $R_k = \{i_1,i_2,\dots,i_k\}$. Plus, at the $k_{th}$
deletion request, we denote by $\bH^{k}_t$ the average Hessian matrix of  $F^k(\w)$ evaluated between $\iw_t(k+1)$ and $\iw_t(k)$:
$$\bH^{k}_t = \frac{1}{n-k}\sum_{i\not\in R_{k}} \int_{0}^1\bH_i(\iw_t(k) + x(\iw_t(k+1)-\iw_t(k)))dx$$

Specifically, $$\bH^{0}_t = \frac{1}{n}\sum_{i=1}^n \int_0^1\bH_i(\iw_t(0) + x(\iw_t(1)-\iw_t(0)))dx.$$

Also the model parameters and the approximate gradients evaluated by \Increm\ at the ${r-1}_{st}$ deletion request are used  at the $r_{th}$ request, and are denoted by: $$\{\iw_{0}(r-1), \iw_{1}(r-1), \dots, \iw_{t}(r-1)\}$$ and $$\{\ag \left(\iw_{0}(r-1)\right), \ag\left(\iw_{1}(r-1)\right), \dots, \ag\left(\iw_{t}(r-1)\right)\}.$$ 

Note that $\ag(\iw_t(k))$ ($k\leq r$) is not necessarily equal to $\nabla F$ due to the approximation brought by \Increm. But due to the periodicity of \Increm, at iteration $0,1,\dots,j_0$ and iteration $j_0 + xT_0$ ($x=1,2,\dots,$), the gradients are explicitly evaluated, i.e.:
$$\ag(\iw_t(k)) = \frac{1}{n-k}\sum_{i\not\in R_k} \nabla F_i(\iw_t(k))$$
for $t = 0,1,\dots, j_0$ or $t = j_0 + xT_0$ ($x \geq 1$) and all $k \leq r$. 

Also, due to the periodicity, the sequence $[\Delta g_{j_0}, \Delta g_{j_1},\dots, \Delta g_{j_{m-1}}]$ used in approximating the Hessian matrix always uses the exact gradient information, which means that:
$$\Delta g_{j_q} = \frac{1}{n-k}[\sum_{i\not\in R_k}\nabla F_i(\iw_{j_q}(k)) - \sum_{i \not\in R_k}\nabla F_i(\iw_{j_q}(k-1))]$$

where $q=1,2,\dots,m-1$. So Lemma \ref{assp: B_K_product_bound} on the bound on the eigenvalues of $B_{j_q}$ holds for all $q$ and $k=1,2,\dots,r$.

But for the iterations where the gradients are not explicitly evaluated, the calculation of $\ag(\iw_t(k))$ depends on the approximated Hessian matrix $\B_{j_m}^{k-1}$ and the approximated gradients calculated at the $t_{th}$ iteration at the $k-1$-st deletion request. So the update rule for $\ag(\iw_t(k))$ is:
\begin{align}\label{eq: g_w_k_deriv}
    \begin{split}
        & \ag(\iw_t(k)) = \frac{1}{n-k}\{(n-k+1)[\B_{j_m}^{k-1}(\iw_t(k) - \iw_t(k-1))\\
        & + \ag(\iw_t(k-1))] - \nabla F_{i_k}(\iw_t(k))\}.
    \end{split}
\end{align}

Here the product $\B_{j_m}^{k-1}\cdot (\iw_t(k) - \iw_t(k-1))$ approximates $$\frac{1}{n-k+1}\sum_{i\not\in R_{k-1}}\nabla F_i(\iw_t(k)) - \nabla F_i(\iw_t(k-1))$$
and $\ag(\iw_t(k-1))$ approximates $\frac{1}{n-k+1}\sum_{i\not\in R_{k-1}}\nabla F_i(\iw_t(k-1))$.

Similarly, the online version of $\Dw$ (at the $k_{th}$ iteration) becomes:
$$\Dw_{j_q}(k) = \iw_{j_q}(k) - \iw_{j_q}(k-1)$$
where $q=1,2,\dots,m-1$.

Similarly, we use $d_{j_a,j_b}(k)$ 
to denote the value of the upper bound $d$ on the distance between the iterates
at the $k_{th}$ deletion request and use $\B_{j_m}^{k-1}$ to denote the approximated Hessian matrix in the $k_{th}$ deletion request, which approximated the Hessian matrix $\bH_t^{k-1}$.

So the update rule for $\iw_t(k)$ becomes:
\begin{align}\label{iw_u_online}
    \iw_{t+1}(k) = \begin{cases}
         &\iw_{t}(k)-\frac{\eta}{n-k} \sum_{i\not\in R_k}\nabla F_i(\iw_t(k)),\,\, [(t - j_0) \mod T_0 = 0]\ or\ t \leq j_0 \\[2pt]
         \begin{split}
        &\iw_{t}(k) - \frac{\eta}{n-k}\{(n-k+1)[\B_{j_m}^{k-1}(\iw_t(k) - \iw_t(k-1)) \\
        &+ \ag(\iw_{t}(k-1))] - \nabla F_{i_k}(\iw_t(k))\},\,\, else  .   
         \end{split}
        \end{cases}
\end{align}



\textbf{Proof preliminaries.}

On each deletion request, the \Basel\ model parameters are retrained from scratch on the remaining samples. This implies that Theorem \ref{wu} still holds, if we replace $\uw_t$, $\w_t$ and $r$ with $\uw_t(k)$, $\uw_t(k-1)$ and 1 respectively:

\begin{theorem}[Bound between iterates deleting one datapoint]\label{wu_multi}
$\|\uw_{t}(r) - \uw_{t}(r-1)\| \leq M_1 \frac{1}{n}$ where $M_1=\frac{2}{\mu }c_2$ is some positive constant that does not depend on $t$. Here $\mu$ is the strong convexity constant, and $c_2$ is the bound on the individual gradients. 
\end{theorem}

By induction, we have:
\begin{align}\label{eq: w_u_w_gap}
    \|\uw_{t}(r) - \w_{t}\| = \|\uw_{t}(r) - \uw_{t}(0)\| \leq M_1 \frac{r}{n}.
\end{align}

Then let us do some analysis on $d_{j_a,j_b}(k)$. We use the notation $\wikbounds$ for $\wikbound{r}$, where $M_1^{r}$ is a constant which does not depend on $k$.

\begin{lemma}\label{lemma: bound_d_j_multi}
If $\|\iw_t(k) - \iw_t(k-1)\| \leq \wikbound{k}$  for all $k \leq r$, then $d_{j_a,j_b}(r) \leq d_{j_a,j_b}(0) + 2r\cdot \wikbounds$ 
where $M_1$ is defined in Theorem \ref{wu_multi}.
\end{lemma}

\begin{proof}
Recall that $d_{j_a, j_b}(k) = \max(\|\iw_y(k) - \iw_z(k)\|)_{j_a < y < z < j_b}$.
Then for two arbitrary iterations $y, z$, let us bound $\|\iw_y(k) - \iw_z(k)\|$ as below:
\begin{align*}
    \begin{split}
        & \|\iw_y(k) - \iw_z(k)\|\\
        & = \|\iw_y(k) - \iw_z(k) + \iw_y(k-1) - \iw_z(k-1) + \iw_z(k-1) - \iw_y(k-1)\|\\
        & \leq \|\iw_y(k) - \iw_y(k-1)\| + \|\iw_z(k) - \iw_z(k-1)\| + \|\iw_z(k-1) - \iw_y(k-1)\|.
    \end{split}
\end{align*}

Then by using the bound on $\|\iw_t(k) - \iw_t(k-1)\|$, the above formula leads to:
\begin{align*}
    \begin{split}
        & \leq 2\cdot \wikbound{k}+ \|\iw_z(k-1) - \iw_y(k-1)\|.
    \end{split}
\end{align*}

By using  that $\wikbound{k} \leq \wikbound{r}$ and applying it recursively for $k=1,2,\dots,r$, we have:
\begin{align*}
    \begin{split}
        &\|\iw_y(r) - \iw_z(r)\|
         \leq 2r\cdot \wikbound{r} + \|\iw_z(0) - \iw_y(0)\|.
    \end{split}
\end{align*}

Then by using the definition of $d_{j_a,j_b}(k)$, the following inequality holds:
\begin{equation*}
    d_{j_a,j_b}(r) \leq d_{j_a,j_b + T_0 - 1}(0)
+ 2r\cdot \wikbound{r}.
\end{equation*}
Recalling the definition of $\wikbounds$, this is exactly the required result.
\end{proof}
We also mention that, since $\wikbound{k} \leq \wikbound{r}$, then $\|\iw_t(k) - \iw_t(k-1)\| \leq \wikbounds$ for any $k \leq r$.

\begin{theorem}\label{theorem: delta_model_para_bound_multi}
Suppose that at the $k_{th}$ deletion request, $\|\iw_{j_q}(k) - \iw_{j_q}(k-1)\| \leq \wikbounds $, where $q=1,2,\dots,m$ and $M_1 = \frac{2c_2}{\mu}$. Let $e = \frac{L(L+1) + K_2 L}{\mu  K_1}$ for the upper and lower bounds $K_1,K_2$ on the eigenvalues of the quasi-Hessian from Lemma \ref{assp: B_K_product_bound}, and for the Lipshitz constant $c_0$ of the Hessian. For $1 \leq z+1 \leq y \leq m$ we have:
$$\|\bH_{j_z}^{k-1} - \bH_{j_y}^{k-1}\| \leq c_0 d_{j_z, j_y}(k-1) + c_0 \wikbounds$$ and 
$$\|\Dg_{j_z} - \B_{j_y}^{k-1}\Dw_{j_z}\| \leq \left[(1+e)^{y-z-1} - 1\right]\cdot c_0 (d_{j_z,j_y} + \wikbounds) \cdot s_{j_1,j_m}(k-1)$$
 where $s_{j_1,j_m}(k-1) = \max\left(\|\Dw_{a}(k-1)\|\right)_{a=j_1,j_{2},\dots,j_m} = \max\left(\|\isw_{a}(k-1)-\isw_{a}(k-2)\|\right)_{a=j_1,j_{2},\dots,j_m}$. 
 Recall that $d$ is defined as the maximum gap between the steps of the algorithm for the iterations from $j_z$ to $j_y$:
\beq\label{d0}
d_{j_z,j_y}(k-1) = \max\left(\|\iw_{a}(k-1) - \iw_{b}(k-1)\|\right)_{j_z \leq a \leq b \leq j_y}.
\eeq
\end{theorem}
\begin{proof}
Let us bound the difference between the averaged Hessians $\|\bH_{j_z}^{k-1} - \bH_{j_y}^{k-1}\|$, where $1 \leq z < y \leq m$, using their definition, as well as using Assumption \ref{assp: hessian_continuous} on the Lipshitzness of the Hessian. First we can get the following equality:
\begin{align}\label{eq: hessian_diff_bound01}
    \begin{split}
      & \|\bH_{j_y}^{k-1} - \bH_{j_z}^{k-1}\|\\
      & = \|\int_0^1 [\bH(\iw_{j_y}(k-1) + x (\iw_{j_y}(k) - \iw_{j_y}(k-1)))]dx\\
      & - \int_0^1 [\bH(\iw_{j_z}(k-1) + x (\iw_{j_z}(k) - \iw_{j_z}(k-1)))]dx\|\\
      & = \| \int_0^1 [\bH(\iw_{j_y}(k-1) + x (\iw_{j_y}(k) - \iw_{j_y}(k-1)))\\
      & -\bH(\iw_{j_z}(k-1) + x (\iw_{j_z}(k) - \iw_{j_z}(k-1)))]dx \|
    \end{split}
\end{align}
Then we can bound this as:
\begin{align*}
    \begin{split}
      & \leq c_0 \int_0^1 \|\iw_{j_y}(k-1) + x (\iw_{j_y}(k) - \iw_{j_y}(k-1))\\
      & - [\iw_{j_z}(k-1) + x (\iw_{j_z}(k) - \iw_{j_z}(k-1))]\| dx \\
      & \leq c_0\| \iw_{j_y}(k-1) - \iw_{j_z}(k-1)\|\\
      & + \frac{c_0}{2}\|\iw_{j_y}(k) - \iw_{j_y}(k-1) - (\iw_{j_z}(k) - \iw_{j_z}(k-1))\| \\
      & \leq c_0 \|\iw_{j_y}(k-1) - \iw_{j_z}(k-1)\|\\
      & + \frac{c_0}{2}\|\iw_{j_z}(k)- \iw_{j_z}(k-1)\| + \frac{c_0}{2}\|\iw_{j_y}(k) - \iw_{j_y}(k-1)\|\\
      & \leq c_0 d_{j_y, j_z}(k-1) + c_0\wikbounds \leq c_0 d_{j_1, j_m +T_0 - 1}(k-1) + c_0\wikbounds.
    \end{split}
\end{align*}
On the last line, we have used the definition of $d_{j_z, j_y}$, and the assumption on the boundedness of $\|\iw_{j_z}(k) - \iw_{j_z}(k-1)\|$.

Then by following the rest of the proof of Theorem \ref{theorem: delta_model_para_bound}, we get: $$\|\Dg_{j_z} - \B_{j_y}\Dw_{j_z}\| \leq \left[(1+e)^{y-z-1} - 1\right]\cdot c_0 (d_{j_z,j_y}(k-1) + \wikbounds) \cdot s_{j_1,j_m}(k-1).$$

\end{proof}

Similarly, the online version of Corollary \ref{corollary: approx_hessian_real_hessian_bound} also holds by following the same derivation as the proof of Corollary \ref{corollary: approx_hessian_real_hessian_bound} (except that $r$, $\xi_{j_1, j_m}$ and $d_{j_1, j_m + T_0 -1}$ is replaced by 1, $\xi_{j_1, j_m}(k-1)$ and $d_{j_1, j_m + T_0 -1}(k-1)$ respectively), i.e.:

\begin{corollary}[Approximation accuracy of quasi-Hessian to mean Hessian (online deletion)]\label{corollary: approx_hessian_real_hessian_bound_multi}
Suppose that at the $k_{th}$ deletion request, $\|\iw_{j_s}(k) - \iw_{j_s}(k-1)\| \leq \wikbounds $ and $\|\iw_{t}(k) - \iw_{t}(k-1)\| \leq \wikbounds$ where $s = 1,2,\dots, m$. Then for $ j_m \le t \leq j_m + T_0 - 1,$
\beq\label{xi_multi_20}
\|\bH_{t}^{k-1} - \B_{j_m}^{k-1}\| \leq \xi_{j_1,j_m}(k-1):= A d_{j_1,j_m + T_0 - 1}(k-1)
+ A \wikbounds.
\eeq 
Recall that $A = \frac{c_0\sqrt{m}[(1+e)^{m}-1]}{c_1} + c_0$, where $c_0$ is the Lipschitz constant of the Hessian,  $c_1$ is the "strong independence" constant from Assumption \ref{assp: singular_lower_bound}, and $d_{j_1,j_m + T_0 - 1}(k-1)$ is the maximal gap between the iterates of the GD algorithm on the full data from $j_1$ to $j_m+T_0-1$ after the $k-1$-st deletion. 
\end{corollary}

Based on this, let us derive a bound on $\|\nabla F_i(\iw_t(r))\|$, $\|\ag(\iw_{t}(r)) - \frac{1}{n-r}\sum_{i\not\in R_r}\nabla F_i(\iw_{t}(r))\|$ and $\|\ag(\iw_{t}(r))\|$.


\begin{lemma}\label{lemma: gradient_multi_bound}
Suppose we are at an iteration $t$ such that $j_m \leq t \leq j_m +T_0 -1$. If the following inequality holds for all $k < r$:
$$\|\iw_{t}(k) - \iw_{t}(k-1)\| \leq \wikbounds,$$
then the following inequality holds for all $i=1,2,\dots,n$:
$$\|\nabla F_i(\iw_t(r-1))\| \leq \wikbounds Lr + c_2.$$
\end{lemma}
\begin{proof}


By adding and subtracting $\nabla F_i(\iw_t(r-2))$ inside $\|\nabla F_i(\iw_t(r-1))\|$, we get:
\begin{align*}
    \begin{split}
        & \|\nabla F_i(\iw_t(r-1))\|\\
        & = \|\nabla F_i(\iw_t(r-1)) - \nabla F_i(\iw_t(r-2)) + \nabla F_i(\iw_t(r-2))\|\\
        & \leq \|\nabla F_i(\iw_t(r-1)) - \nabla F_i(\iw_t(r-2))\| + \|\nabla F_i(\iw_t(r-2))\|\\
    \end{split}
\end{align*}

The last inequality uses the triangle inequality. Then by using the Cauchy mean value theorem, the upper bound on the eigenvalue of the Hessian matrix (i.e. Assumption \ref{assp: strong_convex smooth}) and the bound on $\|\iw_{t}(k) - \iw_{t}(k-1)\|$, the formula above is bounded as (recall $\bH_i$ is an integrated Hessian):
\begin{align*}
    & = \|\bH_i(\iw_t(r-1) + x(\iw_t(r-2)-\iw_t(r-1)))\cdot (\iw_t(r-1) - \iw_t(r-2))\|+ \|\nabla F_i(\iw_t(r-2))\|\\
    & \leq L \wikbounds + \|\nabla F_i(\iw_t(r-2))\|.
\end{align*}
By using this recursively, we get:
\begin{align*}
    & \leq \sum_{k=1}^{r-1} \wikbounds L + \|\nabla F_i(\iw_t(0))\| \leq \wikbounds Lr + c_2.
\end{align*}

\end{proof}

\begin{lemma}\label{lemma: approximate_vs_real_gradient_bound}
If at a given iteration $t$ such that $j_m \leq t \leq j_m +T_0 -1$, for all $k < r$, the following inequalities hold:
$$\|\iw_{t}(k) - \iw_{t}(k-1)\| \leq \wikbounds$$ and $$\xi_{j_1,j_m}(k-1) \leq \frac{\mu}{2},$$ then we have
\begin{align*}
    \begin{split}
        &\|\frac{1}{n-r+1}\sum_{i\not\in R_{r-1}}\nabla F_i(\iw_t(r-1))-\ag(\iw_t(r-1))\|
         \leq r\wikbounds \mu. 
    \end{split}
\end{align*}
\end{lemma}
\begin{proof}
First of all, $\frac{1}{n-r+1}\sum_{i\not\in R_{r-1}}\nabla F_i(\iw_t(r-1))$ can be rewritten as below by using the Cauchy mean-value theorem:
\begin{align*}
\begin{split}
\frac{1}{n-r+1}\sum_{i\not\in R_{r-1}}\nabla F_i(\iw_t(r-1))
&= \frac{1}{n-r+1}[\sum_{i\not\in R_{r-2}}\nabla F_i(\iw_t(r-1)) - \nabla F_{i_{r-1}}(\iw_t(r-1))]\\
& = \frac{1}{n-r+1}\{(n-r+2)[\bH_{t}^{r-2}\times (\iw_t(r-1) - \iw_t(r-2)]\\
& + \sum_{i\not\in R_{r-2}}\nabla F_i(\iw_t(r-2)) - \nabla F_{i_{r-1}}(\iw_t(r-1))\}.
\end{split}
\end{align*}
    
By subtracting the above formula from equation \eqref{eq: g_w_k_deriv}, i.e., the update rule for the approximate gradient,
the norm of the approximation error between true and approximate gradients is:
\begin{align*}
\begin{split}
    & \|\frac{1}{n-r+1}\sum_{i\not\in R_{r-1}}\nabla F_i(\iw_t(r-1)) - \ag(\iw_t(r-1))\| \\
    & = \frac{1}{n-r+1}\|(n-r+2)(\bH_t^{r-2}-\B_{j_m}^{r-2})\times (\iw_t(r-1) - \iw_t(r-2))\\
    & + \sum_{i\not\in R_{r-2}}\nabla F_i(\iw_t(r-2))- (n-r+2)\ag(\iw_t(r-2)) \|\\
\end{split}
\end{align*}

Then by using the triangle inequality, Corollary \ref{corollary: approx_hessian_real_hessian_bound_multi} on the approximation accuracy of the quasi-Hessian (where the bound is in terms of $\xi$), and the bound on $\|\iw_t(r-1)-\iw_t(r-2)\|$, the formula above is bounded as:
\begin{align}\label{eq: real_approx_gradient_gap_deriv1}
    \begin{split}
        & \leq \frac{n-r+2}{n-r+1}\|\bH_t^{r-2} - \B_{j_m}^{r-2}\|\|\iw_t(r-1) - \iw_t(r-2)\|\\
        & + \frac{1}{n-r+1} \|\sum_{i\not\in R_{r-2}}\nabla F_i(\iw_t(r-2)) - (n-r+2)\ag(\iw_t(r-2))\| \\
        & \leq \frac{n-r+2}{n-r+1}\xi_{j_1,j_m}(r-2) \wikbounds + \frac{n-r+2}{n-r+1}\|\frac{1}{n-r+2}\sum_{i\not\in R_{r-2}}\nabla F_i(\iw_t(r-2)) - \ag(\iw_t(r-2))\|
    \end{split}
\end{align}

By using  that $\xi_{j_1,j_m}(r-2) \leq \frac{\mu}{2}$, the formula above is bounded as:
\begin{align*}
    & \leq \frac{n-r+2}{n-r+1}\frac{\mu}{2} (\wikbounds )
       + \frac{n-r+2}{n-r+1}\|\frac{1}{n-r+2}\sum_{i\not\in R_{r-2}}\nabla F_i(\iw_t(r-2)) - \ag(\iw_t(r-2))\|
\end{align*}

We can use this recursively. Note that $\nabla F(\iw_t(0)) = \ag(\iw_t(0))$. In the end, we get the following inequality:
\begin{align*}
    \begin{split}
        & \leq \sum_{k=1}^{r-1}\frac{n-k}{n-r}\frac{\mu}{2} (\wikbounds)
         \leq \wikbounds\frac{\mu}{2} \sum_{k=1}^{r-1}\frac{n-k}{n-r}
    \end{split}
\end{align*}

Also for $r \ll n$, $\frac{n-k}{n-r} \leq 2$ (in fact we assumed $r/n\le\delta$ for a sufficiently small $\delta$, so this holds). So we get the bound
$r\wikbounds \mu$.
    

\end{proof}

Note that for $\|\frac{1}{n-r+1}\sum_{i\not\in R_{r-1}}\nabla F_i(\iw_t(r-1))-\ag(\iw_t(r-1))\|$, we get a tighter bound when $t\rightarrow \infty$ by using equation \eqref{eq: real_approx_gradient_gap_deriv1}, Lemma \ref{lemma: bound_d_j_multi} (i.e. $d_{j_a,j_b}(r) \leq d_{j_a,j_b}(0) + 2r\cdot \wikbounds$) and Lemma \ref{Lemma: bound_d_j} without using $\xi_{j_1,j_m}(r-1) \leq \frac{\mu}{2}$, which starts by bounding $\xi_{j_1,j_m}(k-1)$ where $k <= r$, $j_1 = j_0 + xT_0$ and $j_m = j_0 + (x+m-1)T_0$:

\begin{align}\label{eq: xi_bound_multi}
    \begin{split}
    & \xi_{j_1,j_m}(k-1) = A d_{j_1,j_m + T_0 - 1}(k-1)
+ A \wikbounds \\
& \leq A d_{j_1,j_m + T_0 - 1}(0) + 2(k-1)A \cdot \wikbounds
+ A \wikbounds\\
& \leq A (1-\mu\eta)^{j_0 + xT_0}d_{0,mT_0-1}(0) + A(2k-1)\wikbounds,
    \end{split}
\end{align}

which can be plugged into Equation \eqref{eq: real_approx_gradient_gap_deriv1}, i.e.:
\begin{align}\label{eq: real_approx_gradient_gap_tight}
    \begin{split}
        & \|\frac{1}{n-r+1}\sum_{i\not\in R_{r-1}}\nabla F_i(\iw_t(r-1))-\ag(\iw_t(r-1))\| \\
        & \leq \frac{n-r+2}{n-r+1}\xi_{j_1,j_m}(r-2) \wikbounds\\
        & + \frac{n-r+2}{n-r+1}\|\frac{1}{n-r+2}\sum_{i\not\in R_{r-2}}\nabla F_i(\iw_t(r-2)) - \ag(\iw_t(r-2))\|\\
        & \leq \sum_{k=1}^{r-1} \frac{n-k+1}{n-k}\xi_{j_1,j_m}(k-1)\wikbounds\\
        & \leq \sum_{k=1}^{r-1}\frac{n-k+1}{n-k}[A (1-\mu\eta)^{j_0 + xT_0}d_{0,mT_0-1}(0) + A(2k-1)\wikbounds] \cdot\wikbounds\\
        & \leq 2A (1-\mu\eta)^{j_0 + xT_0}d_{0,mT_0-1}(0)r\wikbounds + 2A(r\wikbounds)^2
    \end{split}
\end{align}

The last step uses  that $\frac{n-k+1}{n-k} \leq 2$ and $\sum_{k=1}^{r-1} (2k-1) < \sum_{k=1}^r (2k-1) = r^2$. So when $t\rightarrow \infty$ and thus $x \rightarrow \infty$, $\|\frac{1}{n-r}\sum_{i\not\in R_r}\nabla F_i(\iw_t(r))-\ag(\iw_t(r))\| = o(\frac{r}{n})$.

Then based on Lemma \ref{lemma: gradient_multi_bound} and \ref{lemma: approximate_vs_real_gradient_bound}, the bound on $\|\ag(\iw_t(r))\|$ becomes:
\begin{align}\label{eq: approximate_gradient_bound}
\begin{split}
    & \|\ag(\iw_t(r-1))\|\\
    & = \|\ag(\iw_t(r-1)) - \frac{1}{n-r+1}\sum_{i\not\in R_{r-1}}\nabla F_i(\iw_t(r-1)) + \frac{1}{n-r+1}\sum_{i\not\in R_{r-1}}\nabla F_i(\iw_t(r-1))\|\\
    & \leq \|\ag(\iw_t(r-1)) - \frac{1}{n-r+1}\sum_{i\not\in R_{r-1}}\nabla F_i(\iw_t(r-1))\| + \|\frac{1}{n-r+1}\sum_{i\not\in R_{r-1}}\nabla F_i(\iw_t(r-1))\|\\
    & = r\wikbounds \mu + \wikbounds Lr + c_2 =(r\mu + Lr) \wikbounds + c_2
\end{split}
\end{align}

\textbf{Main results}

\begin{theorem}[Bound between iterates on full data and incrementally updated ones (online deletions)]\label{bfi_multi}
Suppose that for any $k < r$, $\|\iw_{t}(k) - \iw_{t}(k-1)\| \leq \wikbounds$. At the $r_{th}$ deletion request, consider an iteration $t$ indexed with $j_m$ for which $j_m \leq t < j_m +T_0 -1$, and suppose that we are at the $x$-th iteration of full gradient updates, so  $j_1 = j_0 + xT_0$, $j_m = j_0 + (m - 1 + x) T_0$. Suppose that we have the bounds $\|\bH_{t}^{r-1} - \B_{j_m}^{r-1}\| \leq \xi_{j_1,j_m}(r-1) = Ad_{j_1,j_m + T_0 - 1}(r-1) + A(\wikbounds)$ (where we recalled the definition of $\xi$)  and $$\xi_{j_1,j_m}(r-1) = A d_{j_1,j_m + T_0 - 1}(r-1)
+ A (\wikbounds) \leq \frac{\mu}{2}$$ for all iterations $x$.
Then 
$$\|\iw_{t+1}(r) - \iw_{t+1}(r-1) \| \leq
\wikbounds.$$ 
Recall that $c_0$ is the Lipshitz constant of the Hessian, $M_1$ and $A$ are defined in Theorem \ref{wu_multi} and Corollary \ref{corollary: approx_hessian_real_hessian_bound_multi} respectively, and do not depend on $t$, 
\end{theorem}

Then by using the same derivation as the proof of Theorem \ref{bfi2}, we get the following results at the $r_{th}$ deletion request.

\begin{theorem}[Bound between iterates on full data and incrementally updated ones (all iterations, online deletion)]\label{bfi2_multi}
At the deletion request $r$, if for all $k<r$, $\|\iw_t(k)-\iw_t(k-1)\| \leq \wikbounds$ holds, then for any $j_m < t < j_m +T_0 - 1$,  $$\|\iw_t(r)-\iw_t(r-1)\| \leq \wikbounds$$ and $$\|\bH_{t}^{r-1} - \B_{j_m}^{r-1}\| \leq \xi_{j_1,j_m}(r-1):= A d_{j_1,j_m + T_0 - 1}(r-1)
+ A \wikbounds$$ 
and $$\|\frac{1}{n-r+1}\sum_{i\not\in R_{r-1}}\nabla F_i(\iw_t(r-1))-\ag(\iw_t(r-1))\|\leq r\wikbounds \mu $$
hold
\end{theorem}

Then by induction (the base case is similar to Theorem \ref{bfi2}), we know that the following theorem holds for all iterations $t$:
\begin{theorem}[Bound between iterates on full data and incrementally updated ones (all iterations, all deletion requests, online deletion)]\label{bfi2_multi_all}
At the $r_{th}$ deletion request, for any $j_m < t < j_m +T_0 - 1$,  $$\|\iw_t(r)-\iw_t(r-1)\| \leq \wikbounds$$ and $$\|\bH_{t}^{r-1} - \B_{j_m}^{r-1}\| \leq \xi_{j_1,j_m}(r-1):= A d_{j_1,j_m + T_0 - 1}(r-1)
+ A \wikbounds$$ and 
$$\|\frac{1}{n-r+1}\sum_{i\not\in R_{r-1}}\nabla F_i(\iw_t(r-1))-\ag(\iw_t(r-1))\| \leq r\wikbounds \mu $$
hold
\end{theorem}

Then by induction (from the $r_{th}$ deletion request to the $1_{st}$ deletion request), the following inequality holds:
\begin{align*}
    \|\iw_t(r)-\iw_t(0)\| = \|\iw_t(r)-\w_t\| \leq r \cdot \wikbounds
\end{align*}

Then by using equation \eqref{eq: w_u_w_gap}, the following inequality holds:
\begin{align}\label{eq: w_u_w_i_prev_bound}
    \begin{split}
    &\|\uw_t(r) - \iw_t(r-1)\| = \|\uw_t(r) - \w_t + \w_t - \iw_t(r-1)\|\\
    & \leq \|\uw_t(r) - \w_t\| + \|\w_t - \iw_t(r-1)\| \\
    & \leq M_1 \frac{r}{n} + (r-1)\cdot \wikbounds := M_2 \frac{r}{n}        
    \end{split}
\end{align}

where $M_2$ is a constant which does not depend on $t$ or $k$.

In the end, we get a similar result for the bound on $\|\iw_t(r) - \uw_t(r)\|$:
\begin{theorem}
[Convergence rate of \Increm\ (online deletion)]\label{main10_multi}
At the $r_{th}$ deletion request, for all iterations $t$, the result $\iw_t(r)$ of \Increm, Algorithm \ref{alg: update_algorithm_online}, approximates the correct iteration values $\uw_t(r)$ at the rate
$$\|\uw_t(r) - \iw_t(r)\| = o(\frac{r}{n}).$$
So $\|\uw_t(r) - \iw_t(r)\|$ is of a lower order than $\frac{r}{n}$.
\end{theorem}

\textbf{The proof of Theorem \ref{bfi_multi}}
\begin{proof}

Note that the approximated update rules for $\iw_{t}$ at the $r_{th}$ and the $(r-1)_{st}$ deletion request are:
\begin{align}\label{eq: update_rule_w_it_curr}
\begin{split}
    & \iw_{t+1}(r) = \iw_t(r) - \frac{\eta}{n-r}\{(n-r+1)[\B_{j_m}^{r-1}(\iw_t(r) - \iw_t(r-1))\\
    & + \ag(\iw_{t}(r-1))] - \nabla F_{i_r}(\iw_t(r))\}
\end{split}
\end{align}

and 
\begin{align}\label{eq: update_rule_w_it_prev0}
\begin{split}
    & \iw_{t+1}(r-1) = \iw_{t}(r-1) - \frac{\eta}{n-r+1}\{(n-r+2)[\B_{j_m}^{r-2}(\iw_t(r-1) - \iw_t(r-2))\\
    & + \ag(\iw_{t}(r-2))] - \nabla F_{i_{r-1}}(\iw_t(r-1))\}.
\end{split}
\end{align}

Note that since $$\ag(\iw_t(r-1)) = \frac{1}{n-r+1}\{(n-r+2)[\B_{j_m}^{r-2}(\iw_t(r-1) - \iw_t(r-2))$$
$$+ \ag(\iw_t(r-2))] - \nabla F_{i_{r-1}}(\iw_t(r-1))\},$$
then equation \eqref{eq: update_rule_w_it_prev0} can be rewritten as:
\begin{align}\label{eq: update_rule_w_it_prev}
    \begin{split}
    & \iw_{t+1}(r-1) = \iw_{t}(r-1) - \frac{\eta}{n-r+1}\{(n-r+2)[\B_{j_m}^{r-2}(\iw_t(r-1) - \iw_t(r-2))\\
    & + \ag(\iw_{t}(r-2))] - \nabla F_{i_{r-1}}(\iw_t(r-1))\}\\
    & = \iw_{t}(r-1) - \eta \ag(\iw_t(r-1)).
\end{split}
\end{align}

Then by subtracting equation \eqref{eq: update_rule_w_it_prev0} from equation \eqref{eq: update_rule_w_it_prev}, the result becomes:

\begin{align*}
\begin{split}
    & \iw_{t+1}(r) - \iw_{t+1}(r-1) \\
    & = (\iw_{t}(r) - \iw_{t}(r-1)) - \frac{\eta}{n-r}\{(n-r+1)[\B_{j_m}^{r-1}(\iw_t(r) - \iw_t(r-1))\\
    & + \ag(\iw_{t}(r-1))] - \nabla F_{i_r}(\iw_t(r))\} + \eta \ag(\iw_t(r-1)).
\end{split}
\end{align*}

Then by adding and subtracting $\bH_t^{r-1}$ and $\frac{1}{n-r+1}\sum_{i\not\in R_{r-1}}\nabla F(\iw_t(r-1))$ in the formula above and rearranging the result properly, it becomes:
\begin{align}\label{eq: w_i_t_prev_curr_gap_1}
\begin{split}
    & \iw_{t+1}(r) - \iw_{t+1}(r-1) \\
    & = (\textbf{I} - \eta\frac{n-r+1}{n-r}(\B_{j_m}^{r-1} - \bH_t^{r-1}))(\iw_{t}(r) - \iw_{t}(r-1))\\
    & - \frac{\eta}{n-r}\{(n-r+1)[\bH_t^{r-1}(\iw_t(r) - \iw_t(r-1))\\
    & + \frac{1}{n-r+1}\sum_{i\not\in R_{r-1}}\nabla F(\iw_t(r-1)) - \frac{1}{n-r+1}\sum_{i\not\in R_{r-1}}\nabla F(\iw_t(r-1)) \\
    & + \ag(\iw_{t}(r-1))] - \nabla F_{i_r}(\iw_t(r))\} + \eta \ag(\iw_t(r-1)).
\end{split}
\end{align}

Then by using the fact that 
\begin{align*}
    & \bH_t^{r-1}(\iw_t(r) - \iw_t(r-1)) + \frac{1}{n-r+1}\sum_{i\not\in R_{r-1}}\nabla F(\iw_t(r-1))\\
    & = \frac{1}{n-r+1}\sum_{i\not\in R_{r-1}}\nabla F(\iw_t(r))
\end{align*}

and $$(\sum_{i\not\in R_{r-1}}\nabla F(\iw_t(r))) - \nabla F_{i_r}(\iw_t(r)) = \sum_{i\not\in R_{r}}\nabla F(\iw_t(r)),$$
Equation \eqref{eq: w_i_t_prev_curr_gap_1} becomes:
\begin{align}\label{eq: w_i_t_prev_curr_gap_2}
\begin{split}
    & \iw_{t+1}(r) - \iw_{t+1}(r-1) \\
    & = (\textbf{I} - \eta\frac{n-r+1}{n-r}(\B_{j_m}^{r-1} - \bH_t^{r-1}))(\iw_{t}(r) - \iw_{t}(r-1))\\
    & - \frac{\eta}{n-r}[\sum_{i\not\in R_{r}}\nabla F(\iw_t(r)) - \sum_{i\not\in R_{r-1}}\nabla F(\iw_t(r-1)) \\
    & + (n-r+1)\ag(\iw_{t}(r-1))] + \eta \ag(\iw_t(r-1)).
\end{split}
\end{align}

Also note that by using the Cauchy mean-value theorem, the following equation holds:
\begin{align*}
    \begin{split}
        & \sum_{i\not\in R_{r}}\nabla F_i(\iw_t(r)) - \sum_{i\not\in R_{r-1}}\nabla F_i(\iw_t(r-1)) \\
        & = \sum_{i\not\in R_{r}}\nabla F_i(\iw_t(r)) - \sum_{i\not\in R_{r}}\nabla F_i(\iw_t(r-1)) - \nabla F_{i_r}(\iw_t(r-1)) \\
        & = [\sum_{i \not\in R_r} \int_0^1 \bH_i(\iw_{t}(r-1) + x (\iw_{t}(r) - \iw_{t}(r-1)))dx](\iw_t(r)-\iw_t(r-1)) - \nabla F_{i_r}(\iw_t(r-1)),
    \end{split}
\end{align*}

which can be plugged into equation \eqref{eq: w_i_t_prev_curr_gap_2}, i.e.:
\begin{align}\label{eq: w_i_t_prev_curr_gap_3}
\begin{split}
    & \iw_{t+1}(r) - \iw_{t+1}(r-1) \\
    & = (\textbf{I} - \eta\frac{n-r+1}{n-r}(\B_{j_m}^{r-1} - \bH_t^{r-1}))(\iw_{t}(r) - \iw_{t}(r-1))\\
    & - \frac{\eta}{n-r}\{[\sum_{i \not\in R_r} \int_0^1 \bH_i(\iw_{t}(r-1) + x (\iw_{t}(r) - \iw_{t}(r-1)))dx]\\
    & \cdot(\iw_{t}(r) - \iw_{t}(r-1)) - \nabla F_{i_r}(\iw_t(r-1))\\
    & + (n-r+1)\ag(\iw_{t}(r-1))\} + \eta \ag(\iw_t(r-1)),
\end{split}
\end{align}

which can be rearranged as:
\begin{align}\label{eq: w_i_t_prev_curr_gap_4}
\begin{split}
    & \iw_{t+1}(r) - \iw_{t+1}(r-1) \\
    & = (\textbf{I} - \eta\frac{n-r+1}{n-r}(\B_{j_m}^{r-1} - \bH_t^{r-1}))(\iw_{t}(r) - \iw_{t}(r-1))\\
    & - \frac{\eta}{n-r}\{[\sum_{i \not\in R_r} \int_0^1 \bH_i(\iw_{t}(r-1) + x (\iw_{t}(r) - \iw_{t}(r-1)))dx]\\
    & \cdot(\iw_{t}(r) - \iw_{t}(r-1)) - \nabla F_{i_r}(\iw_t(r-1))\} - \frac{\eta}{n-r}\ag(\iw_t(r-1)).
\end{split}
\end{align}

Then by taking the matrix norm on both sides of equation \eqref{eq: w_i_t_prev_curr_gap_4} and using  that $\|\bH_i(\iw_{t}(r-1) + x (\iw_{t}(r) - \iw_{t}(r-1)))\| \geq \mu$ and $\|\B_{j_m}^{r-1} - \bH_t^{r-1}\| \leq \xi_{j_1,j_m}(r-1)$, equation \eqref{eq: w_i_t_prev_curr_gap_4} can be bounded as:
\begin{align*}
    \begin{split}
        & \|\iw_{t+1}(r) - \iw_{t+1}(r-1)\|\\
        & \leq (1-\eta\mu)\|\iw_{t}(r) - \iw_{t}(r-1)\| \\
        & + \frac{(n-r+1)\eta}{n-r}\xi_{j_1,j_m}(r-1)\|\iw_{t}(r) - \iw_{t}(r-1)\|\\
        & + \frac{\eta}{n-r}\|\nabla F_{i_r}(\iw_t(r-1))\| + \|\frac{\eta}{n-r}\ag(\iw_t(r-1))\|.
    \end{split}
\end{align*}

Then by using Lemma \ref{lemma: gradient_multi_bound} and equation \eqref{eq: approximate_gradient_bound}, the formula above becomes:
\begin{align*}
    & \leq (1-\eta\mu + \frac{(n-r+1)\eta}{n-r}\xi_{j_1,j_m}(r))\|\iw_{t}(r) - \iw_{t}(r-1)\|\\
        & + \frac{\eta}{n-r}(\wikbounds L(r-1) + c_2) + \frac{\eta}{n-r}(\wikbounds(r-1)\mu + \wikbounds L(r-1) + c_2).
\end{align*}

By using the bound on $\xi_{j_1,j_m}(r)$ and applying the above formula recursively across all iterations, the formula above becomes:
\begin{align*}
    \begin{split}
    & \leq \frac{1}{\eta\mu - \frac{\eta(n-r+1)}{n-r}\frac{\mu}{2}}(\frac{\eta}{n-r}(\wikbounds L(r-1) + c_2)\\
    & + \frac{\eta}{n-r}(\wikbounds L(r-1) + \wikbounds \mu (r-1) + c_2)\\
    & = \frac{2}{(n-r-1)\mu}((\wikbounds (2L(r-1) + (r-1)\mu) + 2c_2).
    \end{split}
\end{align*}

Then by using  that $M_1 = \frac{2c_2}{\mu}$ and $\wikbounds = \wikbound{r}$, the formula above can be rewritten as:
\begin{align*}
    & = \frac{2}{(n-r-1)\mu}\frac{\frac{2M_1(r-1)}{n}(2L+\mu) + \mu M_1 (1-\frac{r+1}{n} - \frac{2(r-1)}{n}(\frac{2L+\mu}{\mu})) }{1-\frac{r+1}{n} - \frac{2(r-1)}{n}(\frac{2L+\mu}{\mu})}\\
    & = \frac{2 \frac{M_1}{n}}{1-\frac{r+1}{n} - \frac{2(r-1)}{n}(\frac{2L+\mu}{\mu})} = \wikbounds.
\end{align*}

This finishes the proof.

\end{proof}

\textbf{The proof of Theorem \ref{main10_multi}}
\begin{proof}
Recall that the update rule for $\uw_t(r)$ is:
$$\uw_{t+1}(r) = \uw_t(r) - \eta\frac{1}{n-r}\sum_{i\not\in R_r}\nabla F(\uw_t(r))$$ and the update rule for $\iw_t(r)$ is (where the gradients are explicitly evaluated):
$$\iw_{t+1}(r) = \iw_{t}(r) - \frac{\eta}{n-r}[(n-r+1)(\B_{j_m}^{r-1}(\iw_t(r) - \iw_t(r-1)) + \ag(\iw_{t}(r-1))) - \nabla F_{i_r}(\iw_t(r))].$$
Then by subtracting $\iw_{t+1}(r)$ from $\uw_{t+1}(r)$, we get:
\begin{align*}
    \begin{split}
        & \|\iw_{t+1}(r) - \uw_{t+1}(r)\|\\
        & = \|\iw_{t}(r) - \uw_{t}(r) - \frac{\eta}{n-r}\{(n-r+1)[\B_{j_m}^{r-1}(\iw_t(r) - \iw_t(r-1)) \\
        & + \ag(\iw_{t}(r-1))] - \nabla F_{i_r}(\iw_t(r))\} + \frac{\eta}{n-r}\sum_{i \not\in R_r}\nabla F(\uw_t(r))\|.
    \end{split}
\end{align*}

Then by bringing in $\bH_t^{r-1}$ and $\frac{1}{n-r+1}\sum_{i\in R_{r-1}}\nabla F(\iw_t(r-1))$ into the formula above, we get:
\begin{align*}
    \begin{split}
        & = \|\iw_{t}(r) - \uw_{t}(r) - \frac{(n-r+1)\eta}{n-r}\left[\left(\B_{j_m}^{r-1} - \bH_{t}^{r-1}\right)(\iw_{t}(r) - \iw_{t}(r-1))\right.\\
        & + \bH_{t}^{r-1}\mmop(\iw_{t}(r) - \iw_{t}(r-1)) + \ag(\iw_{t}(r-1))\\
        &- \frac{1}{n-r+1}\sum_{i\in R_{r-1}}\nabla F(\iw_t(r-1)) + \frac{1}{n-r+1}\sum_{i\in R_{r-1}}\nabla F(\iw_t(r-1))]\\
     & + \frac{\eta}{n-r}[\nabla F_{i_r}(\iw_{t}(r)) - \nabla F_{i_r}(\iw_{t}(r-1)) + \nabla F_{i_r}(\iw_{t}(r-1))] + \frac{\eta}{n-r}\sum_{\substack{i \not\in R_r}} \nabla F_i(\uw_{t}(r))\|.
    \end{split}
\end{align*}

Then by using the triangle inequality and the result from equation \eqref{eq: real_approx_gradient_gap_tight}, the formula above can be bounded as:
\begin{align*}
    \begin{split}
        & \leq \|\iw_{t}(r) - \uw_{t}(r) - \frac{(n-r+1)\eta}{n-r}\left[\left(\B_{j_m}^{r-1} - \bH_{t}^{r-1}\right)(\iw_{t}(r) - \iw_{t}(r-1))\right.\\
        & + \left.\bH_{t}^{r-1}\mmop(\iw_{t}(r) - \iw_{t}(r-1)) + \frac{1}{n-r+1}\sum_{i\in R_{r-1}}\nabla F_i(\iw_t(r-1))\right]\\
     & + \frac{\eta}{n-r}[\nabla F_{i_r}(\iw_{t}(r)) - \nabla F_{i_r}(\iw_{t}(r-1)) + \nabla F_{i_r}(\iw_{t}(r-1))]\\
     & + \frac{\eta}{n-r}\sum_{\substack{i \not\in R_r}} \nabla F_i(\uw_{t}(r))\| +2 A (1-\mu\eta)^{j_0 + xT_0}d_{0,(m-1)T_0}(0)r\wikbounds + 2A(r\wikbounds)^2.
    \end{split}
\end{align*}

Note that the first matrix norm in this formula is the same as equation \eqref{eq: w_i_w_u_gap_deriv1} by replacing $n$, $r$, $\iw_t$, $\uw_t$, $\w_t$, $\B_{j_m}$, $\bH_t$ and $\nabla F(\w_t)$ with $n-r+1$, $1$, $\iw_t(r)$, $\uw_t(r)$, $\iw_t(r-1)$, $\B_{j_m}^{r-1}$, $\bH_t^{r-1}$ and $\frac{1}{n-r+1}\sum_{i\not\in R_{r-1}}\nabla F_i(\iw_t(r-1))$ reps.. So by following the same derivation, the formula above can be bounded as:
\begin{align*}
\begin{split}
& \leq \|(\textbf{I} - \frac{\eta}{n-r}\sum_{\substack{i \not\in R_r}}\bH_{t,i}^{r-1})(\iw_{t}(r)-\uw_{t}(r))\|\\
& + \|\frac{(n-r+1)\eta}{n-r}\left[\left(\B_{j_m}^{r-1} - \bH_{t}^{r-1}\right)(\iw_{t}(r) - \uw_{t}(r))\right]\|\\
& + \|\frac{\eta}{n-r}[\sum_{i\not\in R_r}\int_0^1 \bH_i(\iw_t(r-1) + x(\uw_t(r)-\iw_t(r-1)))dx \\
& - \int_0^1 \bH_i(\iw_t(r-1) + x(\iw_t(r)-\iw_t(r-1)))dx](\uw_{t}(r)-\iw_t(r-1))\|\\
& + \|\frac{(n-r+1)\eta}{n-r}\left[\left(\B_{j_m}^{r-1} - \bH_{t}^{r-1}\right)(\uw_{t}(r) - \iw_{t}(r-1))\right]\|\\
& + 2A (1-\mu\eta)^{j_0 + xT_0}d_{0,(m-1)T_0}(0)r\wikbounds + 2A(r\wikbounds)^2.
\end{split}
\end{align*}

Then by using the following facts:
\benum 
\item $\|\textbf{I} - \eta \bH_{t,i}^{r-1}\| \le 1 - \eta \mu;$ 
\item from Theorem \ref{bfi2_multi_all} on the approximation accuracy of the quasi-Hessian to mean Hessian, we have the error bound $\|\bH_t^{r-1} - \B_{j_m}^{r-1}\|\le\xi_{j_1,j_m}(r-1);$
\item we bound the difference of integrated Hessians using the strategy from Equation \eqref{eq: hessian_diff_bound0};
\item from Equation \eqref{eq: w_u_w_i_prev_bound}, we have the error bound $\|\uw_{t}(r)-\iw_t(r-1)\|\le M_2 \frac{r}{n}$ (and this requires no additional assumptions),
\eenum
the expression can be bounded as follows:

\begin{align*}
    \begin{split}
        & \leq (1-\eta\mu + \frac{(n-r+1)\eta}{n-r}\xi_{j_0,j_0+(m-1)T_0}(r-1)+\frac{c_0M_2r\eta}{2n})\|\iw_t - \uw_t\|\\
        & + \frac{M_2(n-r+1)r\eta}{n(n-r)} \xi_{j_1,j_m}(r-1) + 2A (1-\mu\eta)^{j_0 + xT_0}d_{0,(m-1)T_0}(0)r\wikbounds\\
        & + 2A(r\wikbounds)^2,
    \end{split}
\end{align*}


which is very similar to equation \eqref{eq: w_u_w_i_gap_deriv_2} (except the difference in the coefficient). So by following the derivation after equation \eqref{eq: w_u_w_i_gap_deriv_2}, we know that:
\begin{align*}
    \|\iw_t(r)-\uw_t(r)\| = o(\frac{r}{n})
\end{align*}

when $t\rightarrow \infty$.

\end{proof}





\subsection{Extension of \Increm\ for non-strongly convex, non-smooth objective functions}\label{sec: non_convex_alg}

For the original version of the L-BFGS algorithm, strong convexity is essential to make the secant condition hold. In this subsection, we present our extension of \Increm\ to non-strongly convex, non-smooth objectives. 

To deal with non-strongly convex objectives, we assume that convexity holds in some local regions. When constructing the arrays $\Delta G$ and $\Delta W$, only the model parameters and their gradients where local convexity holds are used. 

For local non-smoothness, we found that even a small distance between $\w_t$ and $\iw_t$ can make the estimated gradient $\nabla F(\iw_t)$ drift far away from $\nabla F(\w_t)$.
To deal with this, we explicitly check if the norm of $\B_{j_m}(\w_t - \iw_t)$ (which equals to $\nabla F(\iw_t) - \nabla F(\w_t)$) is larger than the norm of $L(\w_t - \iw_t)$ for a constant $L$. In our experiments, $L$ is configured as 1.
The details of the modifications above are highlighted in Algorithm \ref{alg: update_algorithm_general}.

\begin{algorithm}[!htbp] 
\small
\SetKw{Continue}{continue}
\SetKwInOut{Input}{Input}
\SetKwInOut{Output}{Output}
\Input{The full training set $\left(\textbf{X}, \textbf{Y}\right)$, model parameters cached during the training phase for the full training samples $\{\w_{0}, \w_{1}, \dots, \w_{t}\}$ and corresponding gradients $\{\nabla F\left(\w_{0}\right), \nabla F\left(\w_{1}\right), \dots, \nabla F\left(\w_{t}\right)\}$, the removed training sample or the added training sample $R$, period $T_0$, total iteration number $T$, history size $m$, warmup iteration number $j_0$, learning rate $\eta$}
\Output{Updated model parameter $\iw_{t}$}
Initialize $\iw_{0} \leftarrow \w_{0}$

Initialize an array $\Delta G = \left[\right]$

Initialize an array $\Delta W = \left[\right]$

Initialize $last\_t = j_0$

$is\_explicit=False$

\For{$t=0;t<T; t++$}{

\eIf{\hl{$(t - last_t)  \mod T_0 == 0$ or $t \leq j_0$}}
{
    \hl{$is\_explicit=True$}
}

\eIf{$is\_explicit == True$ or $t \leq j_0$}
{
    $last\_t = t$ \label{line: explicit_start}
        
    compute $\nabla F\left(\iw_{t}\right)$ exactly
    
    compute $\nabla F\left(\iw_{t}\right) - \nabla F\left(\w_{t}\right)$ based on the cached gradient $\nabla F\left(\w_{t}\right)$
    
    \tcc{check local convexity}
    \If{\hl{$<\nabla F\left(\iw_{t}\right) - \nabla F\left(\w_{t}\right), \iw_{t} - \w_{t}> \leq 0$}}
    {
        \hl{compute $\iw_{t+1}$ by using exact GD update
        (equation \eqref{eq: update_rule_naive})    }
        
        \Continue
    }
    
    set $\Delta G\left[k\right] = \nabla F\left(\iw_{t}\right) - \nabla F\left(\w_{t}\right)$
    
    set $\Delta W\left[k\right] = \iw_{t} - \w_{t}$, based on the cached parameters $\w_{t}$
    
    $k\leftarrow k+1$
    
    compute $\iw_{t+1}$ by using exact GD update
    (equation \eqref{eq: update_rule_naive})
}
{
    Pass $\Delta W\left[-m:\right]$, $\Delta G\left[-m:\right]$, the last $m$ elements in $\Delta W$ and $\Delta G$, which are from the $j_1^{th}, j_2^{th},\dots, j_m^{th}$ iterations where $j_1 < j_2< \dots < j_m$ depend on $t$, $\textbf{v} = \iw_{t} - \w_{t}$, and the history size $m$, to the L-BFGFS Algorithm (See Supplement) to get the approximation of $\bH(\w_{t})\textbf{v}$, i.e., $\B_{j_m}\textbf{v}$
    
    \tcc{check local smoothness}
    \If{\hl{$\|\B_{j_m}\textbf{v}\| \geq \|\textbf{v}\|$}}
    {
        \hl{\textbf{go to} line \ref{line: explicit_start}}
    }
    
    Approximate $\nabla F\left(\iw_{t}\right) = \nabla F\left(\w_{t}\right) + \B_{j_m}\left(\iw_{t} - \w_{t}\right)$
    
    Compute $\iw_{t+1}$ by using the "leave-$r$-out" gradient formula, based on the approximated $\nabla F(\iw_{t})$ 
}
}

\Return $\iw_{t}$
\caption{DeltaGrad (general models)}
\label{alg: update_algorithm_general}
\end{algorithm}

%% file: Sections/other_experiements.tex

\section{Supplementary experiments}\label{sec: supple_exp}

In this section, we present some supplementary experiments that could not be presented in the paper due to space limitations.

\yw{
\subsection{Experiments with large deletion rate}}
In this experiment, instead of deleting at most 1\% of training samples each time as we did in Section \ref{sec: experiment} in the main paper, we vary the deletion rate from 0 to up to 20\% on \MNIST\ dataset and still compare the performance between \Increm\ (with $T_0$ as 5 and $j_0$ as 10) and \Basel. All other hyper-parameters such as the learning rate and mini-batch size remain the same as in Section \ref{sec: experiment} in the main paper.

The experimental results in Figure \ref{fig: mnist_large} show that even with the largest deletion rate, i.e. 20\%, \Increm\ can still be 1.67x faster than \Basel\ (2.27s VS 1.53s) and the error bound between their resulting model parameters (i.e. $\oiw$ VS $\ouw$) are still acceptable (on the order of $10^{-3}$), far smaller than the error bound between $\ouw$ and $\ow$ (on the order of $10^{-1}$). Such a small difference between $\oiw$ and $\ouw$ also results in almost the same prediction performance, i.e. $87.460 \pm 0.0011 \%$ and $87.458 \pm0.0012 \%$ respectively. This experiment thus provides some justification for the feasibility of \Increm\ even when the number of the removed samples is not far smaller than the entire training dataset size.

\begin{figure}
\begin{center}
\centerline{\includegraphics[width=0.4\columnwidth, height=0.3\columnwidth]{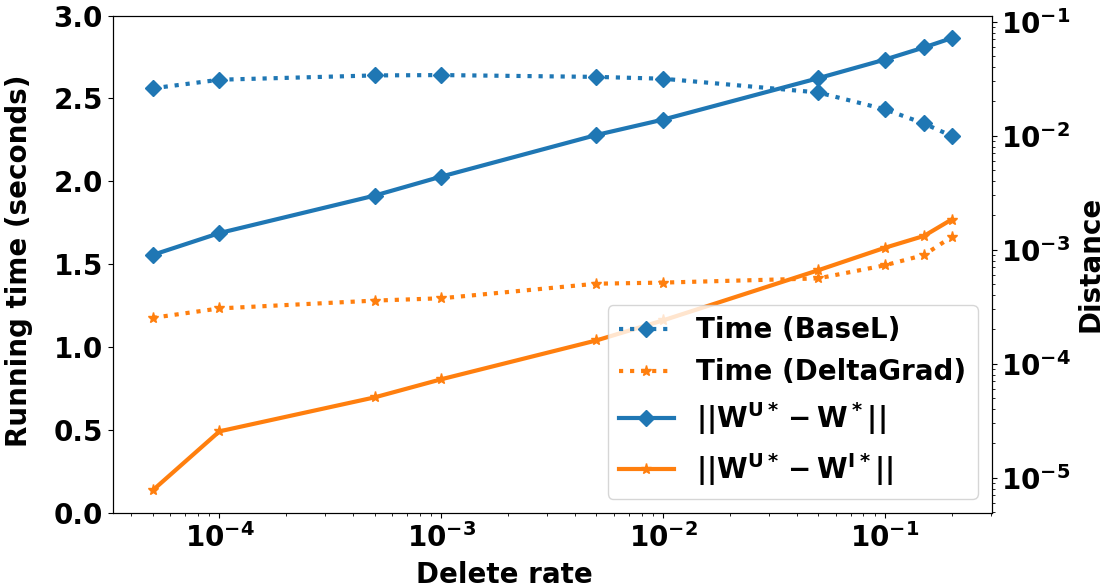}}
\caption{Running time and distance with varied deletion rate up to 20\%}
\label{fig: mnist_large}
\end{center}
\end{figure}

\begin{figure}[ht]
\begin{center}
\centerline{\includegraphics[width=1\columnwidth, height=0.3\columnwidth]{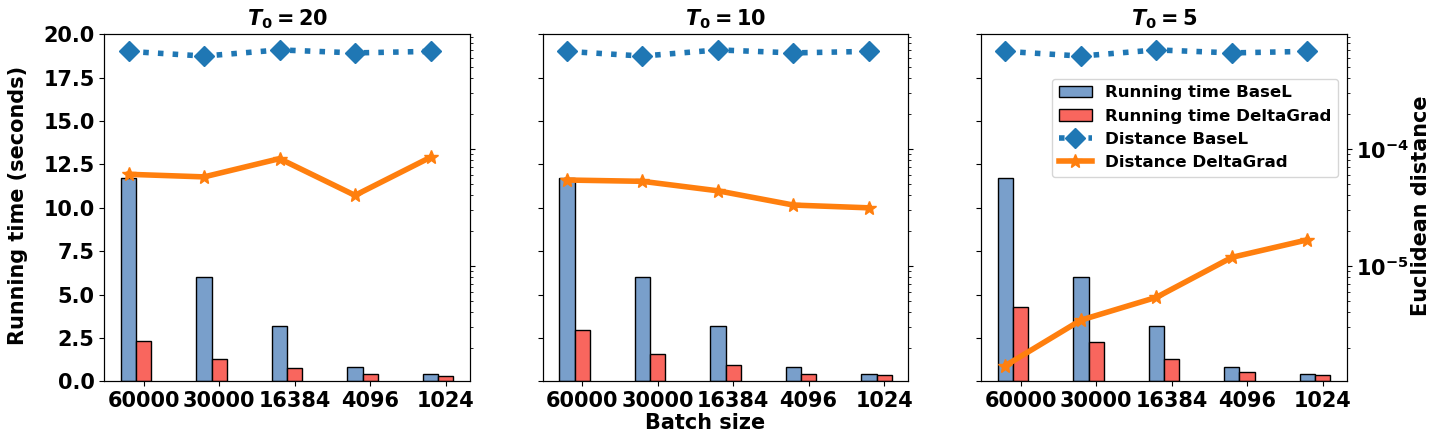}}
\caption{Running time and distance comparison with varying mini-batch size under fixed $j_0 = 10$ and varying $T_0$ ($T_0 = 20$ VS $T_0 = 10$ VS $T_0 = 5$) }
\label{fig: time_distance_varying_batch_period}
\end{center}
\end{figure}

\begin{figure}[ht]
\begin{center}
\centerline{\includegraphics[width=1\columnwidth, height=0.3\columnwidth]{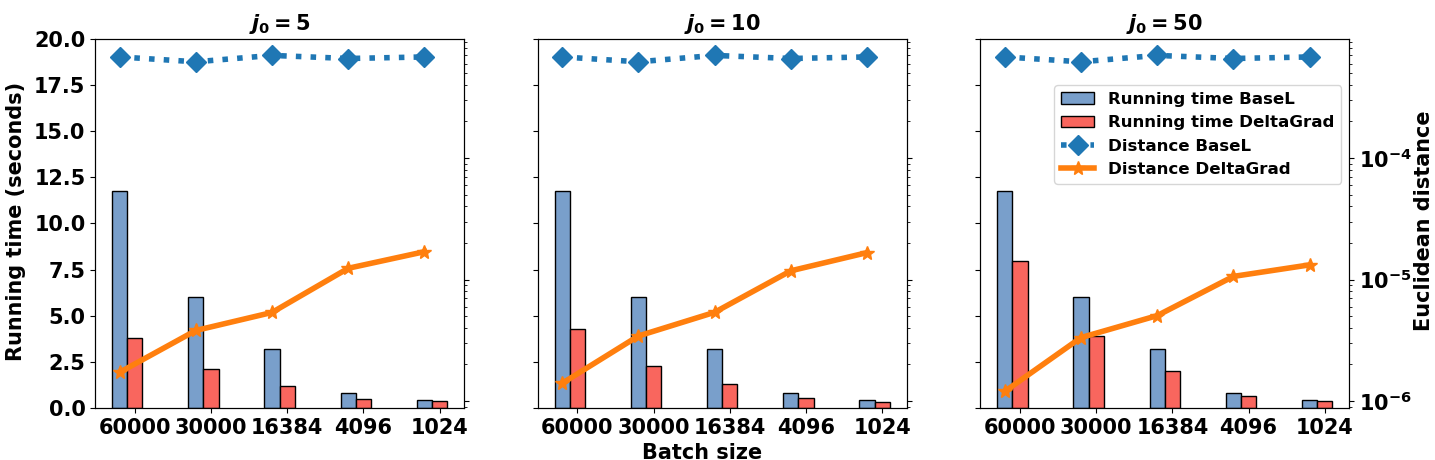}}
\caption{Running time and distance comparison with varying mini-batch size under fixed $T_0 = 5$ and varying $j_0$ ($j_0 = 5$ VS $j_0 = 10$ VS $j_0 = 50$) }
\label{fig: time_distance_varying_batch_init}
\end{center}
\end{figure}

\begin{figure}
\begin{center}
\centerline{\includegraphics[width=0.8\columnwidth, height=0.3\columnwidth]{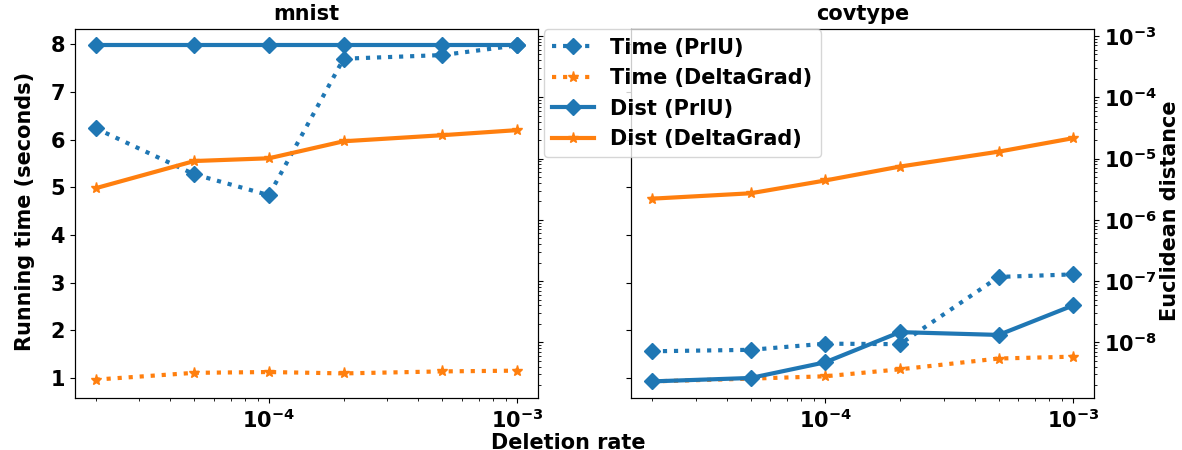}}
\caption{Comparison of \Increm\ and \PrIU}
\label{fig: comparison_MNIST}
\end{center}
\end{figure}

\subsection{Influence of hyper-parameters on performance}\label{sec: other_para_influence}
To begin with, the influence of different hyper-parameters used in SGD and \Increm\ is explored. We delete one sample from the training set of \MNIST\ by running regularized logistic regression with the same learning rate and regularization rate as in Section \ref{sec: experiment} and varying mini-batch sizes (1024 - 60000), $T_0$ ($T_0 = 20$, $10$, $5$) and $j_0$ ($j_0 = 5$, $10$, $50$). 
The experimental results are presented in Figure \ref{fig: time_distance_varying_batch_period}-\ref{fig: time_distance_varying_batch_init}. For different mini-batch sizes, we also used different epoch numbers to make sure that the total number of running iterations/steps in SGD are roughly the same. In what follows, we analyze how the mini-batch size, the hyper-parameters $T_0$ and $j_0$ influence the performance, thus providing some hints on how to choose proper hyper-parameters when \Increm\ is used.

\textbf{Influence of the mini-batch size.} It is clear from 
Figure \ref{fig: time_distance_varying_batch_period}-\ref{fig: time_distance_varying_batch_init} 
that with larger mini-batch sizes, \Increm\ can gain more speed with longer running time for both \Basel\ and \Increm. As discussed in Section \ref{sec: experiment},
to compute the gradients, other GPU-related overhead (the overhead to copy data from CPU DRAM to GPU DRAM, the time to launch the kernel on GPU) cannot be ignored. This can become more significant when compared against the smaller computational overhead for smaller mini-batch data. Also notice that, when $T_0 = 5$, with increasing $B$, the difference between $\uw$ and $\iw$ becomes smaller and smaller, which matches our conclusion in Theorem \ref{main10_sgd}, 
i.e. with larger $B$, the difference $o(\frac{r}{n} + \frac{1}{B^{\frac{1}{4}}})$ is smaller.

\textbf{Influence of $T_0$.} By comparing the three sub-figures in Figure \ref{fig: time_distance_varying_batch_period}, the running time slightly (rather than significantly) decreases with increasing $T_0$ for the same mini-batch size. This is explained by the earlier analysis in Section \ref{sec: experiment} on the non-ideal performance for GPU computation over small matrices. Interestingly, when $T_0 =10$ or $T_0=20$, $\|\isw - \usw\|$ does not decrease with larger mini-batch sizes. This is because in Formula \eqref{eq: wu_w_i_sgd_diff_last}, one component of the bound of $\|\isw - \usw\|$ is $$\frac{M_1(\frac{r}{n} + \frac{1}{B^{1/4}})}{C(1-\frac{r}{n} - \frac{1}{B^{1/4}})}(1-\eta C)^{yT_0}(1-\eta\mu)^{j_0}d_{0,mT_0 - 1}\frac{1}{1-(\frac{1-\eta\mu}{1-\eta C})^{T_0}}$$ (while the other component is $o((\frac{r}{n} + \frac{1}{B^{\frac{1}{4}}}))$). Here $d_{0,mT_0-1}$ increases with larger $T_0$ and the term $(1-\eta C)^{yT_0}$ is not arbitrarily approaching 0 since $yT_0$ cannot truly go to infinity. So when $T_0=20$ and $T_0=10$, this component becomes the dominating term in the bound of $\|\isw -\usw\|$. So to make the bound $o((\frac{r}{n}+\frac{1}{B^{\frac{1}{4}}}))$ hold, so that we can adjust the bound of $\|\isw - \usw\|$ by varying $B$, proper choice of $T_0$ is important. For example, $T_0 = 5$ is a good choice for the \MNIST\ dataset. This can achieve speed-ups comparable to larger $T_0$ without sacrificing the closeness between $\isw$ and $\usw$.

\textbf{Influence of $j_0$.} By comparing the three sub-figures in Figure \ref{fig: time_distance_varying_batch_init}, with increasing $j_0$, long ``burn-in'' iterations are expected, thus incurring more running time. This, however, does not significantly reduce the distance between $\isw$ and $\usw$. It indicates that we can select smaller $j_0$, e.g. 5 or 10 for more speed-up.

\textbf{Discussions on tuning the hyper-parameters for \Increm.} Through our extensive experiments, we found that for regularized logistic regression, setting $T_0$ as 5 and $j_0$ as $5-20$ would lead to some of the most favorable trade-offs between running time and the error $\|\usw-\isw\|$. But in terms of more complicated models, e.g. 2-layer DNN, higher $j_0$ (even half of the total iteration number) and smaller $T_0$ (2 or 1) are necessary. Similar experiments were also conducted on adding training samples, in which similar trends were observed.

\subsection{Comparison against the state-of-the-art work}
To our knowledge, the closest work to ours is \cite{wu2020data}, which targets simple ML models, i.e. linear regression and regularized logistic regression with an ad-hoc solution (called \PrIU) rather than solutions for general models. Their solutions can only deal with the deletion of samples from the training set without supporting the addition of samples. In our experiments, we compared \Increm\ (with $T_0=5$ and $j_0=10$) against \PrIU\ by running regularized logistic regression over \MNIST\ and \covtype\ with the same mini-batch size (16384), the same learning rate and regularization rate, but with varying deletion rates. 

\begin{table}[!htbp]
\centering
\small
\caption{Memory usage of \Increm\ and \PrIU (GB)}
\begin{tabular}[!h]{|>{\centering\arraybackslash}p{2cm}|>{\centering\arraybackslash}p{1cm}|>{\centering\arraybackslash}p{1.5cm}|>{\centering\arraybackslash}p{1cm}|>{\centering\arraybackslash}p{1.5cm}|} \hline
\multirow{2}{*}{\makecell{Deletion rate}} & \multicolumn{2}{c|}{\makecell{\MNIST}} & \multicolumn{2}{c|}{\covtype}\\ \hhline{~----}
 & \PrIU&\Increm & \PrIU& \Increm \\ \hline
$2\times 10^{-5}$& 26.61& 2.74& 9.30&2.56
\\ \hline %
$5\times 10^{-5}$ & 27.02&2.74 & 9.30&2.56\\ \hline
$1\times 10^{-4}$ & 27.13&2.74& 9.30& 2.55\\ \hline
$2\times 10^{-4}$ & 27.75&2.74& 9.31&2.56\\ \hline
$5\times 10^{-4}$ & 29.10&2.74 & 10.67&2.56\\ \hline
$1\times 10^{-3}$ & 29.10&2.74 & 10.67&2.56\\\hline
\end{tabular}
\label{Table: MNIST_curr_old_comparison}
\end{table}

The running time and the distance term $\|\uw-\iw\|$ of both \PrIU\ and \Increm\ with varying deletion rate are presented in Figure \ref{fig: comparison_MNIST}. First, it shows that \Increm\ is always faster than \PrIU, with more significant speed-ups on \MNIST. The reason is that the time complexity of \PrIU\ is $O(r p)$ for each iteration where $p$ represents the total number of model parameters while $r$ represents the reduced dimension after Singular Value Decomposition is conducted over some $p\times p$ matrix. This is a large integer for large sparse matrices, e.g. \MNIST. 

As a result, $O(rp)$ is larger than the time complexity of \Increm. Also, the memory usage of \PrIU\ and \Increm\ is shown in Table \ref{Table: MNIST_curr_old_comparison}. \PrIU\ needs much more DRAM (even 10x in \MNIST) than \Increm. The reason is that to prepare for the model update phase, \PrIU\ needs to collect more information during the training phase over the full dataset. This is needed in the model update phase and is quadratic in the number of the model parameters $p$. The authors of \cite{wu2020data} claimed that their solution cannot provide good performance over sparse datasets in terms of running time, error term $\uw-\iw$ and memory usage. In contrast, both the time and space overhead of \Increm\ are smaller, which thus indicates the potential of its usage in the realistic, large-scale scenarios.

\yw{\subsection{Experiments on large ML models}\label{sec: large_model_exp}}

In this section, we compare \Increm\ with \Basel\ using the state-of-the-art ResNet152 network \citep{he2016deep} (ResNet for short hereafter) with all but the top layer frozen, for which we use the pre-trained parameters from Pytorch torchvision library\footnote{https://pytorch.org/docs/stable/torchvision/models.html}. The pre-trained layers with fixed parameters are regarded as the feature transformation layer, applied over each training sample as the pre-processing step before the training phase. Those transformed features are then used to train the last layer of ResNet, which is thus equivalent to training a logistic regression model.

This experiment is conducted on \cifar\ dataset \citep{krizhevsky2009learning}, which is composed of 60000 32$\times$32 color images (50000 of them are training samples while the rest of are test samples). We run SGD with mini-batch size 10000, fixed learning rate 0.05 and L2 regularization rate 0.0001. Similar to the experimental setup introduced in Section \ref{sec: experiment} in the main paper, the deletion rate is varied from 0 to 1\% and the model parameters are updated by using \Basel\ and \Increm\ (with $T_0$ as 5 and $j_0$ as 20) respectively after the deletion operations. The experimental results are presented in Figure \ref{fig: cifar_resnet}, again showing significant speed-ups for \Increm\ relative to \Basel\ (up to 3x speed-ups when the deletion rate is 0.005\%) with far smaller error bound (up to $4\times 10^{-3}$) than the baseline error bound (up to $2\times 10^{-2}$). Since it is quite common to reuse sophisticated pre-trained  models in practice, we expect that the use of \Increm\ in this manner is applicable in many cases.

\begin{figure}
\begin{center}
\centerline{\includegraphics[width=0.4\columnwidth, height=0.3\columnwidth]{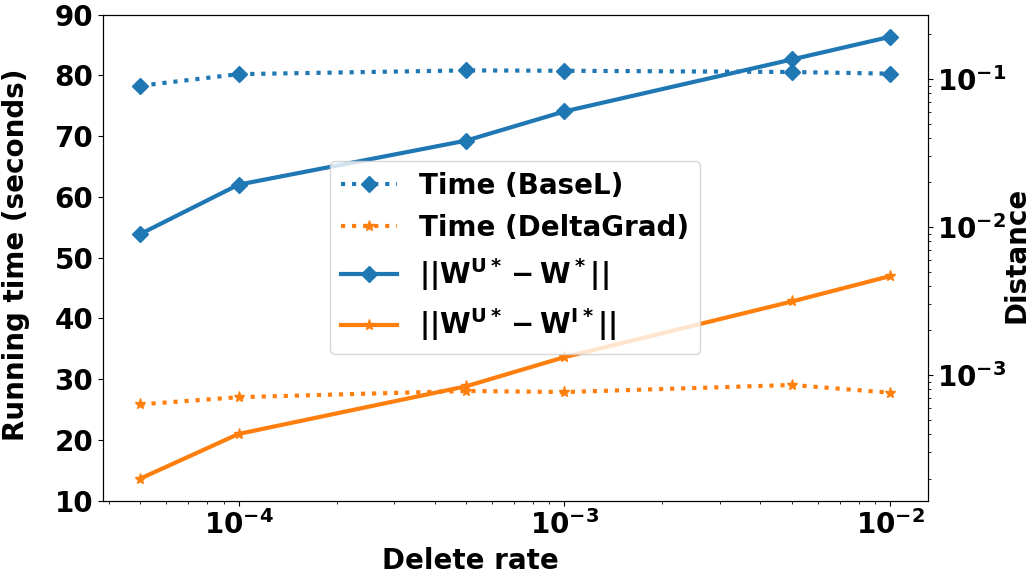}}
\caption{Comparison of \Increm\ and \Basel\ on the \cifar\ dataset with pre-trained ResNet152 network}
\label{fig: cifar_resnet}
\end{center}
\end{figure}

\yw{\subsection{Applications of \Increm\ to robust learning}\label{sec: exp_applications}}

As Section \ref{sec: app} in the main paper reveals, \Increm\ has many potential applications. In this section, we explored how \Increm\ can accelerate the evaluations of the effect of the outliers in robust statistical learning. Here the effect of outliers is represented by the difference of the model parameters before and after the deletion of the outliers (see \cite{yu2017robust}).

In the experiments, we start by training a model on the training dataset (\rcv\ here) along with some randomly generated outliers. Then we remove those outliers and update the model on the remaining training samples by using \Increm\ and \Basel. We also evaluate the effect of the fraction of outlier: the ratio between the number of the outliers and the training dataset size is also defined as the {\em Deletion rate}. It is varied from 1\% to 10\%. According to the experimental results shown in Figure \ref{fig: robustness_learning}, when there are up to 10\% outliers in the training dataset, \Increm\ is at least 2.18x faster than \Basel\ in evaluating the updated model parameters by only sacrificing little computational accuracy (no more than $5\times 10^{-3}$), thus reducing the computational overhead on evaluating the effect of the outliers in robust learning.

\begin{figure}
\begin{center}
\centerline{\includegraphics[width=0.4\columnwidth, height=0.3\columnwidth]{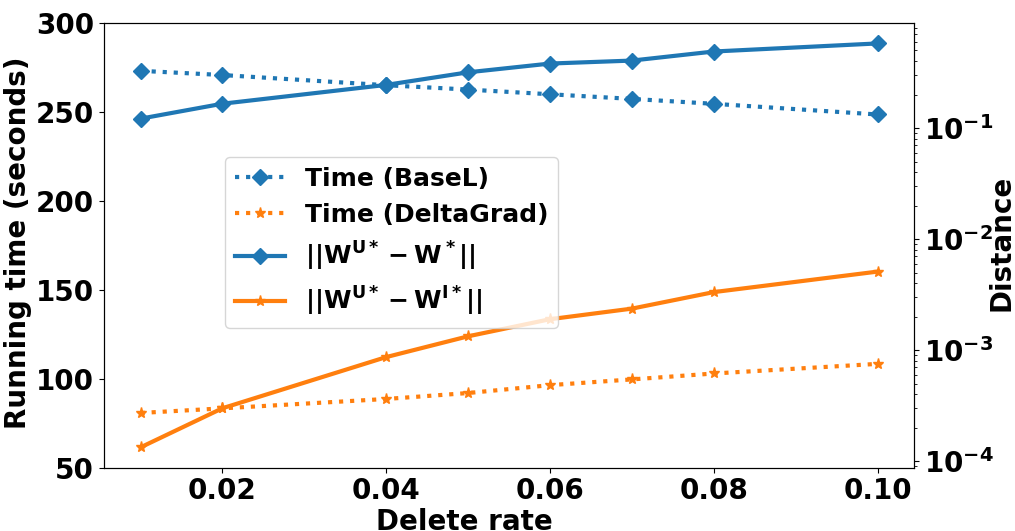}}
\caption{Comparison of \Increm\ and \Basel\ on \rcv\ dataset after deleting outliers}
\label{fig: robustness_learning}
\end{center}
\end{figure}